\documentclass[twoside,11pt]{article}
\usepackage{jmlr2e}

\usepackage{xr}
\usepackage{amsmath}
\usepackage{mathrsfs}
\usepackage{dsfont}
\usepackage{enumerate}
\usepackage{subfig}
\usepackage{float}
\usepackage{tikz}
\usepackage{pgfplots}

\DeclareMathOperator*{\argmin}{arg\,min}

\hypersetup{
colorlinks = true,
urlcolor = blue,
linkcolor = blue,
citecolor = blue,
}

\firstpageno{1}

\usepackage{lastpage}

\begin{document}
\title{Optimal $1$-Wasserstein distance for WGANs}

\author{\name Arthur St\'ephanovitch 
       \email stephanovitch@lpsm.paris \\
       \addr Universit\'e Paris Cit\'e, CNRS, LPSM\\
       F-75013 Paris, France
       \AND
       \name Ugo Tanielian \email ugo.tanielian@gmail.com\\
       \addr Criteo AI Lab, Paris, France
       \AND
       \name Beno\^{\i}t Cadre 
       \email benoit.cadre@univ-rennes2.fr \\
       \addr Univ Rennes, CNRS, IRMAR - UMR 6625\\
       F-35000 Rennes, France
       \AND
       \name Nicolas Klutchnikoff 
       \email nicolas.klutchnikoff@univ-rennes2.fr \\
       \addr Univ Rennes, CNRS, IRMAR - UMR 6625\\
       F-35000 Rennes, France
       \AND
        \name G\'erard Biau 
        \email gerard.biau@sorbonne-universite.fr\\
       \addr Sorbonne Universit\'e, CNRS, LPSM\\
       F-75005 Paris, France}
\maketitle

\vspace{-1.3cm}
\begin{center}
\rule{0.6\linewidth}{1pt}
\end{center}
\begin{abstract}
The mathematical forces at work behind Generative Adversarial Networks raise challenging theoretical issues. Motivated by the important question of characterizing the geometrical properties of the generated distributions, we provide a thorough analysis of Wasserstein GANs (WGANs) in both the finite sample and asymptotic regimes. We study the specific case where the latent space is univariate and derive results valid regardless of the dimension of the output space. We show in particular that for a fixed sample size, the optimal WGANs are closely linked with connected paths minimizing the sum of the squared Euclidean distances between the sample points. We also highlight the fact that WGANs are able to approach (for the $1$-Wasserstein distance) the target distribution as the sample size tends to infinity, at a given convergence rate and provided the family of generative Lipschitz functions grows appropriately. We derive in passing new results on optimal transport theory in the semi-discrete setting.
\end{abstract}
\begin{center}
\rule{0.6\linewidth}{1pt}
\end{center}

\section{Introduction}
    \label{section:introduction}
    Recent years have witnessed the advent of generative methodologies based on Generative Adversarial Networks \citep[GANs,][]{GANs}, with outstanding achievements in the fields of image \citep{radford2015unsupervised, karras2017progressive}, video \citep{vondrick2016generating}, and text generation \citep{SeqGANs}, just to name a few. The surveys by \citet{lucic2018GANs} and \citet{borji2019pros} cover the different GANs techniques together with a comparison of their performances. We are concerned with the Wasserstein GAN (WGAN) approach of \citet{arjovsky2017wasserstein}, which uses the $1$-Wasserstein distance as an alternative to the Jensen-Shannon divergence implemented in traditional GANs. Over the years, WGANs and their derivatives have gained popularity in the machine learning community. They are today considered as one of the most successful generative techniques, achieving state-of-the art results in difficult problems \citep{karras2017progressive, karras2018style} while improving the stability and getting rid of unpleasant issues such as mode collapse \citep{gulrajani2017improved}. 
   
    To get started, let us properly define WGANs. Assume that we are given a sample $X_1, \hdots, X_n$ of independent $\mathbb R^d$-valued random variables, identically distributed according to some unknown distribution $\mu$. Throughout the manuscript, the space $\mathbb R^d$ as well as all other spaces $\mathbb R^k$ are equipped with the Euclidean norm $\|\cdot\|$, with no reference to $d$ or $k$ as the context is clear. The generative problem is to use the sample $X_1, \hdots, X_n$ to learn $\mu$ and, simultaneously, generate new ``fake'' data that look ``similar'' to the $X_i$'s. In the WGAN framework, this problem is addressed by minimizing the $1$-Wasserstein distance between a family of candidate distributions and the empirical measure of the sample. Recall here that for two probability measures $\pi_1$ and $\pi_2$ on $\mathbb R^d$, the $1$-Wasserstein distance $W_1(\pi_1,\pi_2)$ between $\pi_1$ and $\pi_2$ is defined by
    \begin{equation*}
    W_1(\pi_1, \pi_2)=\inf_{\pi \in \Pi (\pi_1, \pi_2)}\int_{\mathbb R^d \times \mathbb R^d} \|x-y\|{\rm d}\pi(x,y),
    \end{equation*}
    where $\Pi(\pi_1, \pi_2)$ denotes the collection of all joint probability measures $\pi$ on $\mathbb R^d \times \mathbb R^d$ with marginals $\pi_1$ and $\pi_2$ \citep[e.g.,][]{villani2008optimal}. Notice that $W_1(\cdot,\cdot)$ is not a distance in the strict sense, because it may take the value $+\infty$. We also recall that the empirical measure $\mu_n$ based on $X_1, \hdots, X_n$ is defined, for any Borel set $A \subseteq \mathbb R^d$, by $\mu_n(A)=\frac{1}{n}\sum_{i=1}^n \mathds 1{\{X_i \in A\}}$. 
    Now, let $U$ be a uniform random variable on $[0,1]^p$ and, for $K>0$, let $\text{Lip}_K(E,E')$ be the set of $K$-Lipschitz continuous functions from $E\subseteq \mathbb R^k$ to $E'\subseteq \mathbb R^{k'}$, equipped with their respective Euclidean norms, that is
    \begin{equation*}
        \text{Lip}_K (E,E')= \{G:E\to E' : \|G(x)-G(y)\|\leqslant K\|x-y\|, \ (x,y) \in E^2\}.    
    \end{equation*} 
    For $G \in \text{Lip}_K([0,1]^p,\mathbb R^d)$, we denote by $G_{\sharp U}$ the pushforward distribution of $U$ by $G$, that is, for any Borel set $A \subseteq \mathbb R^d$, $G_{\sharp U}(A)=\lambda_p(G^{-1}(A))$, where $\lambda_p$ is the Lebesgue measure on $\mathbb R^p$. In their abstract formulation, WGANs use the family of pushforward distributions $\{G_{\sharp U}:G \in \text{Lip}_K([0,1]^p,\mathbb R^d)\}$ as candidate distributions to estimate $\mu$, with the objective of finding the best function $G$ that minimizes the $1$-Wasserstein distance between $G_{\sharp U}$ and the empirical measure $\mu_n$. In other words, one seeks to find an optimal ${\widehat G}_{K} \in \text{Lip}_K([0,1]^p,\mathbb R^d)$ such that
    \begin{equation}\label{eq:studyGANs}
      W_1({\widehat G}_{K\sharp U},\mu_n)=\underset{G \in \text{Lip}_K([0,1]^p,\mathbb R^d)}{\inf} \ W_1(G_{\sharp U}, \mu_n).
    \end{equation}
    Once a minimizer ${\widehat G}_{K}$ has been found, it is easy to generate ``fake'' observations, by simply taking a uniform i.i.d.~sample $U_1, \hdots, U_m$ and computing ${\widehat G}_{K}(U_1), \hdots,  {\widehat G}_{K}(U_m)$. In the GAN literature, the space $[0,1]^p$ is called the latent space and the distribution of the random variable $U$ the latent distribution. It should be stressed that assuming Lipschitz continuous candidate functions $G$ is classical when defining WGANs \citep[e.g.,][]{zhou2019lipschitz}. However, some authors have also considered smoother classes, such as for example functions with Lipschitz partial derivatives up to some order \citep[e.g.,][]{luise2020generalization, schreuder2021statistical}. To keep things as simple as possible, we do not make further assumptions on the generative functions other than their Lipschitz property.  
    
    The key to approach the infimum in~\eqref{eq:studyGANs} is to use the dual formulation of the $1$-Wasserstein distance \citep[][]{kantorovich1958space}. Indeed, one has
    \begin{align*}
     W_1(G_{\sharp U}, \mu_n) &= \sup \limits_{ f \in \text{Lip}_1(\mathbb R^d,\mathbb R)}\ \int_{\mathbb R^d} f{\rm d}{G_{\sharp U}} - \int_{\mathbb R^d} f{\rm d}\mu_n \\
     &=\sup \limits_{ f \in \text{Lip}_1(\mathbb R^d,\mathbb R)}\ \int_{[0,1]^p}f(G(u)){\rm d}u - \frac{1}{n}\sum_{i=1}^n f(X_i),
    \end{align*}
    so that the WGAN optimization Problem~\eqref{eq:studyGANs} takes the min-max form
    \begin{equation}
    \label{problem}
        W_1({\widehat G}_{K\sharp U},\mu_n)=\underset{G \in \text{Lip}_K([0,1]^p,\mathbb R^d)}{\inf} \ \sup \limits_{f \in \text{Lip}_1(\mathbb R^d,\mathbb R)}\ \int_{[0,1]^p} f(G(u)){\rm d}u - \frac{1}{n}\sum_{i=1}^n f(X_i).
    \end{equation}
    Since the nonparametric classes $\text{Lip}_K([0,1]^p,\mathbb R^d)$ and $\text{Lip}_1(\mathbb R^d,\mathbb R)$ are too large to be implemented, they are replaced in practice by parametric models, respectively called the generator and the discriminator. In most applications, these parametric models take the form of multilayer neural networks, either feedforward or convolutional, hence the name WGANs. It is also important to note that in practice the function $G_{\sharp U}$ in \eqref{eq:studyGANs} is estimated by random samples $G(U_1), \hdots, G(U_m)$ drawn from $U$. In other words, there exists an estimation error---on top of the approximation error by neural networks---between the optimum $\inf_G W_1(G_{\sharp U} , \mu_n)$ and any simulation. However, sampling from $U$ is easy and one can take sufficiently large $m$. From an optimization perspective, the training of (W)GANs is challenging. The min-max optimum in~\eqref{problem} is usually found by using stochastic gradient descent, alternatively on the generator's and the discriminator's parameters. Studying the convergence of the different learning procedures is an interesting question, tackled for example by \citet{kodali2017convergence} and \citet{mescheder2018training}.
    
    In addition to the numerous empirical research studies, several theoretical articles aimed at understanding the mathematical and statistical properties of the adversarial problem~\eqref{problem} and its extensions to integral probability metrics \citep[IPM,][]{IPMsMuller}. For example, leveraging the approximation properties of some family of neural networks, \citet{BST} study the convergence of the model as the sample size tends to infinity, and clarify the respective effects of the generator and the discriminator by underlining some trade-off properties. Assuming smoothness properties on the generator and the discriminator, \citet{liang2018well} and \citet{nonparametric2018singh} exhibit rates of convergence under an IPM-based loss for estimating densities that live in Sobolev spaces, while \citet{nonparametric2019wallach} explore the case of Besov spaces. More recently, \citet{schreuder2021statistical} have stressed the properties of IPM losses defined with smooth functions on a compact set. Remarkably, \citet{liang2018well} discusses bounds for the Kullback-Leibler divergence, the Hellinger distance, and the $1$-Wasserstein distance. Studying a different facet of the problem, \citet{luise2020generalization} analyze the interplay between the latent distribution and the complexity of the pushforward map, and how it affects the overall performance. 
    
    In this paper, we seek to describe the properties of the $K$-Lipschitz continuous functions that achieve the infimum in~\eqref{eq:studyGANs}. Our approach is motivated by an active line of experimental research, which aims at characterizing the distributions output by GANs, typically the geometry of their supports. For example, when dealing with the learning of disconnected manifolds, \citet{tanielian2020learning} derived lower-bounds on the measure of the proposal distribution that lies out of the target manifold. Another much-debated question is to understand to what extent GANs memorize the dataset \citet{nagarajan2018theoretical}. In this regard, \citet{gulrajani2018towards} stress their tendency to memorize, and, in turn, propose a new evaluation protocol that enhances generalization. Yet, most of the conclusions on this subject are of an experimental nature, without clear theoretical arguments regarding the statistical properties of the distribution produced by GANs. 
    
    Motivated by the above, we provide in the present article a thorough analysis of Problem~\eqref{eq:studyGANs}. Since this question is highly nontrivial, we deeply study the univariate latent setting ($p=1$). Beyond the technical aspects, the motivation to study the univariate case is related to the so-called manifold hypothesis \citep{fefferman2016testing, facco2017}, which states that high-dimensional datasets may lay on manifolds of lower dimensions. For instance, \citet{YoonHaeng2021} show that using a latent dimension $p=2$ is already sufficient to generate high-quality images for the MNIST dataset. We later give intuitions for the case $p>1$. 
    
    Our contributions are the following:
    \begin{enumerate}
        \item To grasp how WGANs can approach the distribution $\mu$, we start in Section~\ref{Kfixed} by an asymptotic analysis of $W_1({\widehat G}_{K\sharp U},\mu)$ as the sample size $n$ tends to infinity, assuming that the Lipschitz constant $K$ is kept fixed, independent of the data. We show in particular that in most situations, and independently of the dimension $d$, one has $\liminf_{n\to \infty}  W_1({\widehat G}_{K\sharp U},\mu)>0 \text{ a.s.}$
        \item Next, we provide in Section~\ref{1DO} a thorough finite sample analysis of the case $d=1$, that is, whenever the output space is univariate. In this context, the Lipschitz constant $K$ is allowed to depend on the sample $X_1, \hdots, X_n$. We explicitly describe the (two) functions achieving the infimum in~\eqref{eq:studyGANs}, give the exact value of the infimum, and show that the corresponding optimal distributions have atoms at the $X_i$'s. Finally, taking an asymptotic point of view, we prove that $\lim_{n\to \infty}  W_1({\widehat G}_{K\sharp U},\mu)=0$ and offer convergence rates.
        \item We then discuss in Section~\ref{sec:opt_transport} new existence results on transport maps in semi-discrete optimal transport theory, for measures that are non necessarily absolutely continuous with respect to the Lebesgue measure on $\mathbb{R}^d$. This step is necessary before diving into the analysis of Problem~\eqref{eq:studyGANs} for $d>1$.
        \item In Section~\ref{sec:multi_dim_observations}, we move to the case where the observations are multivariate ($d>1$) and derive a finite sample bound on the infimum in~\eqref{eq:studyGANs}. We show in particular, provided $K$ is allowed to depend on the sample, that the bound is achieved by a distribution concentrated on a shortest-path-type graph constructed on the $X_i$'s. Up to our knowledge, this is the first time that such bounds are available in the literature. Taking neural networks for the generator and the discriminator classes, we illustrate the results empirically. Similarly to Section~\ref{1DO}, we also provide convergence rates for $\lim_{n\to \infty}  W_1({\widehat G}_{K\sharp U},\mu)$.
        \end{enumerate}
    All the proofs are gathered in the Annex \citep{supplementarymaterial}, with the exception of the proofs of Theorem~\ref{theorem2} and Theorem~\ref{theorem3}.
    \section{Asymptotic analysis}\label{Kfixed}
    The study begins with an asymptotic analysis of Problem~\eqref{eq:studyGANs}, when the sample size $n$ tends to infinity and the Lipschitz constant $K$ is assumed to be fixed. For more clarity, the univariate case $d=1$ is handled in Theorem~\ref{theorem1asymptotique1} and the multivariate case $d>1$ in Theorem~\ref{theorem1asymptotique2}. Recall that the latent variable $U$ is assumed to be uniformly distributed on $[0,1]$, and that the data $X_1,\hdots,X_n$ are i.i.d~with unknown distribution $\mu$. Throughout, we let 
    \[
    {\widehat{\mathscr{G}}}_K=\argmin_{G \in \text{Lip}_K ([0,1],\mathbb R^d)}W_1(G_{\sharp U}, \mu_n) 
    \]
    be the set of minimizers of Problem~\eqref{eq:studyGANs}, that is,
    \begin{equation*}
        {\widehat{\mathscr{G}}}_K = \{\widehat{G}_K \in \text{Lip}_K ([0,1],\mathbb R^d) : W_1(\widehat{G}_{K\sharp U}, \mu_n) = \underset{G \in \text{Lip}_K([0,1],\mathbb R^d)}{\inf} W_1(G_{\sharp U}, \mu_n) \}.
    \end{equation*}
    Observe that $\{\widehat{G}_{K\sharp U}: \widehat{G}_K \in {\widehat{\mathscr{G}}}_K\}$ is the collection of optimal distribution(s). Whenever $\mu$ is of order $1$, i.e., $\mathbb E\|X_1\|=\int_{\mathbb R^d} \|x\|\mu({\rm d}x)<\infty$, it is convenient to consider $\mathscr{G}_K$, the population version of ${\widehat{\mathscr{G}}}_K$ defined by
    \[
    \mathscr{G}_K=\argmin_{G_K \in \text{Lip}_K ([0,1],\mathbb R^d)}W_1(G_{K\sharp U}, \mu).
    \]
    We start with the following simple but useful lemma.
    \begin{lemma}
    \label{lemma1}
    The set ${\widehat{\mathscr{G}}}_K$ is not empty. In addition, assuming that $\mu$ is of order $1$, the set $\mathscr{G}_K$ is not empty. 
    \end{lemma}
    In the sequel, we let $S(\mu)$ be the support of $\mu$, i.e., 
    \[S(\mu)=\{x \in \mathbb R^d: \mu(B(x,\varepsilon))>0 \text{ for all } \varepsilon >0\},\]
    where $B(x,\varepsilon)$ is the closed ball in $\mathbb R^d$ centered at $x$ of radius $\varepsilon$. We are now ready to state the first theorem, which reveals the different behaviors of the quantity $W_1({\widehat G}_{K\sharp  U},\mu)$ in dimension $d=1$, provided ${\widehat G}_K$ is any minimizer in $\mathscr {\widehat G}_{K}$. Interestingly, we distinguish different cases depending on both the smoothness of the distribution function of $\mu$ and the boundedness of its support $S(\mu)$. 
    \begin{theorem}[Case $d=1$] 
    \label{theorem1asymptotique1}
    Let ${\widehat G}_{K} \in \mathscr {\widehat G}_{K}$. Assume that $\mu$ is of order $1$, and let $F^{-1}$ be the generalized inverse of the distribution function $F$ of $\mu$, i.e., for all $u\in (0,1)$,
    $F^{-1}(u)=\inf\{ x\in\mathbb R \, :\, F(x)\geqslant u\}.$
    \begin{enumerate}
    \item Assume that $S(\mu)$ is bounded. 
    \begin{enumerate}
    \item[$(i)$] If $F^{-1}\in {\rm Lip}_{K_0}([0,1],\mathbb R)$ for some $K_0>0$, then, for all $K\geqslant K_0$, 
    \[\lim_{n\to \infty} W_1({\widehat G}_{K\sharp  U},\mu)=0 \text{ a.s.}\]
    \item[$(ii)$] If $F\in {\rm Lip}_{K_1}(\mathbb R,[0,1])$ for some $K_1>0$, then, for all $K<1/K_1$, 
    \[\liminf_{n\to \infty} \ W_1({\widehat G}_{K\sharp  U},\mu)>0 \text{ a.s.}\]
    \end{enumerate}
    \item Assume that $S(\mu)$ is unbounded. Then, for all $K>0$, 
    \[\liminf_{n\to \infty} \ W_1({\widehat G}_{K\sharp  U},\mu)>0 \text{ a.s.}\]
    \end{enumerate}
    \end{theorem}
    A first remark could be that, in $1(i)$, the support $S(\mu)$ is necessarily bounded since $F^{-1}$ is assumed to be a $K_0$-Lipschitz function on $[0,1]$. Next, note that both conditions in $1(i)$ and $1(ii)$ may be satisfied simultaneously or not. For example, when $\mu$ is the uniform distribution on $[0,1]$, they are both verified with $K_0=K_1=1$. Also, observing that $K_0K_1\geqslant 1$ (since $F\circ F^{-1}$ is the identity function), these two conditions focus in fact on different regimes. The first one pertains to the case where the set of generative functions ought to be big while the second one claims that a smaller class cannot recover the target distribution $\mu$. In $2$, we notice however that independently of the smoothness of $\mu$ and the magnitude of $K$, WGANs cannot recover the target distribution. This is for example the case when $\mu$ is a standard Gaussian distribution on the real line. The mechanism is illustrated in Figure~\ref{fig:figure1_dim1_real}, which shows the values of $W_1({\widehat G}_{K\sharp U},\mu)$ as a function of both $n$ and $K$, when the target distribution $\mu$ is either uniform (left) or Gaussian (right). In the uniform case, as predicted by Theorem~\ref{theorem1asymptotique1}, we see that $W_1({\widehat G}_{K\sharp U},\mu)$ significantly decreases for $K$ larger than $1$ and stays rather constant for smaller $K$. In the Gaussian setting, $1$-Wasserstein distances are far from zero, independently of the value of $n$ and $K$. In the experiment, the generator is a 3-layer feedforward neural network while the discriminator is a 5-layer network.
    \begin{figure}[h]
        \centering
        \subfloat[The distribution $\mu$ is uniform.]
        {
            \includegraphics[width=0.44\linewidth]{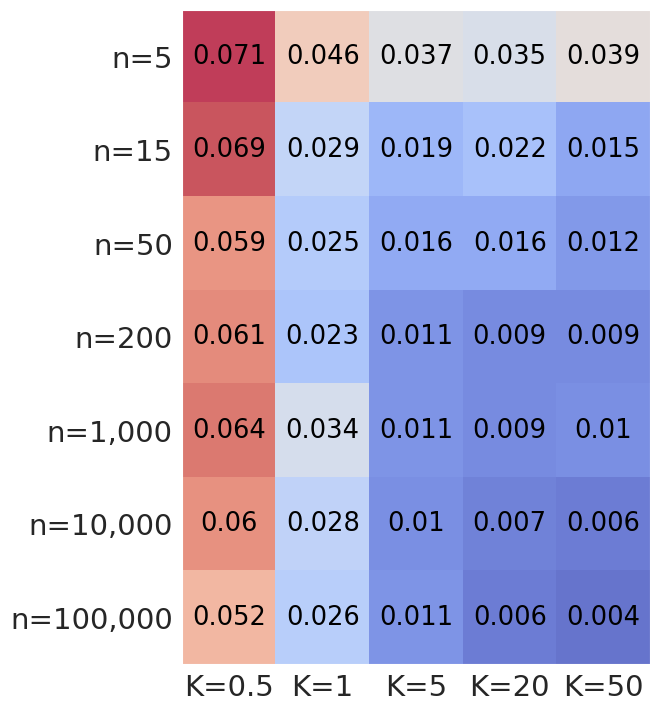}
        }\hfill
        \subfloat[The distribution $\mu$ is standard Gaussian.] 
        {
            \includegraphics[width=0.44\linewidth]{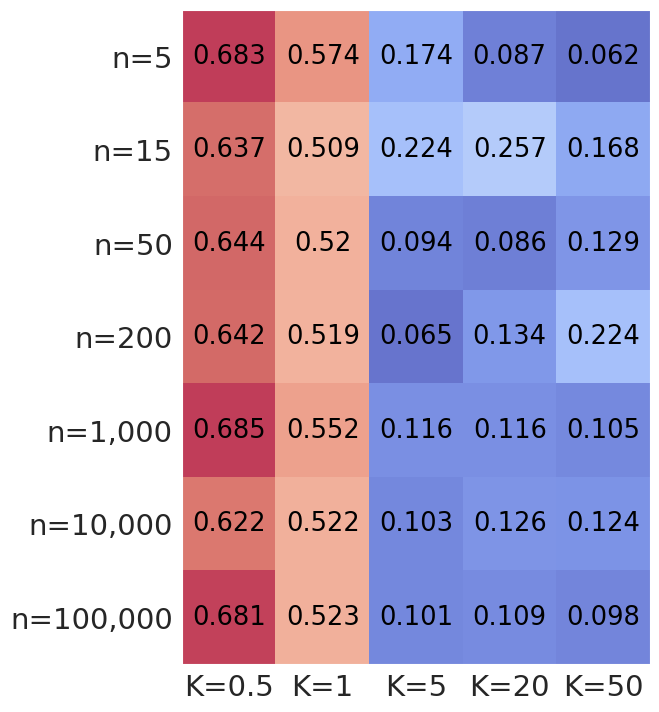}
        }
        \caption{$1$-Wasserstein distance $W_1({\widehat G}_{K\sharp U}, \mu)$ as a function of $n$ and $K$ (the bluer the lower and the redder the higher). Results are averaged over $2$ runs. \label{fig:figure1_dim1_real}}
    \end{figure}

    These results should of course be appreciated in the light of the specific case where both the latent space and the target distribution $\mu$ share the same dimension $1$. In the case where $\mu$ lies on a space of dimension strictly larger than $1$, then the  minimizers in $\mathscr {\widehat G}_{K}$ cannot reconstruct $\mu$, as stated by the following theorem.
    \begin{theorem}[Case $d>1$] 
    \label{theorem1asymptotique2}
    Let ${\widehat G}_{K} \in \mathscr {\widehat G}_{K}$. Assume that $\mu$ is of order $1$ and that $\lambda_d(S(\mu))>0$, where $\lambda_d$ denotes the Lebesgue measure on $\mathbb R^d$. Then, for all $K>0$,
    \[\liminf_{n\to \infty} \ W_1({\widehat G}_{K\sharp U},\mu)>0 \text{ a.s.}\]
    \end{theorem}
    The condition on the support of $\mu$ states that $\mu$ is a ``true'' measure on $\mathbb R^d$. We leave it as an exercise to prove that the same result holds by assuming that the Hausdorff dimension of $S(\mu)$ is strictly larger than 1. 
    \section{Finite sample analysis in a univariate output space}\label{1DO}
    The topic of the present section is to fully describe the set of minimizers $\mathscr {\widehat G}_{K}$ (Lemma~\ref{lemma1}), in the specific setting where both the output and the latent spaces are univariate. We denote by $X_{(1)}, \hdots, X_{(n)}$ the reordering of $X_1, \hdots, X_n$ according to their increasing values, that is $X_{(1)} \leqslant X_{(2)} \leqslant \cdots \leqslant X_{(n)}$, where ties are broken arbitrarily. Importantly, the Lipschitz constant $K$ is now allowed to depend upon the sample and is chosen to satisfy the constraint $K\geqslant n \max \limits_{i=1,\hdots,n-1} (X_{(i+1)}-X_{(i)})$. 
    
    The analysis starts by introducing the following function ${\widehat G}_{K}^{\star}:[0,1] \to \mathbb R$, which will play a key role in solving Problem~\eqref{eq:studyGANs}: for all $u\in [0,1]$,
    \begin{equation} 
    \label{Gstar}
    {\widehat G}^{\star}_{K}(u) = \left\{
    \begin{array}{ll}
     X_{(1)} & \text{if } u\in\big[0,\frac{1}{n}-\frac{X_{(2)}-X_{(1)}}{2K}\big] \medskip\\
        X_{(i)}+ K\big(u-(\frac{i}{n}-\frac{X_{(i+1)}-X_{(i)}}{2K})\big) & \text{if } u \in\big[\frac{i}{n}-\frac{X_{(i+1)}-X_{(i)}}{2K},\frac{i}{n}+\frac{X_{(i+1)}-X_{(i)}}{2K}\big]  \smallskip\\
        &  \text{for } 1\leqslant i\leqslant n-1 \medskip\\
        X_{(i+1)} & \text{if } u\in \big[\frac{i}{n}+\frac{X_{(i+1)}-X_{(i)}}{2K},\frac{i+1}{n}-\frac{X_{(i+2)}-X_{(i+1)}}{2K}\big] \smallskip\\
        & \text{for } 1\leqslant i\leqslant n-2 \medskip\\
        X_{(n)} & \text{if } u\in \big[ \frac{n-1}{n} + \frac{X_{(n)}-X_{(n-1)}}{2K},1\big].
    \end{array}
    \right.
    \end{equation} 
    Observe that ${\widehat G}_{K}^{\star}$ is piecewise linear and that the condition $K\geqslant n \max \limits_{i=1,\hdots,n-1} (X_{(i+1)}-X_{(i)})$ ensures that this function is well-defined. We also note that ${\widehat G}_{K}^{\star} \in \text{Lip}_K([0,1],\mathbb R)$ and that it visits each data point, going iteratively from $X_{(i)}$ to $X_{(i+1)}$. A typical example is shown in Figure~\ref{Gstargraph}. Observe that, for each $i\in \{1, \hdots, n\}$,
    \begin{equation}
    \label{NFR}
    \lambda_1\big (\{u \in [0,1] :  |{\widehat G}^{\star}_{K}(u) -X_i |\leqslant  |{\widehat G}^{\star}_{K}(u) -X_j |\ :\ j=1, \hdots, n \}\big) = \frac{1}{n}.
    \end{equation}
    This geometric feature has an interpretation in terms of Voronoi cells and will play an important role in the multivariate extension of ${\widehat G}^{\star}_{K}$, as we will see in Section~\ref{sec:multi_dim_observations}.
    \begin{figure}
    \centering
    \begin{tikzpicture}
     \begin{axis}[
        xmin = 0, xmax = 1 ,
        ymin = 0, ymax = 10,
                    grid = both,
                minor tick num = 0,
                major grid style = {lightgray},
                minor grid style = {lightgray!25},
        ytick={1.0,2.0,4.0,7.0,9.0},
        yticklabels={$X_{(1)}$,$X_{(2)}$,$X_{(3)}$,$X_{(4)}$,$X_{(5)}$},
        xtick={0,0.2,0.4,0.6,0.8,1},
        xticklabels={0,$\frac{1}{5}$,$\frac{2}{5}$,$\frac{3}{5}$,$\frac{4}{5}$,1},
        width = 0.80\linewidth, height = 0.45\linewidth,]
     \addplot[
        thin,
        blue,
       thick
    ] file[skip first]{function.dat}; 
     \end{axis}
    \end{tikzpicture}
    \caption{An example of function ${\widehat G}^{\star}_{K}$, with $n=5$ and $K=25$.}
    \label{Gstargraph}
    \end{figure}
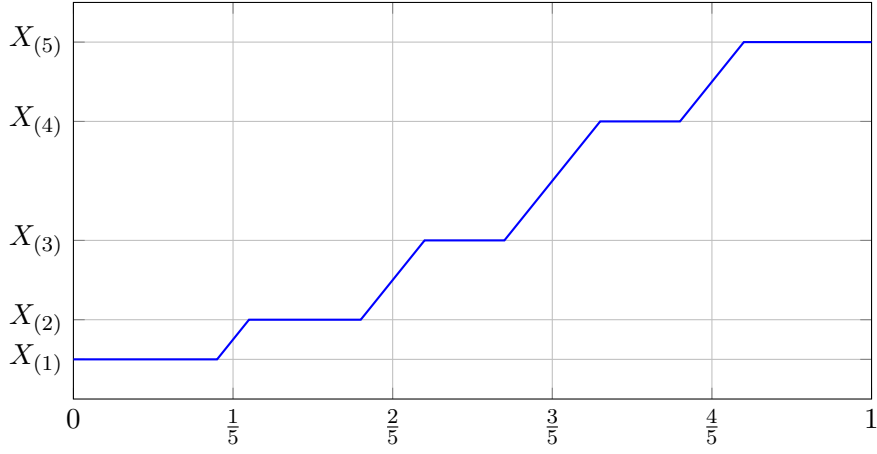
    
    \begin{proposition} 
    \label{proposition1}
    Assume that $K\geqslant n \max \limits_{i=1,\hdots,n-1} (X_{(i+1)}-X_{(i)})$, and let ${\widehat G}_{K}^{\star} \in {\rm Lip}_K([0,1],\mathbb R)$ be defined in~\eqref{Gstar}. Then
    \[
    W_1({\widehat G}^{\star}_{K\sharp  U},\mu_n) = \frac{1}{4K} \sum \limits_{i=1}^{n-1}(X_{(i+1)}-X_{(i)})^2.
    \]
    \end{proposition}
    The key message of Proposition~\ref{proposition1} is that the $1$-Wasserstein distance between ${\widehat G}^{\star}_{K\sharp U}$ and $\mu_n$ depends on the sum of the {\it squared} distances $(X_{(i+1)}-X_{(i)})^2$. We are now in a position to state the main result of the section. 
    \begin{theorem}\label{theorem1}
    Assume that $K\geqslant n \max \limits_{i=1,\hdots,n-1} (X_{(i+1)}-X_{(i)})$, and let the function ${\widehat G}_{K}^{\star} \in {\rm Lip}_K([0,1],\mathbb R)$ be defined in~\eqref{Gstar}. Then 
    \begin{equation*}
    W_1({\widehat G}^{\star}_{K\sharp  U},\mu_n)=\underset{G \in {\rm Lip}_K([0,1],\mathbb R)}{\inf} \ W_1(G_{\sharp U}, \mu_n)=\frac{1}{4K} \sum \limits_{i=1}^{n-1}(X_{(i+1)}-X_{(i)})^2.
    \end{equation*}
    Moreover, ${\widehat{\mathscr{G}}}_K= \{{\widehat G}_K^\star, {\widehat G}_K^\star\circ S\}$,  where $S(u)=1-u$, $u\in [0,1]$.
    \end{theorem}
    Theorem~\ref{theorem1} states that there are only two minimizers in ${\widehat{\mathscr{G}}}_K$. Moreover, the two distributions ${\widehat G}^{\star}_{K\sharp  U}$ and $({\widehat G}_{K}^{\star}\circ S)_{\sharp  U}$ are identical. We thus conclude that in the univariate setting, the distribution output by the WGAN Problem~\eqref{eq:studyGANs} exists and is unique, provided $K$ is large enough. It is important to note that the distribution ${\widehat G}^{\star}_{K\sharp  U}$ has atoms at the $X_i$'s, of respective sizes
    \begin{align}\label{eq:measure_atoms}
        \frac{1}{n}- \frac{X_{(2)}-X_{(1)}}{2K} \quad \text{for } X_{(1)} &,\quad \frac{1}{n}-\frac{X_{(n)}-X_{(n-1)}}{2K} \quad \text{for } X_{(n)}, \nonumber \\
        \frac{1}{n}-\frac{X_{(i+1)}-X_{(i-1)}}{2K} \quad &\text{for } X_{(i)}, \ i=2,\hdots,n-1,
    \end{align}
    and that it is absolutely continuous with respect to the Lebesgue measure elsewhere. 
    
    Being able to describe the minimizers of Problem~\eqref{eq:studyGANs} helps us to better understand the overall objective of WGANs when playing with different parameters. For example, when the dataset (and thus the sample size $n$) is kept fixed, the $1$-Wasserstein distance $W_1({\widehat G}^{\star}_{K\sharp  U},\mu_n)$ decreases towards $0$ as the Lipschitz constant $K$ gets bigger. This is easily explained by the fact that when $K$ increases, the class of generative distributions increases as well, and the measure of the atoms in~\eqref{eq:measure_atoms} of ${\widehat G}^{\star}_{K\sharp  U}$ grows towards $1$. In this regime, the optimal distribution ${\widehat G}^{\star}_{K\sharp  U}$ tends to memorize the data samples, that is the WGAN overfits the data. On the opposite, the measures of the different atoms $X_{(i)}$, $i \in \{1, \hdots, n\}$, decrease with the distance $X_{(i+1)}-X_{(i-1)}$. Consequently, any outlier data, far from its nearest neighbors, will be less sampled by the optimal distribution.
    
    In order to illustrate the result of Theorem~\ref{theorem1}, we consider a synthetic setting where both the class of generative and discriminative functions are replaced by parametric neural networks. The generator is composed of ReLU neural networks of respective depths $3$ (Figure~\ref{fig:dim1a} and~\ref{fig:dim1b}) and $5$ (Figure~\ref{fig:dim1c} and~\ref{fig:dim1d}), with a width $100$, while the discriminator is composed of  ReLU neural networks of depth $5$ and width $100$. The true distribution is assumed to be uniform on $[0,10]$. We train a WGAN architecture in the setting of both $n=5$ and $n=9$, with the choice $K=50$ (we choose $K$ big enough such that $K \geqslant n \max \limits_{i=1,\hdots,n-1} (X_{(i+1)}-X_{(i)})$). The Lipschitz constraint on the generator is implemented using a gradient penalty similar to the one used for the discriminator in \citet{gulrajani2017improved}. The obtained results are depicted in Figure~\ref{fig:dim1}. 
    \begin{figure}[!h]
        \centering
        \subfloat[Fitting $n=5$ data points with a generator depth equal to $3$. $W_1({\widehat G}^{\star}_{K\sharp  U},\mu_n)=0.080$ and $W_1(G^{\theta}_{\sharp  U},\mu_n)=0.501$.\label{fig:dim1a}]
        {
            \includegraphics[width=0.44\linewidth]{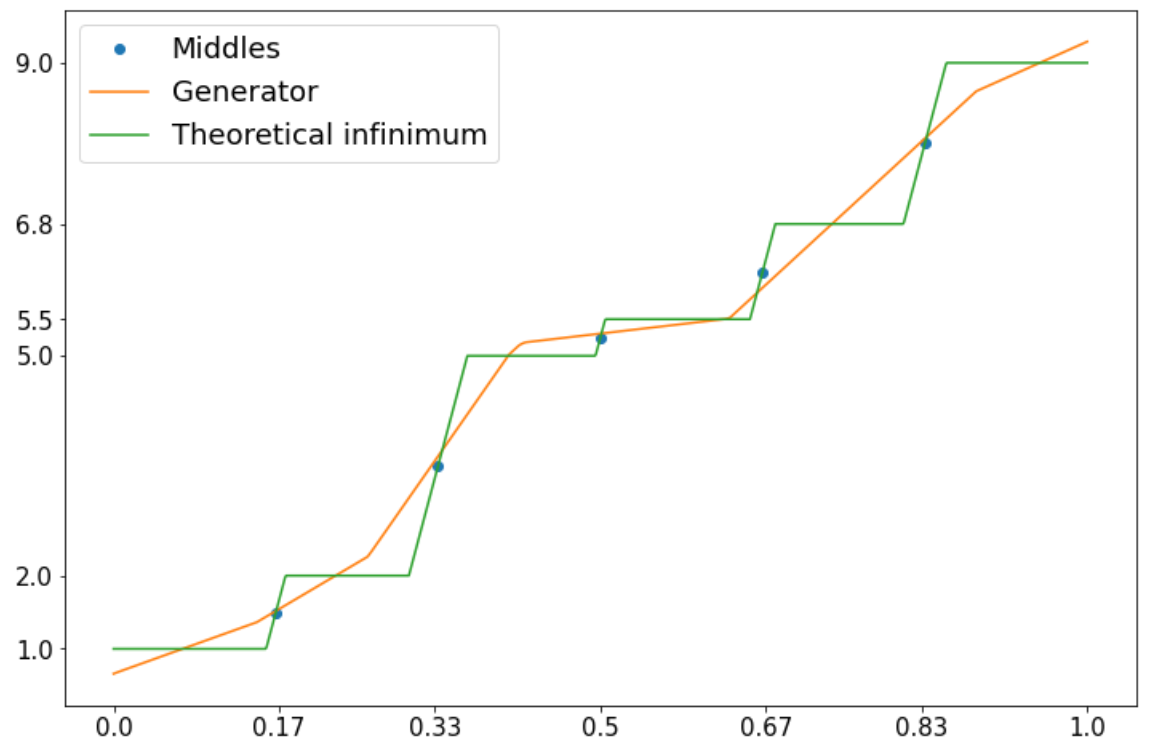}
        }\hfill
        \subfloat[Fitting $n=5$ data points with a generator depth equal to $5$. $W_1({\widehat G}^{\star}_{K\sharp  U},\mu_n)=0.080$ and $W_1(G^{\theta}_{\sharp  U},\mu_n)=0.165$.\label{fig:dim1c}]
        {
            \includegraphics[width=0.44\linewidth]{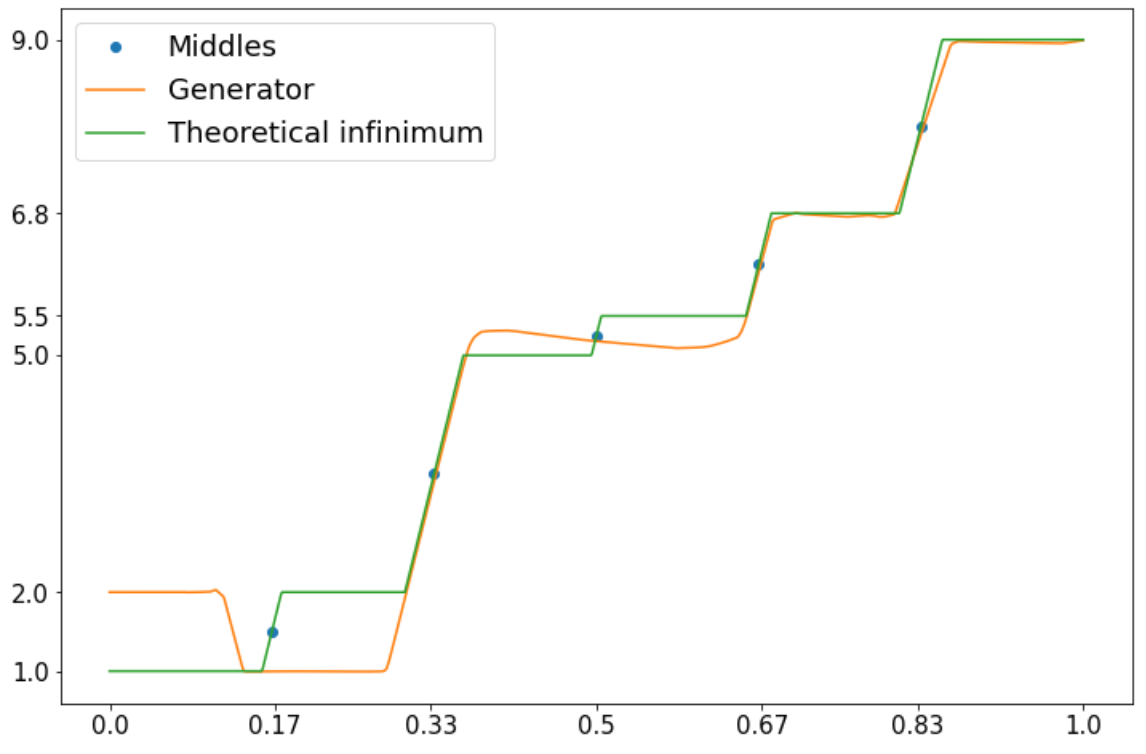}
        }\vfill
        \subfloat[Fitting $n=9$ data points with a generator depth equal to $3$. $W_1({\widehat G}^{\star}_{K\sharp  U},\mu_n)=0.033$ and $W_1(G^{\theta}_{\sharp  U},\mu_n)=0.280$.\label{fig:dim1b}]
        {
            \includegraphics[width=0.44\linewidth]{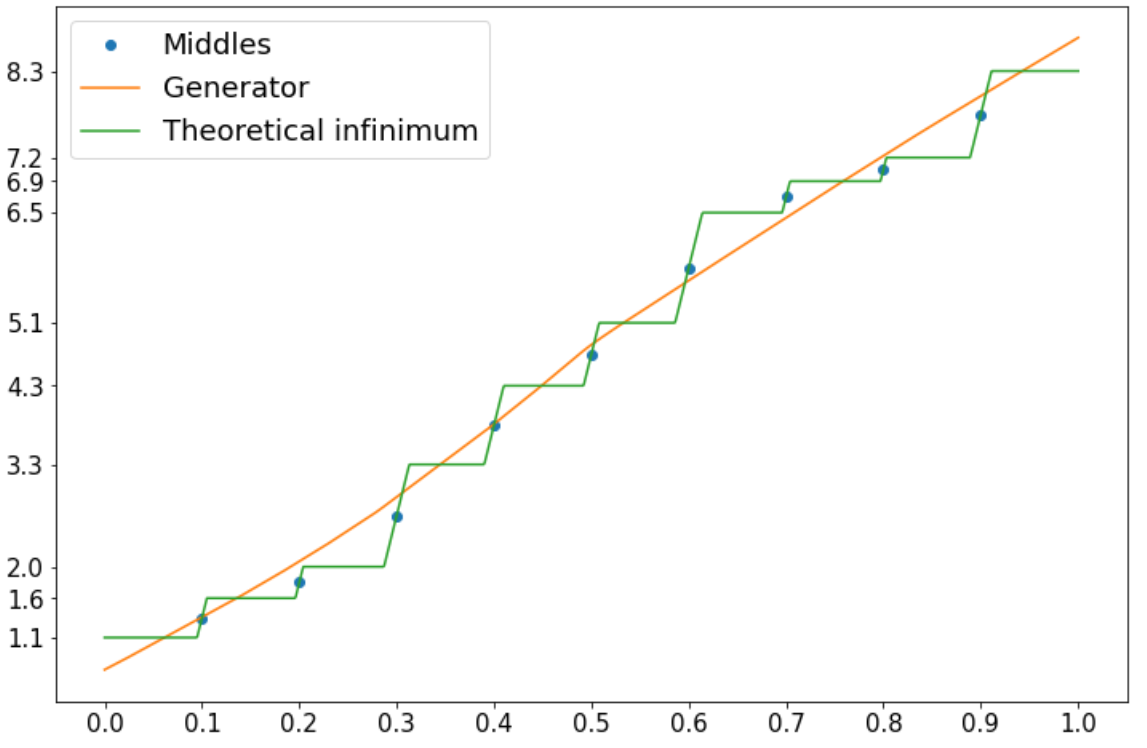}
        }
        \hfill
        \subfloat[Fitting $n=9$ data points with a generator depth equal to $5$. $W_1({\widehat G}^{\star}_{K\sharp  U},\mu_n)=0.033$ and $W_1(G^{\theta}_{\sharp  U},\mu_n)=0.210$.\label{fig:dim1d}]
        {
            \includegraphics[width=0.44\linewidth]{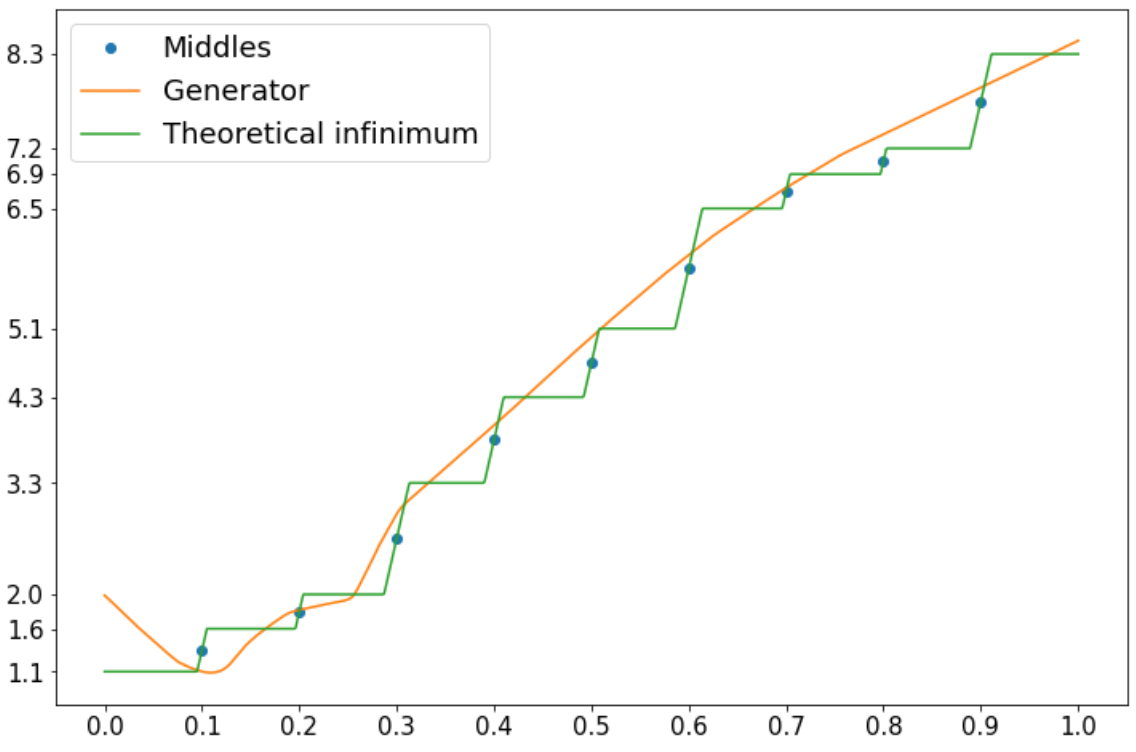}
        }
        \caption{Output functions $G^{\theta}$ of the WGANs compared with the optimal ${\widehat G}^{\star}_{K}$. \label{fig:dim1}}
    \end{figure}
    
    We see that the parametric WGANs (denoted by $G^{\theta}$) get close to the optimal function ${\widehat G}^{\star}_{K}$ while operating some smoothing. This smoothing is due to the fact that the networks cannot replicate all Lipschitz functions. Therefore, the optimal parametric WGANs have a higher $1$-Wasserstein distance to the empirical distribution than ${\widehat G}^{\star}_{K}$. Interestingly, as the number of samples increases and the architecture remains fixed, it gets more complicated for the generator to memorize the dataset. As expected, the results of the parametric WGANs get better as the depth of the generator increases.
    
    Changing a little bit the way of looking at the problem, one may take an asymptotic point of view in the sample size $n$ and analyze the asymptotic behavior of the $1$-Wasserstein distance $W_1({\widehat G}^{\star}_{K\sharp  U},\mu)$, as done in Section~\ref{section:introduction}. However, a major difference is that, in accordance with Theorem~\ref{theorem1}, the Lipschitz constant $K$ is now viewed as a data-dependent random variable larger than $\underline {K}_1$, where 
    \[\underline {K}_1:=n\max_{i=1,\hdots,n-1} (X_{(i+1)}-X_{(i)}).\]  
    \begin{proposition}
    \label{K1d1}
    Assume that $S(\mu)=[A,B]$, where $-\infty<A<B<\infty$. \begin{enumerate}
        \item If $\mu$ admits a strictly positive probability density on $[A,B]$, continuously differentiable, with a unique minimum on $[A,B]$, then
    \[\frac{1}{\underline {K}_1}=\mathscr O((\log n)^{-1})\text{ a.s.}\]
        \item For all $K\geqslant \underline {K}_1$, 
    \[ W_1({\widehat G}^{\star}_{K\sharp  U},\mu)=\mathscr O(n^{-1/2})\text{ in probability}.\]
    \end{enumerate}
    \end{proposition}
    The proof of Proposition~\ref{K1d1} reveals that $W_1({\widehat G}^{\star}_{K\sharp  U},\mu_n)=\mathscr O(n^{-1})$ in probability, which should be compared with the rate $W_1(\mu,\mu_n)=\mathscr O(n^{-1/2})$ \citep[][Theorem 1]{FournierGuillin2015}. Therefore, the speed of convergence to $0$ of $W_1({\widehat G}^{\star}_{K\sharp  U},\mu)$ is significantly slowed down by the term $W_1(\mu,\mu_n)$. Besides, the assumptions on $\mu$ are made here for simplicity, and many other cases may be handled similarly by connecting $\underline {K}_1$ to statistical results regarding the analysis of maximal spacings. For example, built on results from Extreme Values Theory, \citet[][Theorem 1 or Example 1]{Deheuvels86} entails that when $\mu$ is standard Gaussian, then, in probability,
    \[\frac{1}{\underline {K}_1}=
    \mathscr O\Big(\frac{\sqrt{\log n}}{n}\Big) \quad \mbox {and} \quad W_1({\widehat G}^{\star}_{K\sharp  U},\mu)=\mathscr O (n^{-1/2}).\]
    Similar results, yet with different rates, may be obtained for the Cauchy and Gamma distributions \citep[][Example $2$ and Example $3$]{Deheuvels86}. The general message is that, provided the class of candidate distributions grows with the sample size $n$, then the WGANs can asymptotically recover the target distribution $\mu$.
    \section{A general result in semi-discrete optimal transport}\label{sec:opt_transport}
    We now turn to the multivariate case, assuming that the observations $X_1,\hdots ,X_n$ are i.i.d.~according to an unknown distribution $\mu$ on $\mathbb{R}^d$, $d>1$. As we will see below, characterizing the optimal transport problem is much more complicated in the multivariate setting, and requires a more involved analysis. The key is to better describe the optimal transport function between $G_{\sharp  U}$ and $\mu_n$, keeping in mind that for $d>1$, $G_{\sharp  U}$ is never absolutely continuous with respect to the Lebesgue measure $\lambda_d$ on $\mathbb{R}^d$. Therefore, we need to extend existing results to larger classes of distributions.
    
    Recall, as we saw in the introduction, that for two probability measures $\pi_1$ and $\pi_2$ on $\mathbb R^d$,
    \begin{equation*}
    W_1(\pi_1, \pi_2)=\inf_{\pi \in \Pi (\pi_1, \pi_2)}\int_{\mathbb R^d \times \mathbb R^d} \|x-y\|{\rm d}\pi(x,y),
    \end{equation*}
    where $\Pi(\pi_1, \pi_2)$ denotes the collection of all transport plans between $\pi_1$ and $\pi_2$, that is, the joint probability measures $\pi$ on $\mathbb R^d \times \mathbb R^d$ with marginals $\pi_1$ and $\pi_2$. When $\pi_1$ is {\it nonatomic}, then, according to \citet[][Theorem B]{pratelli2020existence},
    \begin{equation}
    \label{W1Monge}
    W_1(\pi_1,\pi_2)=\inf_T \int_{\mathbb R^d} \|x-T(x)\|{\rm d}\pi_1(x),
    \end{equation}
    where the infimum is taken over all measurable functions $T:\mathbb R^d \to \mathbb R^d$ satisfying $T_{\sharp  \pi_1}=\pi_2$. Such a function $T$ is called a transport map from $\pi_1$ to $\pi_2$, and~\eqref{W1Monge} is referred to as the Monge formulation of the $1$-Wasserstein distance. Providing existence and unicity results for transport maps is, in general, a difficult question. It turns out however that in the so-called semi-discrete setting, where $\pi_1$ is absolutely continuous with respect to the Lebesgue measure and $\pi_2=\sum_{i=1}^n \alpha_i\delta_{x_i}$ is 
    discrete ($\alpha_i \geqslant 0$, $\sum_{i=1}^n{\alpha_i}=1$), the Monge problem has a simple and elegant solution in terms of additively weighted Voronoi diagram of $\mathbb R^d$ \citep[e.g.,][]{aurenhammer1998minkowski} around the atoms $\{x_1, \hdots, x_n\}$ of $\pi_2$. Recall that for a vector $w=(w_1, \hdots, w_n) \in \mathbb R^n$ that assigns to each $x_i$ a weight $w_i$, the additively weighted Voronoi tessellation is the set of cells
    \[
    \text{Vor}^w(i) = \big\{x \in \mathbb{R}^d :  \|x-x_i\|-w_i \leqslant \|x-x_j\|-w_j \text{ for all }j\neq i\big\}, \quad i=1, \hdots, n.
    \]
    Now, according to \citet[][Theorem 2 and Theorem 3]{Hartmann2020} (see also \citealp{geiss2013optimally}), there exists in this semi-discrete setting a $\pi_1$-almost surely unique transport map $T^{\star}$ such that $W_1(\pi_1,\pi_2)=\int_{\mathbb R^d} \|x-T^{\star}(x)\|{\rm d}\pi_1(x)$. Noting that the intersection of two boundaries has Lebesgue measure zero (and thus $\pi_1$-measure zero by absolute continuity), this optimal function $T^{\star}$ is defined $\lambda_d$-almost surely and has the form
    \begin{equation}\label{eq:optimalmapHartman}
    T^{\star}(x)=\sum_{i=1}^n x_i \mathds 1\{x \in \text{Vor}^{w^{\star}}(i)\},
    \end{equation}
    where the weight vector $w^{\star}$ is {\it adapted} to $(\pi_1,\pi_2)$ in the sense that $\pi_1(\text{Vor}^{w^{\star}}(i))=\alpha_i$ for all $i \in \{1, \hdots, n\}$. The existence of such an adapted vector is guaranteed by Theorem 3 of \citealp{Hartmann2020}, who also provide an algorithm to compute it.
    
    Returning to the WGAN problem, it seems natural to consider the semi-discrete setting with $\pi_1=G_{\sharp  U}$ and $\pi_2=\mu_n$, and to describe the optimal transport maps between these two distributions in order to gain information on $W_1(G_{\sharp  U},\mu_n)$. Unfortunately, there is no reason for $G_{\sharp  U}$ to be nonatomic and, even if this is the case, it is impossible for this distribution to be absolutely continuous with respect to the Lebesgue measure as soon as $d> 1$. We therefore conclude that none of the above results can be used to characterize the infimum in~\eqref{eq:studyGANs} and that some extensions are needed. In the rest of this section, we address this issue and offer a solution in two steps. First, we prove in Proposition~\ref{proposition2} that the WGAN optimization in Problem~\eqref{eq:studyGANs} can be safely restricted to distributions $G_{\sharp  U}$ that are nonatomic. Second, we provide in Theorem~\ref{theorem2} a solution to the Monge problem under the sole assumption that $\pi_1$ is nonatomic with compact support, getting rid of the absolute continuity requirement. To the extent of our knowledge, this is the first time that such a theorem has been proved, and it therefore provides a new resource in the toolbox of optimal transport theory.
    \begin{proposition}\label{proposition2}
    Let ${\rm Lip}_K^-([0,1],\mathbb R^d)=\{G \in {\rm Lip}_K([0,1],\mathbb R^d) : G_{\sharp  U} \mbox{ is nonatomic}\}$. Then 
    \[ 
    \inf \limits_{G \in {\rm Lip}_K([0,1],\mathbb R^d)} W_1(G_{\sharp  U},\mu_n) = \inf \limits_{G \in {\rm Lip}_K^-([0,1],\mathbb R^d)} W_1(G_{\sharp  U},\mu_n).
    \]
    \end{proposition}
    In the following, for $w \in \mathbb{R}^n$ and $i\in \{1,\hdots  ,n\}$, we let $\text{Vor}^w(i)$ be the $i$-th weighted Voronoi cell associated with the sample $X_1, \hdots, X_n$. We denote by $\partial \text{Vor}^w(i)$ the boundary of $\text{Vor}^w(i)$ and let $\text{Vor}^w(i)^\circ = \text{Vor}^w(i) \setminus \partial \text{Vor}^w(i)$ be its interior. For any $p\in \{1,\hdots  ,n\}$ and any set $\{j_1,\hdots  ,j_p\}$ where the $j_k$'s are all different and in $\{1,\hdots ,n\}^p$, we let
\begin{equation}\label{gammai1...ip}
        \Gamma_{j_1\hdots  j_p}^w = \bigcap \limits_{k=1}^p \text{Vor}^{w}(j_k) \setminus \Big( \bigcup \limits_{\ell \notin \{j_1,\hdots ,j_p\}}  \text{Vor}^{w}(\ell) \Big).
    \end{equation}
    Observe that each set $\Gamma_{j_1\hdots j_p}^w$ above is the subset of the common boundary of the Voronoi cells $\text{Vor}^{w}(j_1), \hdots, \text{Vor}^{w}(j_p)$ that has no intersection with any other cell $\text{Vor}^w(l)$, for all $l \notin \{j_1,\hdots  ,j_p\}$. Note also that together, the $\Gamma_{j_1\hdots j_p}^w$ (for all $p$ and all different sets $\{j_1,\hdots  ,j_p\}$) form a partition of the set of the boundaries of the Voronoi cells. For a given $w$, we will be interested in the class $\mathscr {\mathscr H}^w$ of functions taking values in the sample $X_1, \hdots, X_n$ defined by
    \begin{align}
    {\mathscr H}^w & =\big\{T:\mathbb{R}^d \to \mathscr \{X_1, \hdots, X_n\} \ : \ \forall x \in \text{Vor}^w(i)^\circ, T(x)=X_i \nonumber\\
    & \hspace{4.5cm}\text{and } \forall x \in \Gamma_{j_1\hdots j_p}^w, T(x) \in \{X_{j_1},\hdots  ,X_{j_p}\}\big\}.\label{Hw}
    \end{align}
    The following result states under which assumptions we can find an optimal transport map from a nonatomic probability measure $\nu$ to the empirical measure $\mu_n$. 
    \begin{proposition}\label{proposition3}
    Let $\nu$ be a probability measure on $\mathbb{R}^d$ with finite first moment. If there exists $w^{\star}\in \mathbb{R}^n$ and $T^{\star}\in \mathscr H^{w^{\star}}$ such that $T_{\sharp  \nu}^{\star} = \mu_n$, then $T^{\star}$ is an optimal transport map from $\nu$ to $\mu_n$.
    \end{proposition}
    We deduce from Proposition~\ref{proposition3} that in order to state the existence of an optimal transport map, it is enough to show that there exist $w^{\star}\in \mathbb{R}^n$ and $T^{\star}\in \mathscr H^{w^{\star}}$ such that $T_{\sharp  \nu}^{\star} = \mu_n$. This result plays a key role in the proof of the next theorem, which guarantees the existence of an optimal transport map between {\it any} nonatomic probability measure $\nu$ (so, non necessarily absolutely continuous with respect to the Lebesgue measure) and the empirical measure $\mu_n$. It should be stressed that Theorem~\ref{theorem2} also holds if the empirical measure $\mu_n$ is replaced by a more general discrete measure, with a finite number of atoms. The adaptation is easy and is left to the reader.
    \begin{theorem}\label{theorem2}
    Let $\nu$ be a nonatomic probability measure on $\mathbb{R}^d$ with compact support. Then there exists an optimal transport map from $\nu$ to $\mu_n$, which is defined $\lambda_d$-almost everywhere by
    \[
    T^{\star}(x)=\sum_{i=1}^n X_i \mathds 1\{x \in {\rm Vor}^{w^{\star}}(i)\},
    \]
    for some $w^{\star}\in \mathbb R^n$.
    \end{theorem}
\begin{proof}
    Let $K \subseteq \mathbb{R}^d$ be the compact support of $\nu$. For $\varepsilon \in (0,1]$, we let $\nu_\varepsilon$ be the probability measure on $\mathbb{R}^d$ defined for any Borel subset $A$ by
    \[
    \nu_\varepsilon(A) = \int_{\mathbb{R}^d} \frac{\lambda_d(A\cap B(x,\varepsilon))}{\lambda_d(B(x,\varepsilon))}{\rm d}\nu(x),
    \]
    where $B(x,\varepsilon)$ stands for the closed ball centered at $x$ of radius $\varepsilon$. Observe that $\nu_\varepsilon$ has compact support $K_\varepsilon$, where 
    \[
    K_\varepsilon = \{x \in \mathbb{R}^d :  \exists z \in K \text{ such that } \|x-z\|\leqslant \varepsilon\}.
    \]
    Since, for any Borel subset $A$ such that $\lambda_d(A)=0$ one has $\nu_\varepsilon(A)=0$, we see that $\nu_\varepsilon$ is absolutely continuous with respect to the Lebesgue measure. Thus, according to \citet[][]{Hartmann2020}, there exists $w_\varepsilon=(w_{\varepsilon_1}, \hdots,w_{\varepsilon_n})  \in \mathbb{R}^n$ solution to the Monge problem between $\nu_\varepsilon$ and $\mu_n$. In particular, for each $i \in \{1,\hdots  ,n\}$, $\nu_\varepsilon(\text{Vor}^{w_\varepsilon}(i))=\frac{1}{n}$.
    
    Clearly, adding a constant to each $w_{\varepsilon_i}$ does not change the definition of the cells. Thus, in the sequel, it is assumed that $w_{\varepsilon_1}=0$. Let $C=\max \limits_{i =1,\hdots  ,n} \max \limits_{z\in K_\varepsilon} \ \|X_i-z\|$. If there exists $i\in \{2,\hdots  ,n\}$ such that $w_{\varepsilon_i}> C$, then $\text{Vor}^{w_\varepsilon}(1) \cap K_\varepsilon = \emptyset$. This is not possible since $\nu_\varepsilon(\text{Vor}^{w_\varepsilon}(1))= \frac{1}{n}$. Likewise, if $w_{\varepsilon_i}< -C$, then $\text{Vor}^{w_\varepsilon}(i) \cap K_\varepsilon = \emptyset$. Therefore, we may only consider $w_\varepsilon$'s such that $\|w_\varepsilon\|_\infty \leqslant C$, where $\|\cdot\|_{\infty}$ stands for the supremum norm on $\mathbb R^n$.
    
    Consider now the sequence $(w_{1/m})_{m\in \mathbb{N}^{\star}}$, which, as we have seen, takes its values in a compact set. Thus, there exists a subsequence $(w_{1/\varphi(m)})_{m\in \mathbb{N}^{\star}}$ that converges to some $w^{\star} \in \mathbb{R}^n$. As for now, to lighten the notation, we let $\psi(m) = 1/\varphi(m)$ for all $m\in{\mathbb N}^\star$. Clearly, we have
     \begin{align}
     \nu_{\psi(m)}(\text{Vor}^{w_{\psi(m)}}(i)) & = \int_{\mathbb{R}^d} \mathds{1}\{x \in \text{Vor}^{w^{\star}}(i)^c  \} \frac{\lambda_d\big(\text{Vor}^{w_{\psi(m)}}(i)\cap B(x,\psi(m))\big)}{\lambda_d\big(B(x,\psi(m))\big)}{\rm d}\nu(x)\nonumber\\
    & \quad + \int_{\mathbb{R}^d} \mathds{1}\{x \in \text{Vor}^{w^{\star}}(i)^\circ  \} \frac{\lambda_d\big(\text{Vor}^{w_{\psi(m)}}(i)\cap B(x,\psi(m))\big)}{\lambda_d\big(B(x,\psi(m))\big)}{\rm d}\nu(x)\nonumber\\
    & \quad + \int_{\mathbb{R}^d} \mathds{1}\{x \in \partial \text{Vor}^{w^{\star}}(i)  \} \frac{\lambda_d\big(\text{Vor}^{w_{\psi(m)}}(i)\cap B(x,\psi(m))\big)}{\lambda_d\big(B(x,\psi(m))\big)}{\rm d}\nu(x).\label{3I}
    \end{align}
    For $x \in  \text{Vor}^{w^{\star}}(i)^c$ and all $m$ large enough, \[\frac{\lambda_d\big(\text{Vor}^{w_{\psi(m)}}(i) \cap B(x,\psi(m))\big)}{\lambda_d \big( B(x,\psi(m))\big)}=0.
    \]
    Therefore, by dominated convergence, the first integral in identity~\eqref{3I} tends to $0$ as $m$ tends to infinity. Similarly, for $x \in  \text{Vor}^{w^{\star}}(i)^\circ$ and all $m$ large enough, 
    \[
    \frac{\lambda_d\big(\text{Vor}^{w_{\psi(m)}}(i) \cap B(x,\psi(m))\big)}{\lambda_d \big( B(x,\psi(m))\big)}=1.
    \]
    Thus, by dominated convergence, the second integral tends to $\nu(\text{Vor}^{w^{\star}}(i)^\circ )$. The analysis of the third integral in~\eqref{3I} is more delicate and is done by carefully studying each part of the boundary $\partial \text{Vor}^{w^{\star}}(i)$. For any $p\in \{1,\hdots  ,n\}$ and $j_1,\hdots  ,j_p$ all different, let
    \begin{equation*}
        \Gamma_{j_1\hdots  j_p}^{w^\star} = \bigcap \limits_{k=1}^p \text{Vor}^{w^\star}(j_k) \setminus \Big( \bigcup \limits_{\ell \notin \{j_1,\hdots ,j_p\}}  \text{Vor}^{w^\star}(\ell) \Big).
    \end{equation*}
    Using the notation
    \[
    \alpha_{j_1\hdots  j_p}(i)_{\psi(m)} := \int_{\mathbb{R}^d}\mathds{1}\{ x \in \Gamma_{j_1\hdots j_p}^{w^\star} \} \frac{\lambda_d\big(\text{Vor}^{w_{\psi(m)}}(i)\cap B(x,\psi(m))\big)}{\lambda_d\big(B(x,\psi(m))\big)}{\rm d}\nu(x),
    \]
    we see that
    \[\nu(\Gamma_{j_1\hdots j_p}^{w^\star})=\sum \limits_{i=1}^n \alpha_{j_1\hdots  j_p}(i)_{\psi(m)}.
    \]
    Observe that for $i \notin \{j_1,\hdots ,j_p\}$,  $\alpha_{j_1\hdots j_p}(i)_{\psi(m)} \to 0$ as $m$ tends to infinity, since for all $ x\in  \Gamma_{j_1\hdots j_p}^{w^\star}$,
    \[
    \text{Vor}^{w_{\psi(m)}}(i) \cap B(x,\psi(m)) =\emptyset
    \]
    for all $m$ large enough. Moreover, $\alpha_{j_1\hdots j_p}(j_1)_{\psi(m)} \in [0,1/n]$. Thus, we can extract a subsequence $(\psi_1(m))_{m\in \mathbb{N}^{\star}}$ such that $\alpha_{j_1\hdots j_p}(j_1)_{\psi_1(m)}$ converges to some $\alpha_{j_1\hdots j_p}(j_1)$ as $m$ tends to infinity. Likewise, we can extract a subsequence $\psi_{12}(m)$ such that  $\alpha_{j_1\hdots j_p}(j_2)_{\psi_{12}(m)}$ converges to some $\alpha_{j_1\hdots j_p}(j_2)$. Repeating the same procedure, we obtain a subsequence $\psi_{1\hdots p}(m)$ such that each $\alpha_{j_1\hdots j_p}(j_k)_{\psi_{1\hdots p}(m)}$, $k \in \{1,\hdots,p\}$, converges to $\alpha_{j_1\hdots j_p}(j_k)$ as $m$ tends to infinity. In particular, 
    \[\nu(\Gamma_{j_1\hdots j_p}^{w\star})=\sum \limits_{k=1}^p \alpha_{j_1\hdots j_p}(j_k).
    \]
    
    Starting from the subsequence $\psi_{1\hdots p}(m)$, we may repeat the previous exercise for all sets $\Gamma_{j_1\hdots j_p}^{w^\star}$, where $p\in \{1,\hdots ,n\}$, $j_1,\hdots ,j_p$ are all different, and the subsequence $\Psi(m)$ is such that any $\alpha_{j_1\hdots j_p}(j_k)_{\Psi(m)}$ converges to some $\alpha_{j_1\hdots j_p}(j_k)$, for all $j_1,\hdots ,j_p$. We conclude that there exists a subsequence of the third integral in~\eqref{3I} that converges to $\sum \limits_{p=1}^n\sum \limits_{j_1,\hdots, j_p} \alpha_{j_1\hdots j_p}(i)$. Since, for $i\in \{1,\hdots ,n\}$, $\nu_{\Psi(m)}(\text{Vor}^{w_{\Psi(m)}}(i))=\frac{1}{n}$ for all $m$, we have, letting $m\to\infty$,
    \begin{equation}
    \nu(\text{Vor}^{w^{\star}}(i)^\circ)+ \sum \limits_{p=1}^n\sum \limits_{j_1,\hdots,j_p} \alpha_{j_1\hdots j_p}(i)=\frac{1}{n}.\label{transportmap}
    \end{equation}
    Now, cut each $\Gamma_{j_1\hdots j_p}^{w^\star}$ into arbitrarily $p$ disjoints parts $A_{j_1\hdots j_p}(j_k)$ such that $\nu(A_{j_1\hdots j_p}(j_k)) = \alpha_{j_1\hdots j_p}(j_k)$ (this is always possible since $\nu$ is nonatomic). Let $T^{\star} : \mathbb{R}^d \to \{X_1, \hdots, X_n\}$ be defined by
    \begin{equation*} 
    T^{\star}(x) = \left\{
    \begin{array}{ll}
     X_i & \text{for } x \in \text{Vor}^{w^{\star}}(i)^\circ \\
    X_{j_k} & \text{for }x \in A_{j_1\hdots j_p}(j_k).
    \end{array}
    \right.
    \end{equation*} 
    Then $T^{\star} \in \mathscr H^{w^{\star}}$ and, by identity~\eqref{transportmap}, $T^{\star}_{\sharp  \nu}=\mu_n$. This, together with Proposition~\ref{proposition3}, concludes the proof of the theorem.
\end{proof}
While the expression of the optimal transport map as given in Theorem~\ref{theorem2} (for nonatomic source measure) and the one from \citet{Hartmann2020}, as recalled in equation \eqref{eq:optimalmapHartman} (for absolutely continuous source measure), are the same, there is a significant difference in their definition at the boundaries of the cells. Indeed, these boundaries have Lebesgue measure zero. Therefore, when the source measure is absolutely continuous, the optimal mapping can take any values at the boundaries. However, when the source is assumed to be only nonatomic, the boundaries may have strictly positive measure. Consequently, the choice of values for an optimal mapping at the boundaries should be made with care. In the proof of Theorem~\ref{theorem2}, we show that for each set $\Gamma_{j_1\hdots j_p}^{w^\star}$ and each $i\in \{1,\hdots n\}$, there exist weights $\alpha_{j_1\hdots j_p}(i)\geq 0$ such that
 \[\nu(\Gamma_{j_1\hdots j_p}^{w\star})=\sum \limits_{k=1}^p \alpha_{j_1\hdots j_p}(j_k)
    \]
and 
\[
\nu(\text{Vor}^{w^{\star}}(i)^\circ)+ \sum \limits_{p=1}^n\sum \limits_{j_1,\hdots,j_p} \alpha_{j_1\hdots j_p}(i)=\frac{1}{n}.
\]
Then, cutting each subset $\Gamma_{j_1\hdots j_p}^{w^\star}$ into $p$ arbitrarily disjoint parts $A_{j_1\hdots j_p}(j_k)$ such that $\nu(A_{j_1\hdots j_p}(j_k)) = \alpha_{j_1\hdots j_p}(j_k)$ and defining $T^{\star} : \mathbb{R}^d \to \{X_1, \hdots, X_n\}$ by
    \begin{equation*} 
    T^{\star}(x) = \left\{
    \begin{array}{ll}
     X_i & \text{for } x \in \text{Vor}^{w^{\star}}(i)^\circ \\
    X_{j_k} & \text{for }x \in A_{j_1\hdots j_p}(j_k),
    \end{array}
    \right.
    \end{equation*} 
we obtain that $T^{\star}$ is an optimal transport map.

Combining the result of Proposition~\ref{proposition2} with Theorem~\ref{theorem2}, we can now properly characterize the $1$-Wasserstein distance between the optimal distribution $G_{\sharp U}$ derived from ${\rm Lip}_K([0,1],\mathbb R^d)$ and the empirical measure $\mu_n$. 
    \section{Finite sample analysis in a multivariate output space} \label{sec:multi_dim_observations}
    We are now ready to analyze Problem~\eqref{eq:studyGANs} in the more realistic multivariate setting. In the remainder of the section, it is therefore assumed that the observations $X_1, \hdots, X_n$ take their values in $\mathbb R^d$ with $d>1$, while the latent space still has dimension 1. Following the schema of Section~\ref{1DO}, we first define a candidate function ${\widehat G}_K^\star \in \text{Lip}_K ([0,1],\mathbb R^d)$, compute $W_1({\widehat G}_{K\sharp  U}^\star,\mu_n)$ in Proposition~\ref{proposition4}, and finally show in Theorem~\ref{theorem2} that ${\widehat G}^{\star}_{K}$ solves Problem~\eqref{eq:studyGANs} in a large subset of $\text{Lip}_K ([0,1],\mathbb R^d)$.
    Finally, similarly to Section~\ref{1DO}, we conclude with an asymptotic analysis of $W_1({\widehat G}_{K\sharp  U}^\star,\mu_n)$ when $K$ is a function of the sample size $n$.
    
    \subsection{Construction of ${\widehat G}^{\star}_{K}$}
    In the multivariate setting, the shortest path among the $n$ data samples $X_i$, $i \in \{1, \hdots, n\}$, plays an essential role in the definition of the optimal ${\widehat G}_{K}^{\star}$. The set of paths connecting all data points $X_1, \hdots, X_n$, while minimizing the sum of the squared Euclidean distances, is defined as follows:
    \begin{equation}\label{shortestpath}
    (k,\sigma) \in \argmin \Big\{ \sum \limits_{i=1}^{n+k'-1} \|X_{\sigma'(i+1)}-X_{\sigma'(i)}\|^2 : k' \in \mathbb{N}, \sigma' \in \mathscr S_{k'}\Big\},
    \end{equation}
    where $\mathscr S_{k'}$ denotes the set of all discrete functions $\sigma':\{1,\hdots,n+k'\} \to \{1, \hdots, n\}$ such that $\sigma'(\{1,\hdots,n+k'\}) = \{1,\hdots,n\}$ and $\sigma'(j)\neq \sigma'(j+1)$. Note that such a pair $(k,\sigma)$ may not be unique and keep in mind that $\sigma$ depends on $k$. 
    
    An important remark is that any shortest path (with a squared norm) is allowed to visit several times the same point (i.e., $k>0$). This is a consequence of the fact that the squared Euclidean distance does not verify the triangle inequality. Note also that the number of visits to the point $X_i$ is equal to $|\sigma^{-1}(i)|$. An illustration of four shortest paths in dimension $2$ is provided in Figure~\ref{fig:squared_euc_distance}. On the top, every single data point is visited once (i.e., $k=0$ in formula~\eqref{shortestpath}), contrary to the two examples in the bottom, where a point is visited twice (i.e., $k=1$). 
    
    \begin{figure}[h]
       \centering
        \subfloat[Shortest path with $n=4$, $k=0$ in~\eqref{shortestpath}.]
        {
            \includegraphics[width=0.425\textwidth]{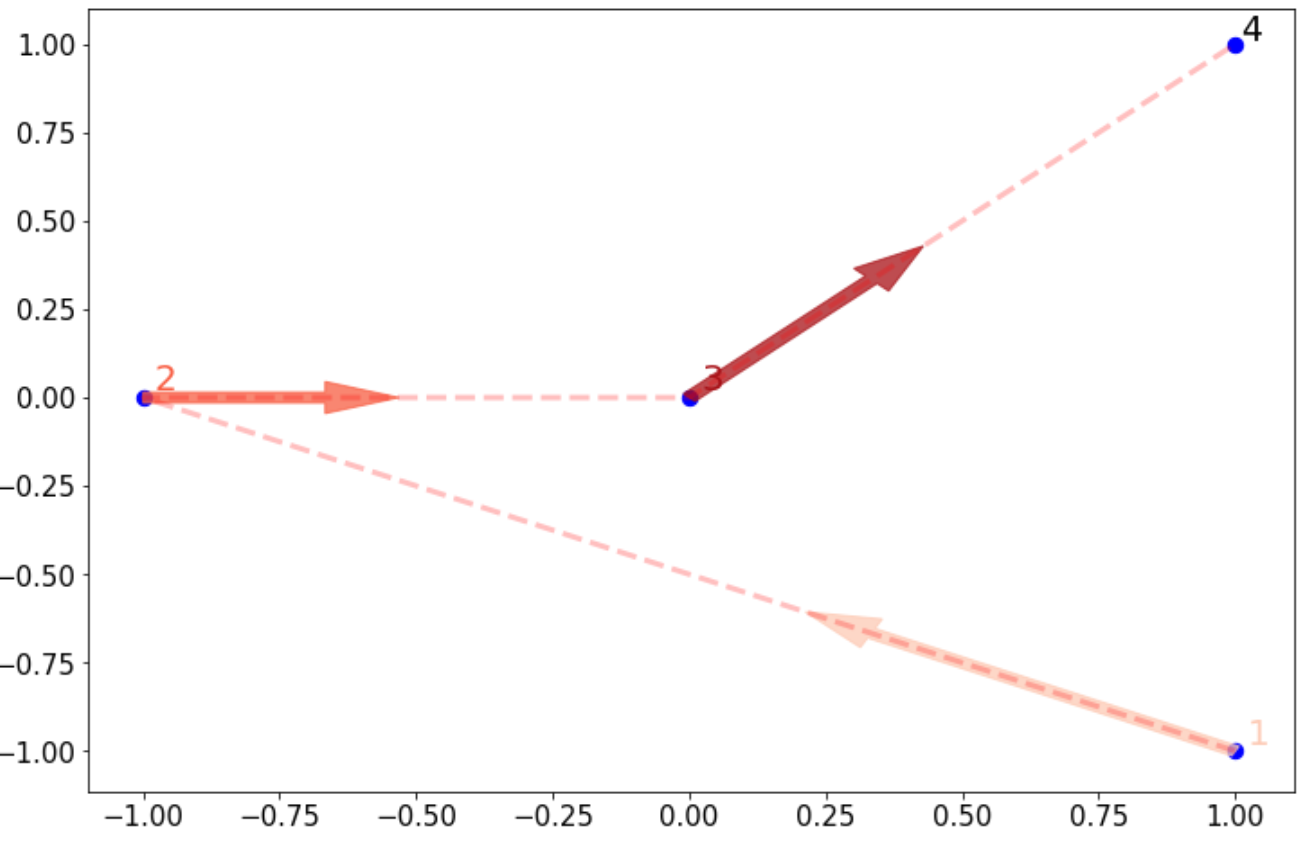}
        }\hspace{1cm}
        \subfloat[Shortest path with $n=7$, $k=0$ in~\eqref{shortestpath}.]
        {
            \includegraphics[width=0.43\textwidth]{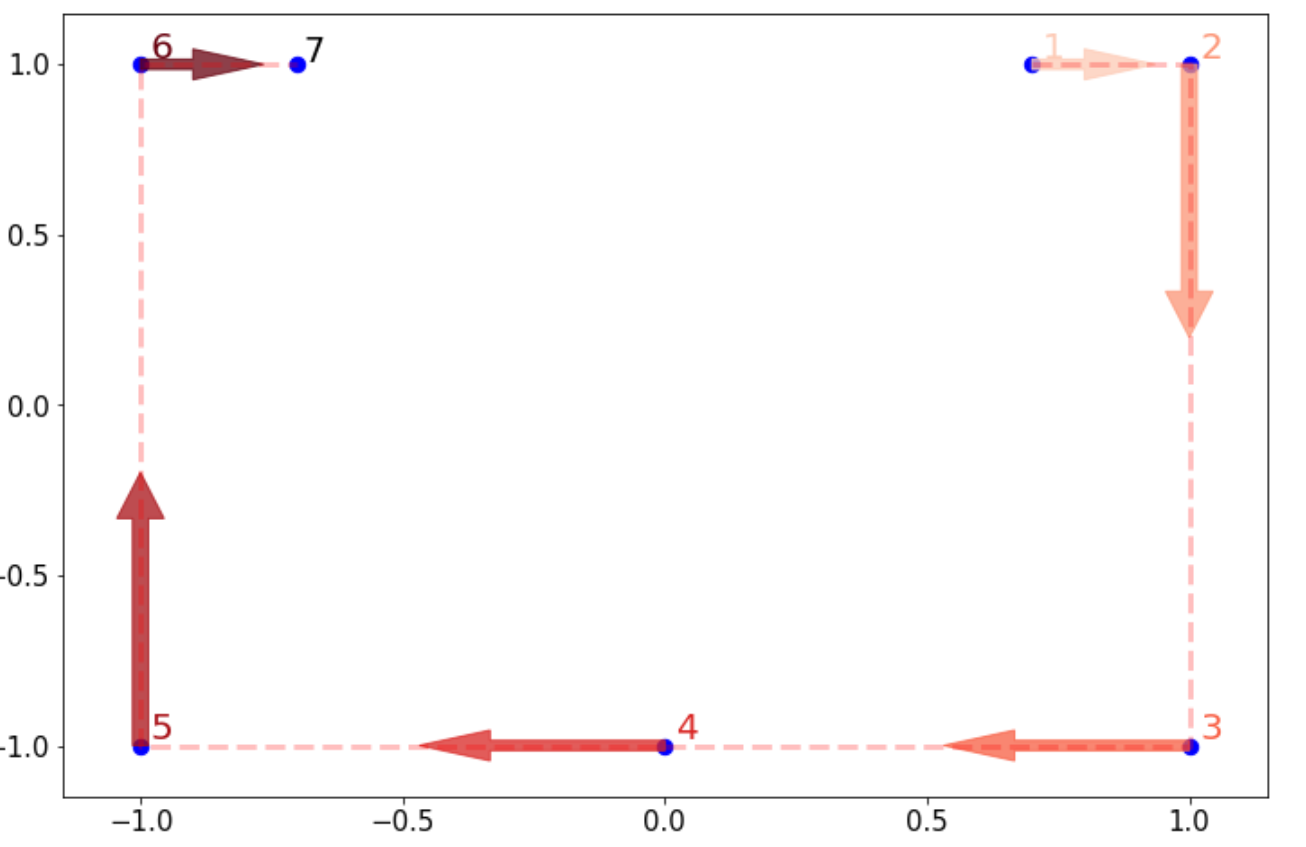}
        }\hfill
        \subfloat[Shortest path with $n=6$, $k=1$ in~\eqref{shortestpath}.]
        {
            \includegraphics[width=0.425\textwidth]{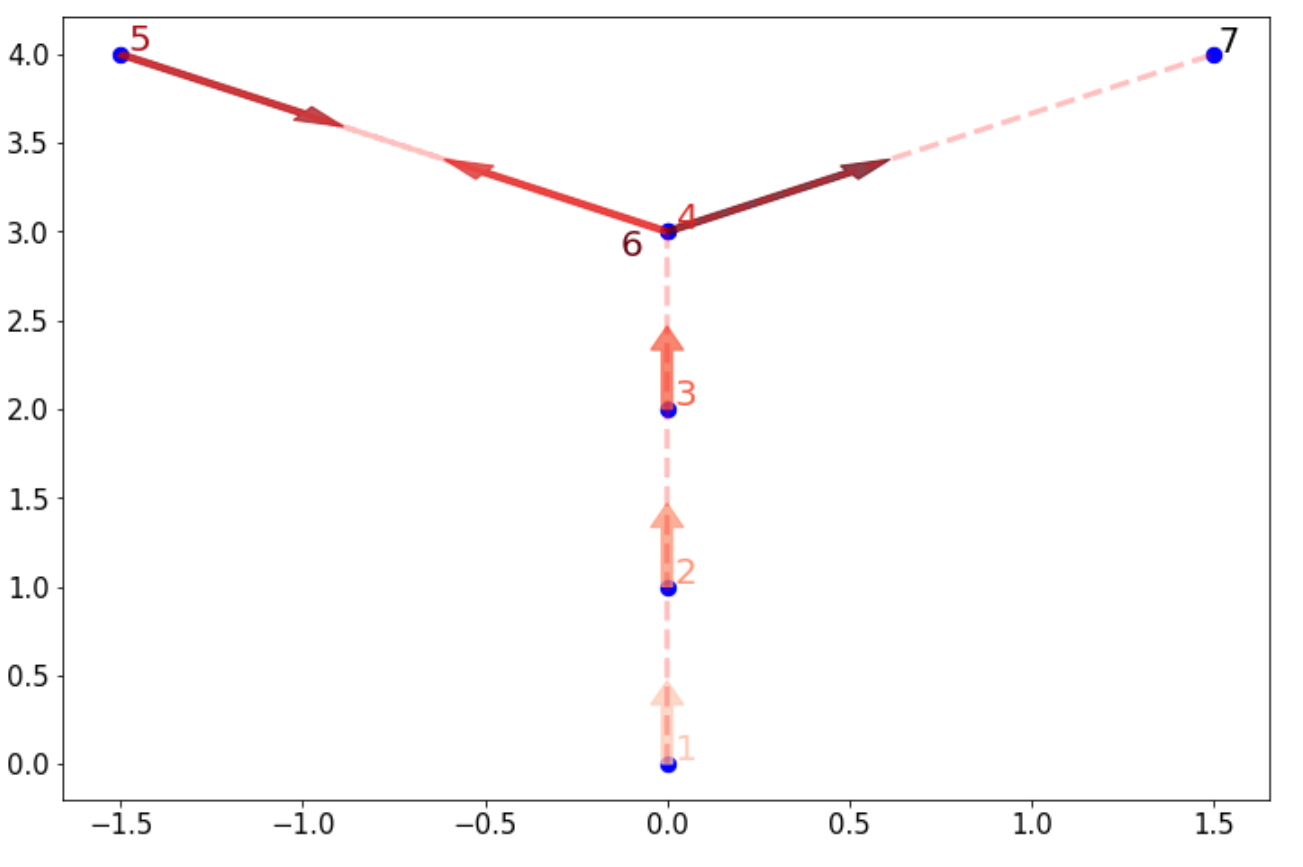}
        }\hspace{1cm}
        \subfloat[Shortest path with $n=15$, $k=1$ in~\eqref{shortestpath}.]
        {
            \includegraphics[width=0.43\textwidth]{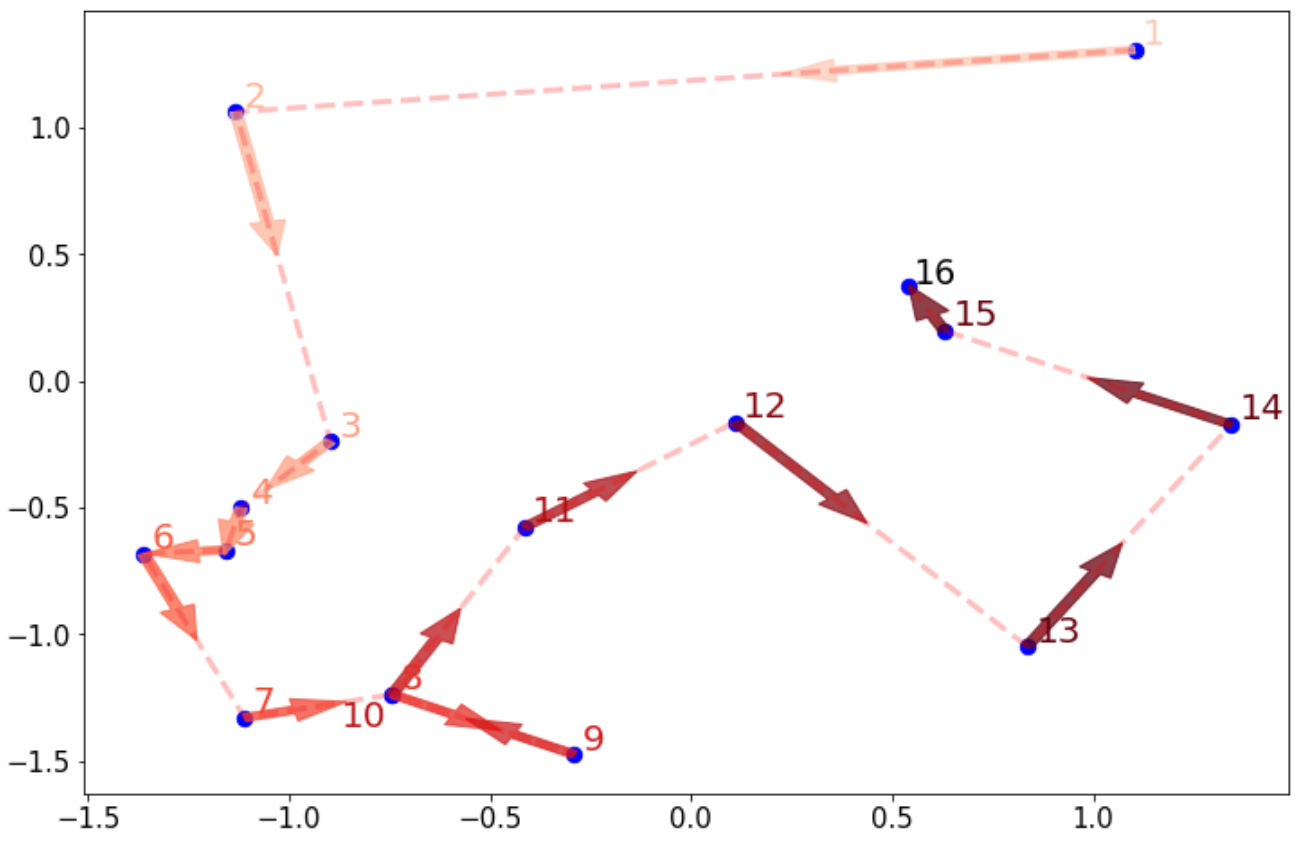}
        }
        \caption{Examples of shortest paths in dimension $2$, with the squared Euclidean distance. \label{fig:squared_euc_distance}}
    \end{figure}
    
    Let us now provide some intuition on the way the optimal function ${\widehat G}_K^\star: [0,1] \to \mathbb{R}^d$ is obtained. In a nutshell, this function strictly follows $\sigma$, one of the optimal paths in~\eqref{shortestpath}. Thus, there exist some $0 \leqslant t_1 < \cdots < t_{n+k} \leqslant 1$ such that ${\widehat G}_K^\star(t_j) = X_{\sigma(j)}$, $j \in \{1, \hdots, n+k\}$. Since the optimal path (and therefore ${\widehat G}^{\star}_{K}$) can visit several times each sample point $X_i$, we need to take into account how long ${\widehat G}^{\star}_{K}$ stays constant at $X_i$, whenever it visits this data point. This period of time is denoted by $\varphi(i)$ and chosen to be equal to
    \[
    \varphi(i) =\frac{1}{|\sigma^{-1}(i)|} \Big( \frac{1}{n}- \sum \limits_{j\in \sigma^{-1}(i)} \frac{1}{2K}(\|X_{\sigma(j-1)}-X_i\|+\|X_{\sigma(j+1)}-X_i\|)\Big)
    \]
    (by convention, $X_{\sigma(0)} = X_{\sigma(1)}$ and $X_{\sigma(n+k+1)} = X_{\sigma(n+k)}$). The quantity $|\sigma^{-1}(i)|\times \varphi(i)$ thus corresponds to the total measure of the atoms $X_i$ under the distribution ${\widehat G}^{\star}_{K\sharp  U}$. Finally, for any $j \in \{1, \hdots, n+k \}$, we let 
    \[
        V_j  = \sum \limits_{\ell=1}^{j-1} \Big(\varphi(\sigma(\ell)) + \frac{\|X_{\sigma(\ell+1)}-X_{\sigma(\ell)}\|}{K}\Big) = V_{j-1} + \varphi(\sigma(j-1)) + \frac{\|X_{\sigma(j)}-X_{\sigma(j-1)}\|}{K}.
    \] 
    This quantity $V_j$ is more complicated to grasp, but intuitively, it corresponds to the time steps where the function ${\widehat G}^{\star}_{K}$ has arrived on a sample point $X_{\sigma(j)}$ and will pause at $X_{\sigma(j)}$ for a time equal to $\varphi(\sigma(j))$. 
    
    A visual explanation of the construction mechanism of ${\widehat G}^{\star}_{K}$ is depicted in Figure~\ref{fig:Gstar_explained}. The top shows the trajectory of ${\widehat G}^{\star}_{K}$ following an optimal path $\sigma$ in~\eqref{shortestpath}. The bottom shows the succession of time steps at which $\widehat{G}^{\star}_{K}$ passes from one point to another.
     \begin{figure}[h]
        \centering
        {
            \includegraphics[width=0.80\linewidth]{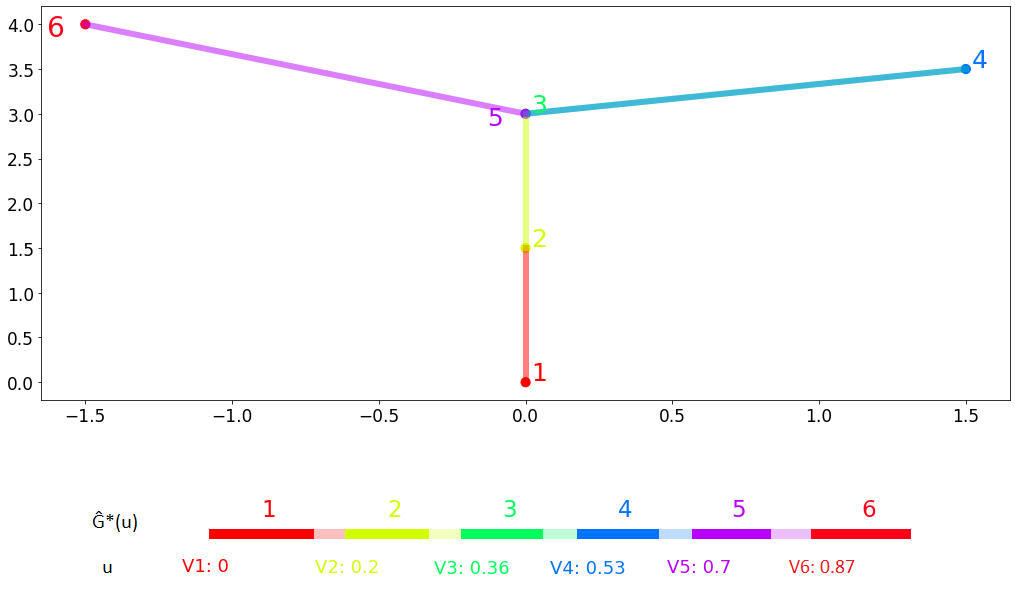}
        }
        \caption{${\widehat G}^{\star}_{K}$ explained. For each data point $X_{\sigma(j)}$ (embedded by a specific color), the bold part of the interval symbolizes the time ${\widehat G}^{\star}_{K}$ is equal to $X_{\sigma(j)}$, while the light part refers to the jump from $X_{\sigma(j)}$ to $X_{\sigma(j+1)}$. Note that ${\widehat G}^{\star}_{K}$ follows the optimal path under the squared Euclidean norm. 
        \label{fig:Gstar_explained}}
    \end{figure}
    
    Equipped with this notation, we may now properly define the function ${\widehat G}_K^\star :[0,1]\to \mathbb{R}^d$, as follows:
    \begin{equation} 
    \label{Gstar2}
    {\widehat G}^{\star}_{K}(u) = \left\{
    \begin{array}{ll}
     X_{\sigma(j)} &   \text{if } u \in [V_j,V_j + \varphi(\sigma(j))] \smallskip\\
        &  \text{for } 1\leqslant j\leqslant n+k \medskip\\
        X_{\sigma(j)} + \big(u- (V_j + \varphi(\sigma(j)))\big)K \frac{X_{\sigma(j+1)}-X_{\sigma(j)}}{\|X_{\sigma(j+1)}-X_{\sigma(j)}\|} & \text{if }  u \in [V_j+ \varphi(\sigma(j)),V_{j+1}] \smallskip\\
        &  \text{for } 1\leqslant j\leqslant n+k-1.
    \end{array}
    \right.
    \end{equation} 
    Observe that the function ${\widehat G}_{K}^{\star}$ is well-defined as soon as 
    \[K\geqslant n \max \limits_{i=1,\hdots,n} \sum \limits_{j\in \sigma^{-1}(i)} \frac{1}{2}(\|X_{\sigma(j-1)}-X_i\|+\|X_{\sigma(j+1)}-X_i\|),\] 
    and that it belongs to $\text{Lip}_K([0,1],\mathbb R^d)$. Making the connection with the univariate case of Section \ref{1DO}, we have that if $d=1$, then each point is visited only once, so that $k=0$ and, for each $i \in \{1, \hdots, n\}$, $X_{\sigma(i)}=X_{(i)}$ (or $X_{\sigma(i)}=X_{(n-i+1)}$). Besides, $|\sigma^{-1}(i)|=1$ and  $\varphi(i)=1/n - (X_{(i+1)} - X_{(i-1)})/(2K)$. We thus recover the univariate function defined in \eqref{Gstar}.
    \subsection{Optimality properties}
    In this subsection, we first compute the $1$-Wasserstein distance between ${\widehat G}^{\star}_{K\sharp  U}$ and $\mu_n$, and then prove that this value minimizes Problem~\eqref{eq:studyGANs} in the multivariate setting, under a mild assumption.
    \begin{proposition} 
    \label{proposition4}
    Assume that 
    \[K\geqslant n \max \limits_{i=1,\hdots,n} \sum \limits_{j\in \sigma^{-1}(i)} \frac{1}{2}(\|X_{\sigma(j-1)}-X_i\|+\|X_{\sigma(j+1)}-X_i\|),\] 
    and let ${\widehat G}_{K}^{\star} \in {\rm Lip}_K([0,1],\mathbb R^d)$ be defined in~\eqref{Gstar2}. Then 
    \[
    W_1({\widehat G}^{\star}_{K\sharp  U},\mu_n) =   \frac{1}{4K} \sum \limits_{j=1}^{n+k-1} \|X_{\sigma(j+1)}-X_{\sigma(j)}\|^2.
    \]
    \end{proposition}
    By construction, it is clear that ${\widehat G}_{K}^{\star}$ visits each data point, following the optimal path $\sigma$. The proof of Proposition \ref{proposition4} reveals that $\widehat {G}^{\star}_{K}$ spends a time $1/n$ (in terms of Lebesgue measure) in each ``standard'' Voronoi cell $\text{Vor}(i)$, that is
        \[
        \text{Vor}(i) = \big\{x \in \mathbb{R}^d :  \|x-X_i\| \leqslant \|x-X_j\|\text{ for all }j\neq i\big\}, \quad i=1, \hdots, n.
        \]
    These cells correspond to additively weighted Voronoi cells with weight $w=(0, \hdots, 0)$. We define in the same way $\Gamma_{j_1\hdots j_p}^0$ and $\mathscr H^{0}$ as in~\eqref{gammai1...ip} and~\eqref{Hw}, respectively.
    
    In the remainder of the subsection, we prove the optimality of $\widehat G^\star_K$ on a subset smaller than $\text{Lip}_K([0,1],\mathbb R^d)$. This subset is denoted by $\text{Lip}^{\circ}_K([0,1],\mathbb R^d)$ and is defined below. Recall that for any $G \in \text{Lip}_K([0,1],\mathbb R^d)$ such that $G_{\sharp  U}$ is nonatomic, there exists according to Theorem~\ref{theorem2} a weight $w\in \mathbb{R}^n$ and an optimal transport map $T^w$ from $G_{\sharp  U}$ to $\mu_n$ such that
    \[
    W_1(G_{\sharp  U},\mu_n)  = \int_{\mathbb R^d} \|x-T^w(x)\|{\rm d}G_{\sharp  U}(x) = \int_0^1 \|G(u)-T^w(G(u))\|{\rm d}u.
    \]
    \begin{definition}\label{definition1}
    Let $G \in {\rm Lip}_K([0,1],\mathbb R^d)$. We say that $G$ is in ${\rm Lip}^\circ_K([0,1],\mathbb R^d)$ if $G_{\sharp  U}$ is nonatomic and, for all $u \in [0,1]$ and all $i\in \{1,\hdots,n\}$ such that $T^w(G(u))=X_i$, there exists $v \in [0,1]$ such that $G(v)=X_i$ and $\forall x \in [u,v]$ (or $[v,u]$), $T^w(G(x))=X_i$.
    \end{definition}
    Definition~\ref{definition1} means that as soon as the function $G$ enters a weighted Voronoi cell, then it must passes through its center. Even though ${\widehat G}^{\star}_{K\sharp  U}$ has atoms, the following theorem shows that $W_1({\widehat  G}^{\star}_{K\sharp  U},\mu_n)$ achieves the infimum of Problem~\eqref{eq:studyGANs} over $\text{ Lip}_K^\circ([0,1],\mathbb R^d)$. 
    \begin{theorem}\label{theorem3}
    Assume that $K\geqslant n \max \limits_{i=1,\hdots,n} \sum \limits_{j\in \sigma^{-1}(i)} \frac{1}{2}(\|X_{\sigma(j-1)}-X_i\|+\|X_{\sigma(j+1)}-X_i\|)$, and let the function ${\widehat G}_{K}^{\star} \in {\rm Lip}_K([0,1],\mathbb R^d)$ be defined in~\eqref{Gstar2}. Then 
    \[
    W_1({\widehat G}^{\star}_{K\sharp  U},\mu_n) = \inf \limits_{G\in {\rm Lip}^{\circ}_K ([0,1],\mathbb R^d)} W_1(G_{\sharp  U},\mu_n) = \frac{1}{4K} \sum \limits_{j=1}^{n+k-1}   \|X_{\sigma(j+1)}-X_{\sigma(j)}\|^2.
    \]
    \end{theorem}
\begin{proof}
    Let $G\in \text{Lip}^\circ_K([0,1],\mathbb R^d)$.
    According to Theorem~\ref{theorem2}, there exists a weight $w\in \mathbb R^n$ and an optimal transport map $T^w$ from $G_{\sharp  U}$ to $\mu_n$. We denote by $[a_1,a_2], [a_2,a_3],\hdots,[a_{p-1},a_p]$ the intervals such that $a_1=0$, $a_p=1$, and, for all $j \in \{1,\hdots,p-1\}$, there exists $\tau(j) \in \{1,\hdots,n\}$ such that $u \in [a_j,a_{j+1}]$ implies that $T^w(G(u))=X_{\tau(j)}$ (with $\tau(j)\neq \tau(j+1)$).
    
    Using the fact that $G$ is $K$-Lipschitz and satisfies Definition~\ref{definition1}, it is easy to see that
    \[a_{j+1}-a_j \geqslant \frac{1}{K}\big(\|G(a_j)-X_{\tau(j)}\|+\|G(a_{j+1})-X_{\tau(j)}\|\big).\]
    Observe that
    \begin{align*}
    W_1(G_{\sharp  U},\mu_n) & = \int_{\mathbb R^d} \|x-T^w(x)\|{\rm d}G_{\sharp  U}(x) = \int_0^1 \|G(u)-T^w(G(u))\|{\rm d}u\\
    & = \sum \limits_{j=1}^{p-1} \int_{a_j}^{a_{j+1}} \|G(u)-X_{\tau(j)}\|{\rm d}u.
    \end{align*}
    Therefore,
    \begin{align*}
    W_1(G_{\sharp  U},\mu_n) & \geqslant \sum \limits_{j=1}^{p-1} \int_{a_j}^{a_{j}+\frac{\|G(a_{j})-X_{\tau(j)}\|}{K}}( \|G(a_j)-X_{\tau(j)}\|-K(u-a_j)){\rm d}u\\
    & \quad + \int_{a_{j+1}-\frac{\|G(a_{j+1})-X_{\tau(j)}\|}{K}}^{a_{j+1}}K\big(u-\big(a_{j+1}-\frac{\|G(a_{j+1})-X_{\tau(j)}\|}{K}\big)\big){\rm d}u\\
    & = \frac{1}{2K} \sum \limits_{j=1}^{p-1}\big(\|G(a_{j})-X_{\tau(j)}\|^2+\|G(a_{j+1})-X_{\tau(j)}\|^2\big).
    \end{align*}
    Observe that, by the triangle inequality, 
    \[
    \|G(a_{j})-X_{\tau(j-1)}\|+\|G(a_{j})-X_{\tau(j)}\|\geqslant \|X_{\tau(j-1)}-X_{\tau(j)}\|.\]
    So, 
    \[\|G(a_{j})-X_{\tau(j-1)}\|^2+\|G(a_{j})-X_{\tau(j)}\|^2\geqslant \frac{1}{2} \|X_{\tau(j-1)}-X_{\tau(j)}\|^2.\]
    Using~\eqref{shortestpath} (Main Document), we conclude that
    \[W_1(G_{\sharp  U},\mu_n)\geqslant \frac{1}{4K} \sum \limits_{j=1}^{p-1} \|X_{\tau(j-1)}-X_{\tau(j)}\|^2\geqslant W_1({\widehat G}^{\star}_{K\sharp  U},\mu_n).\]
    Therefore,
    \[
     \inf \limits_{G\in {\rm Lip}^{\circ}_K ([0,1],\mathbb R^d)} W_1(G_{\sharp  U},\mu_n) \geqslant W_1({\widehat G}^{\star}_{K\sharp  U},\mu_n).
    \]
    Finally, a slight adaptation of Proposition~\ref{proposition2} shows that 
    \[ 
    W_1({\widehat G}^{\star}_{K\sharp  U},\mu_n) \geqslant \inf \limits_{G\in {\rm Lip}^{\circ}_K ([0,1],\mathbb R^d)} W_1(G_{\sharp  U},\mu_n),
    \]
    and the theorem is proved.
\end{proof}
Note that we could not show the optimality of ${\widehat G}^{\star}_{K}$ on $\text{Lip}_K([0,1],\mathbb R^d)$. However, all the numerical experiments indicate that the generative functions $G^\theta$ output by WGANs satisfy $W_1({\widehat G}^{\star}_{K\sharp  U}, \mu_n) < W_1(G^\theta_{\sharp  U}, \mu_n)$. Consequently, restricting the set of Lipschitz continuous functions to $\text{Lip}_K^\circ([0,1],\mathbb R^d)$ might not be necessary. We leave it as an open problem to prove that ${\widehat G}^{\star}_{K}$ is indeed the infimum over the whole set $\text{Lip}_K([0,1],\mathbb R^d)$. Similarly to the univariate case, the distribution ${\widehat G}^{\star}_{K\sharp  U}$ has atoms located at the sample points $X_i$, with respective mass $|\sigma^{-1}(i)| \times \varphi(i)$. It is also noteworthy that the minimizer ${\widehat G}^{\star}_{K}$ is not necessarily unique, because there may be different paths $\sigma$ minimizing the sum of the squared Euclidean distances in~\eqref{shortestpath}. Furthermore, if $|\sigma^{-1}(i)|\geqslant 2$, one can arbitrarily choose how to split the time period $|\sigma^{-1}(i)|\times \varphi(i)={\widehat G}^{\star}_{K\sharp  U}(\{ X_i\})$ according to the different moments ${\widehat G}^{\star}_{K}$ passes by $X_i$.
    
    In order to illustrate Theorem~\ref{theorem3}, we propose in Figure~\ref{fig:dim2} a $2$-dimensional experiment that compares the $1$-Wasserstein distance $W_1({\widehat G}^{\star}_{K\sharp  U}, \mu_n)$ with the results of parametric WGANs. The generator is composed of ReLU neural networks of depth $3$ and $6$, and width $100$, while the discriminator is composed of ReLU neural networks of depth $5$ and width $100$. We train a WGAN architecture on two different configurations, $n=5$ for the first and $n=10$ for the second, both with the choice $K=50$ (compatible with the assumption on $K$ in Theorem~\ref{theorem3}). We see, as expected, that the parametric WGAN (denoted by $G^{\theta}_{\sharp  U}$) gets close to the optimal function $\widehat{G}^\star_K$. However, since neural networks lack capacity and cannot replicate all Lipschitz functions, they operate some smoothing. Finally, observe that as $n$ grows, mimicking the optimal function $\widehat{G}^\star_K$ is harder, while increasing the depth can help.
   \begin{figure}[!h]
        \centering
        \subfloat[The sample size is $n=5$ and the depth of the generator is equal to $3$. The WGAN misses the shortest path leading to a deteriorated $1$-Wasserstein distance: $W_1({\widehat G}^{\star}_{K\sharp  U},\mu_n)=0.030$ and $W_1(G^{\theta}_{\sharp  U},\mu_n)=0.286$.]
        {
            \includegraphics[width=0.425\linewidth]{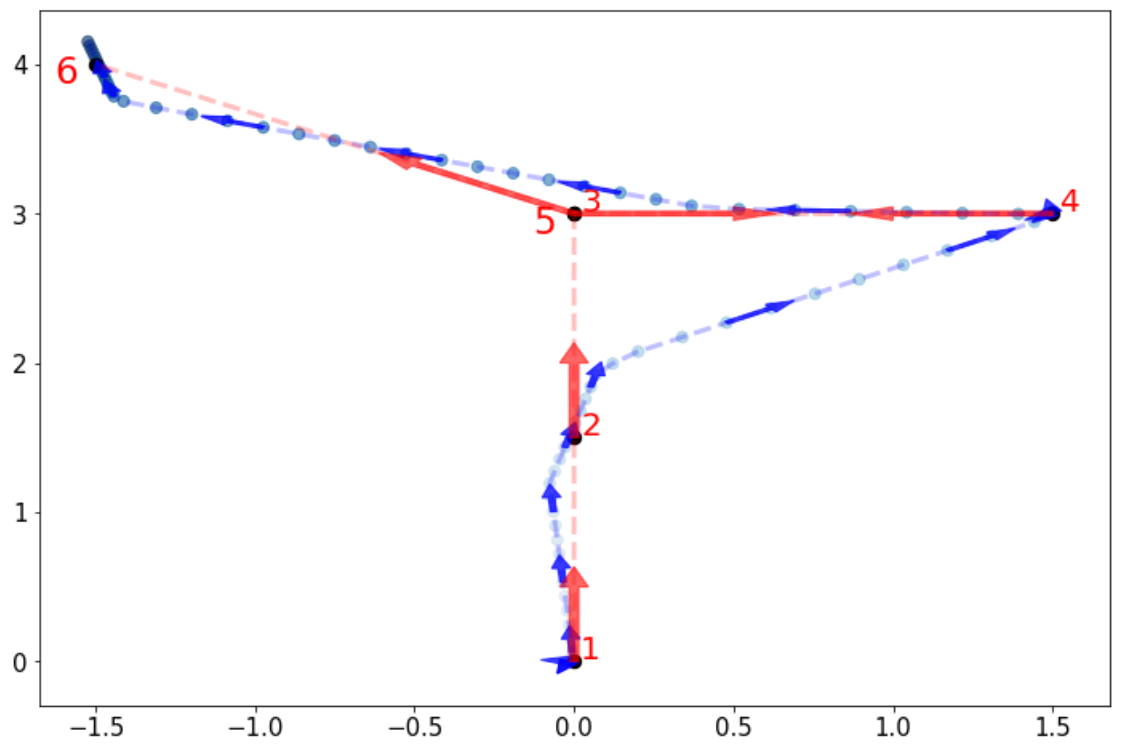}
        }\hspace{1cm}
        \subfloat[The sample size is $n=5$ and the depth of the generator is equal to $6$. The WGAN is closer to the shortest path: $W_1({\widehat G}^{\star}_{K\sharp  U},\mu_n)=0.018$ and $W_1(G^{\theta}_{\sharp  U},\mu_n)=0.174$.]
        {
            \includegraphics[width=0.425\linewidth]{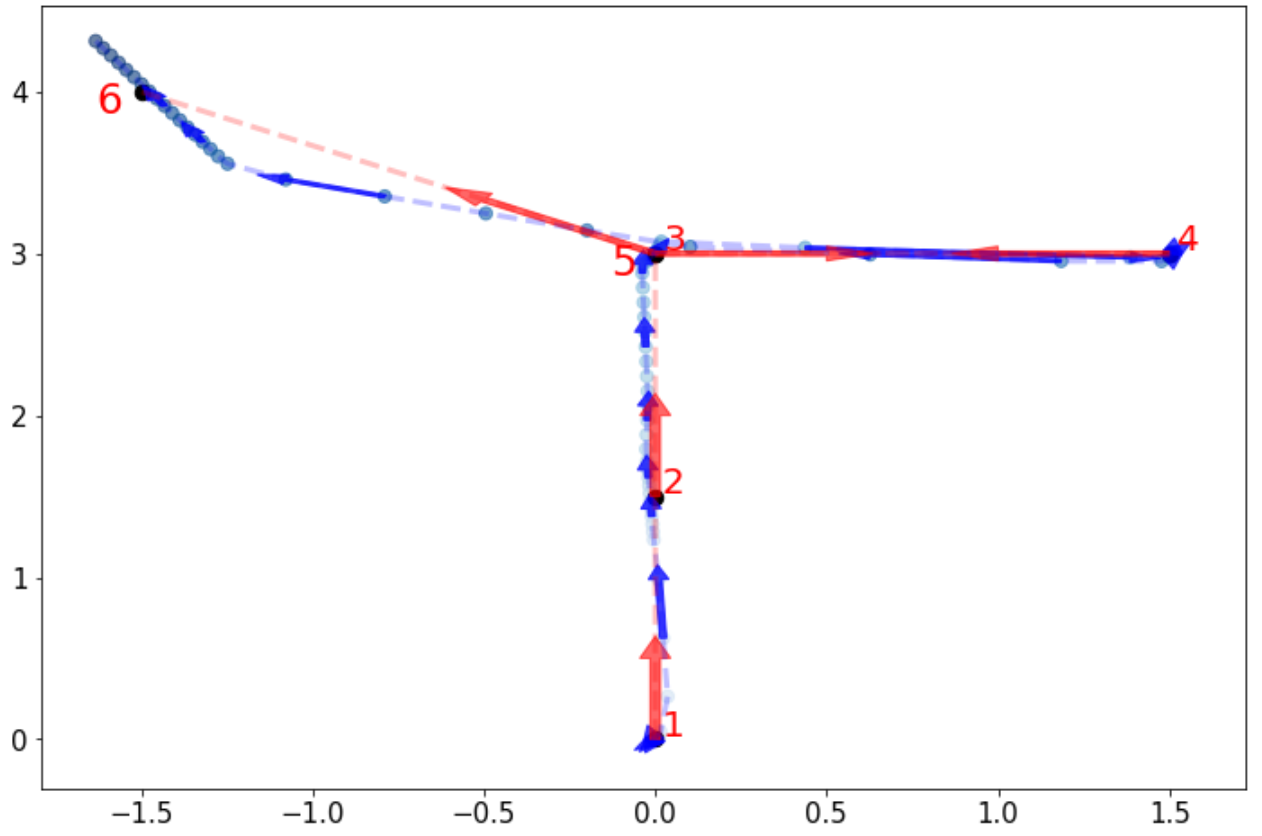}
        }\vfill
        \subfloat[The sample size is $n=10$ and the depth of the generator is equal to $3$. The WGAN misses the shortest path, with a worsened $1$-Wasserstein distance: $W_1({\widehat G}^{\star}_{K\sharp  U},\mu_n)=0.025$ and $W_1(G^{\theta}_{\sharp  U},\mu_n)=0.321$.]
        {
            \includegraphics[width=0.425\linewidth]{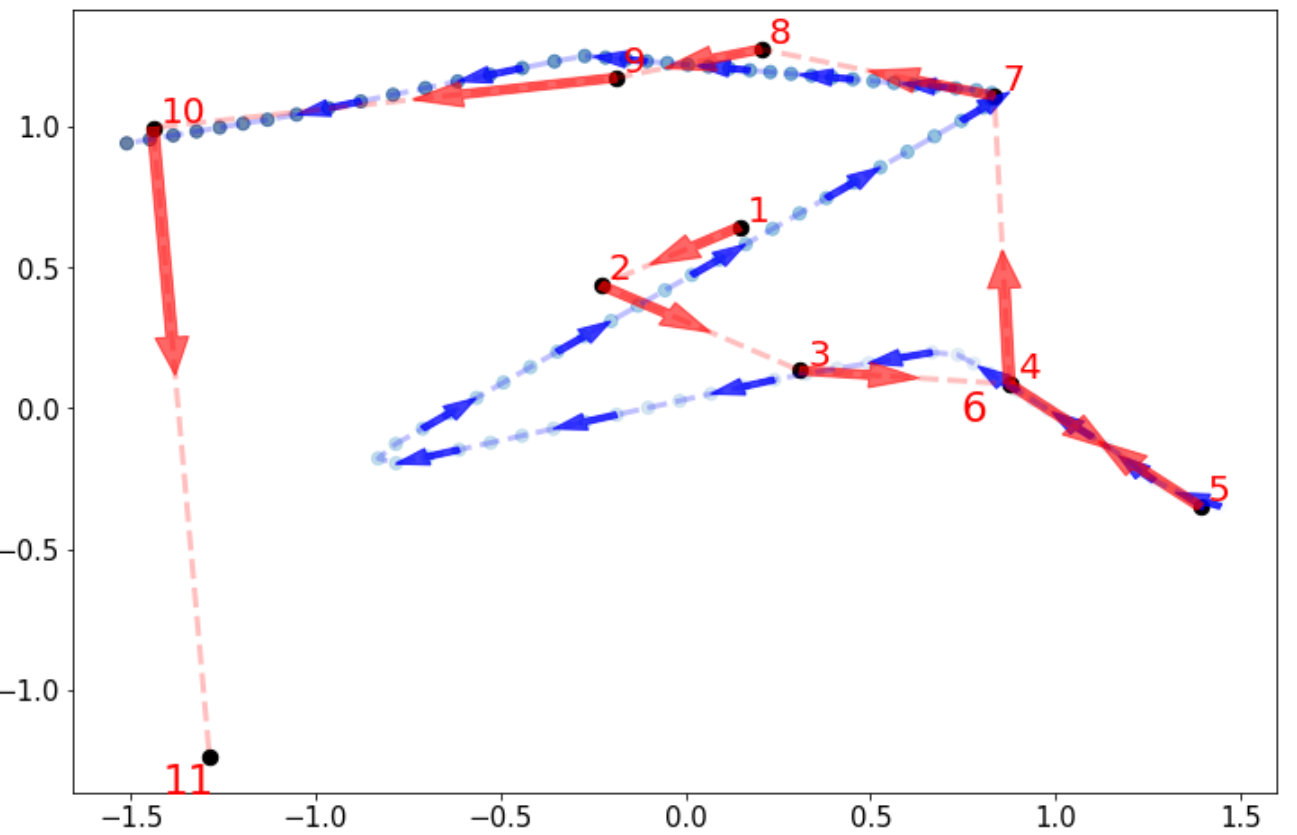}
        }\hspace{1cm}
        \subfloat[The sample size is $n=10$ and the depth of the generator is equal to $6$. The WGAN is closer to the shortest path: $W_1({\widehat G}^{\star}_{K\sharp  U},\mu_n)=0.025$ and $W_1(G^{\theta}_{\sharp  U},\mu_n)=0.160$.]
        {
            \includegraphics[width=0.425\linewidth]{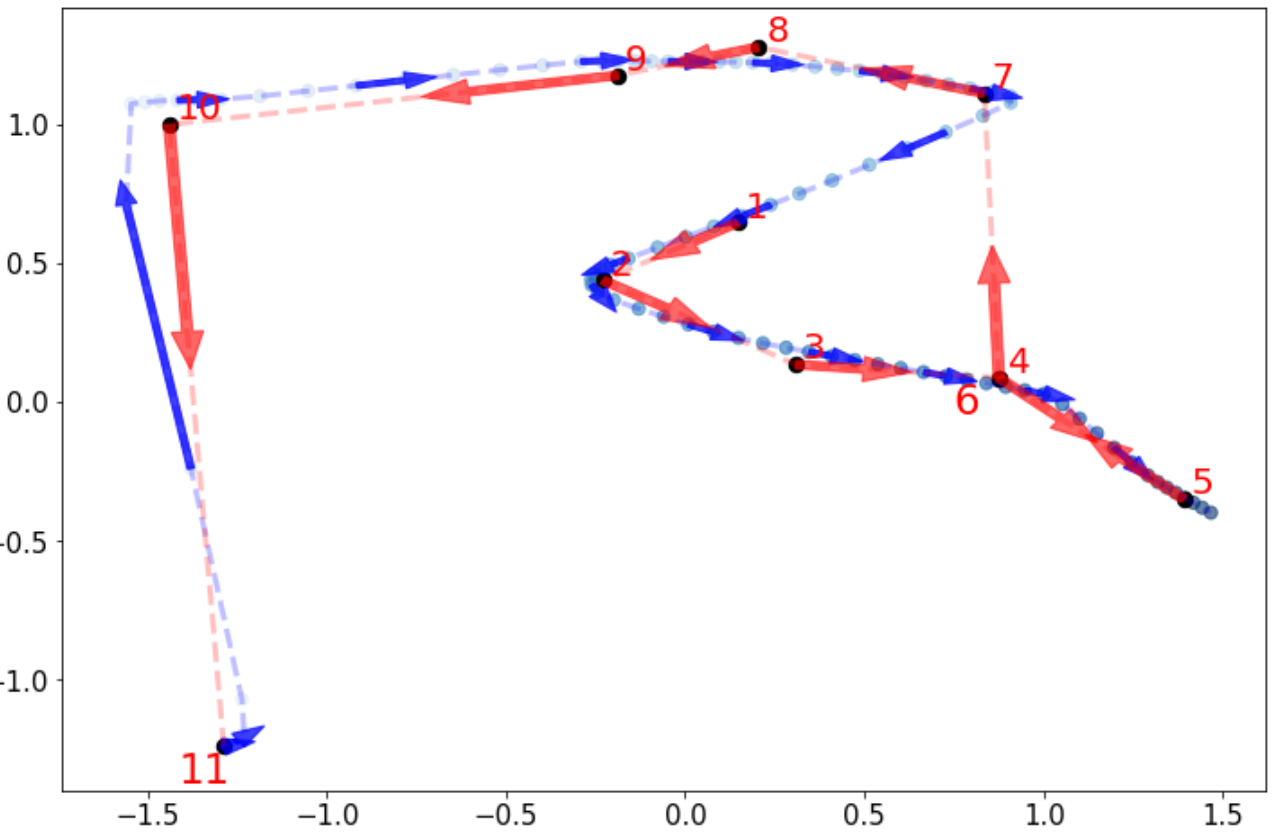}
        }
        \caption{Fitting $2$-dimensional data points with a univariate latent space. The blue curves are the ones reached after optimization with WGANs $G^\theta_{\sharp U}$ and the red curves are the constructed ones ${G}^\star_{K \sharp U}$}. \label{fig:dim2}
    \end{figure}
      
    Theorem~\ref{theorem3} is valid under the condition $K\geqslant \underline {K}_2$, where 
    \[\underline {K}_2:= n\max_{i=1,\hdots,n} \sum_{j\in \sigma^{-1}(i)} \frac{1}{2} (\|X_{\sigma(j-1)}-X_i\|+\|X_{\sigma(j+1)}-X_i\|).\]
    As $\underline {K}_2$ (and thus $K$) is a function of $n$, it is therefore natural to understand the behavior of $W_1({\widehat G}^{\star}_{K\sharp  U},\mu)$ when $n$ tends to infinity. 
    \begin{proposition}\label{propositionK2}
    Assume that $\mu$ has a probability density with respect to the Lebesgue measure on $\mathbb R^d$ and that $S(\mu)$ is bounded. 
    \begin{enumerate}
    \item We have  
    \[\frac{1}{\underline {K}_2}=\mathscr O(n^{-1+1/d}) \text{ a.s.}\]
    \item If, in addition, the density of $\mu$ is bounded away from $0$ on $S(\mu)$, then, for all $K\geqslant \underline {K}_2$, in probability,
    \[W_1({\widehat G}^{\star}_{K\sharp  U},\mu)=\left\{
        \begin{array}{ll}
            \mathscr O(\frac{\log n}{\sqrt{n}}) & \mbox{for } d=2 \\
            \mathscr O (n^{-1/d}) & \mbox{for } d\geqslant 3.
        \end{array}
    \right.
    \]
    \end{enumerate}
    \end{proposition}
    The proof of Proposition~\ref{propositionK2} reveals that, for $d \geqslant 2$, $W_1({\widehat G}^{\star}_{K\sharp  U},\mu_n)=\mathscr O(n^{-1/d})$ in probability, which coincides with the rate of $W_1(\mu,\mu_n)$ for $d \geqslant 3$ \citep[][Theorem 1]{FournierGuillin2015}. However, for $d=2$, $W_1(\mu,\mu_n)=\mathscr O(\frac{\log n}{\sqrt{n}})$, and the speed of convergence to $0$ of $W_1({\widehat G}^{\star}_{K\sharp  U},\mu)$ is therefore slightly slowed down by the term $W_1(\mu,\mu_n)$. In essence, this proposition states that while $K\geqslant \underline {K}_2$ tends to infinity as $n$ grows, the infimum is taken over a larger collection of functions, which enables ${\widehat G}^{\star}_{K\sharp  U}$ to get closer to the target distribution $\mu$ for the $1$-Wasserstein distance. \citet{liang2018well} derived minimax-type results for classes of absolutely continuous distributions defined with Sobolev constraints. 
    \section{Conclusion}\label{sec:conclusion}
    We provided in this paper a thorough analysis of the properties of WGANs, in both the finite sample and asymptotic regimes. Although the dimension of the latent space is assumed to be equal to $1$, the results are valid regardless of the dimension $d$ of the output space. In this setting, we showed that for a fixed sample size $n$, optimal WGANs are closely linked with connected paths minimizing the sum of the squared Euclidean distances between the sample points. We also highlighted the fact that WGANs are able to approach (for the $1$-Wasserstein distance) the target distribution as $n$ tends to infinity, at a given convergence rate and provided the family of Lipschitz functions grows with $K$. We derived in passing new results on optimal transport theory in the semi-discrete setting. In a nutshell, the main message is that WGANs generate data that lie on very specific regions of the ambient space---thus showing some limited ``creativity''--- while being able to asymptotically recover the unknown distribution of the observations under appropriate assumptions.
    
    Nevertheless, many questions remain open and should, in our eyes, be given special attention. First, the current approach is based on a somewhat ideal definition of WGANs, in the sense that we use $\text{Lip}_K([0,1],\mathbb R^d)$ and $\text{Lip}_1(\mathbb R^d,\mathbb R)$ for, respectively, the generator and the discriminator. However, one should keep in mind that in practice both the generator and the discriminator are implemented by deep neural networks. It follows that the results of the paper have to be appreciated in light of the approximation capabilities of neural networks. In particular, larger datasets will require deeper and more expressive networks to reconstruct the optimal functions ${\widehat G}_{K}^{\star}$. Also, using neural networks, the sample points in the dataset are less likely to be overfitted, thus getting closer to the true purpose of generative models, which is to mimic the observations without resampling from the learning database. We believe that studying the potential benefits of this regularization effect is an interesting problem. Next, it was assumed throughout that the latent random variable $U$ is uniform. The extension to latent variables with unbounded support, such as Gaussian distributions, is not straightforward and requires careful investigation. 
    \begin{figure}[h]
        \centering
        \subfloat[$W_1({\widehat G}^{\star}_{K\sharp  U},\mu_n)=0.27$.]
        {
            \includegraphics[width=0.23\linewidth]{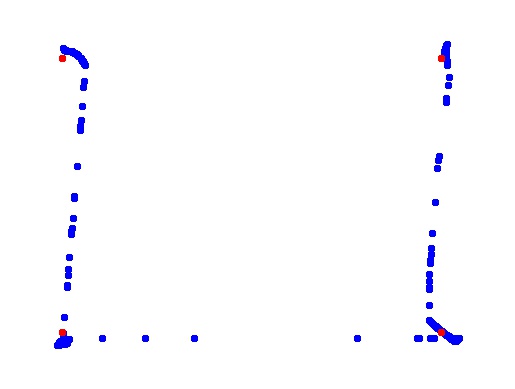}
        }\hfill
        \subfloat[$W_1({\bar G}^{\star}_{K\sharp  U},\mu_n)=0.15$.]
        {
            \includegraphics[width=0.23\linewidth]{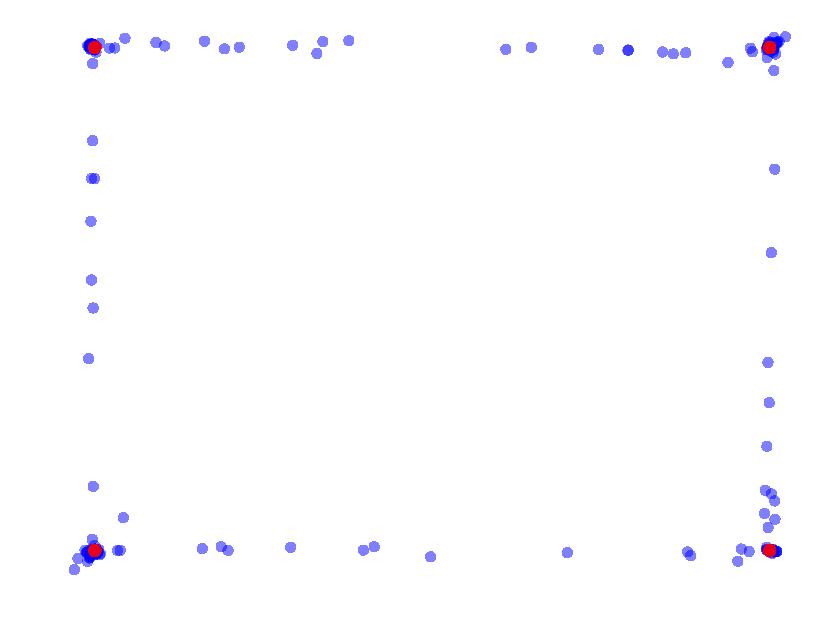}
        }
        \subfloat[Latent space heatmap.]
        {
            \includegraphics[width=0.23\linewidth]{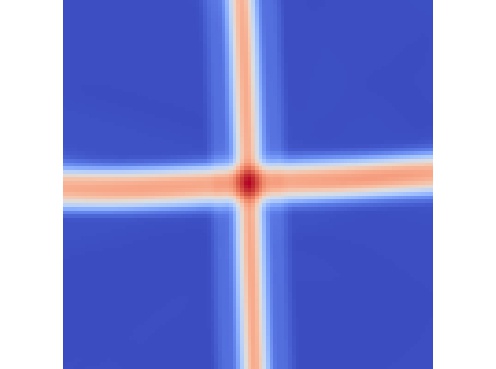}
        }
        \subfloat[Supp. green points.]
        {
            \includegraphics[width=0.23\linewidth]{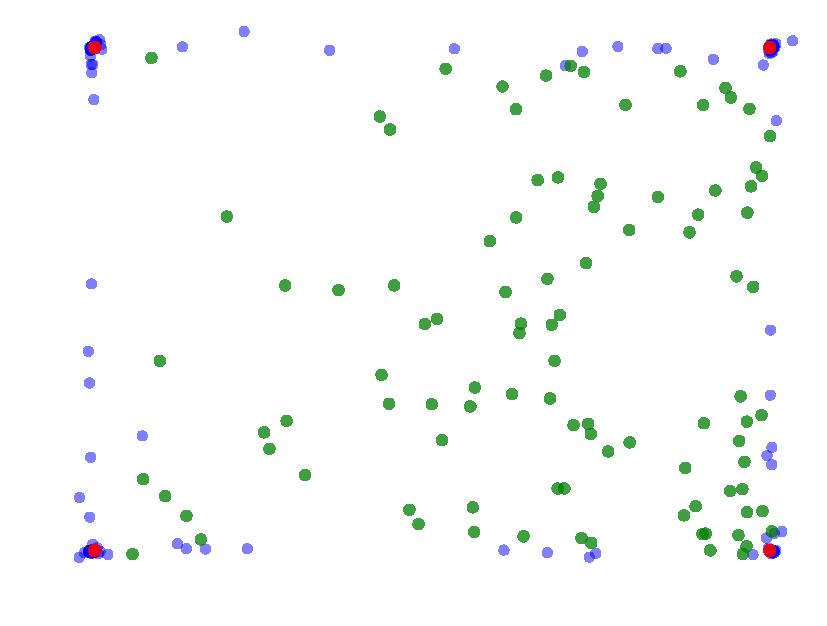}
        } \vfill
        \subfloat[$W_1({\widehat G}^{\star}_{K\sharp  U},\mu_n)=0.16$.]
        {
            \includegraphics[width=0.23\linewidth]{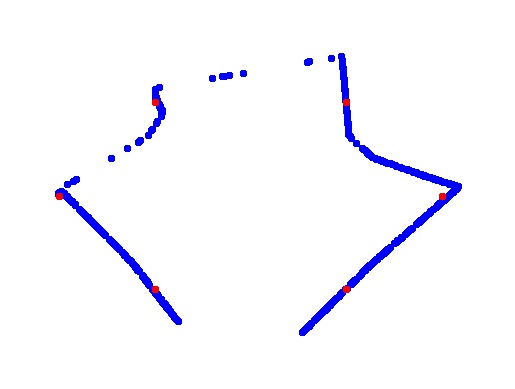}
        }\hfill
        \subfloat[$W_1({\bar G}^{\star}_{K\sharp  U},\mu_n)=0.10$.]
        {
            \includegraphics[width=0.23\linewidth]{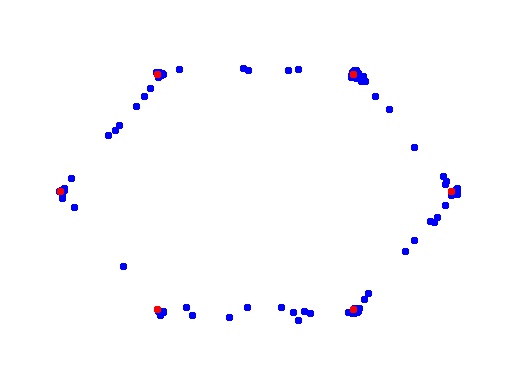}
        }
        \subfloat[Latent space heatmap.]
        {
            \includegraphics[width=0.23\linewidth]{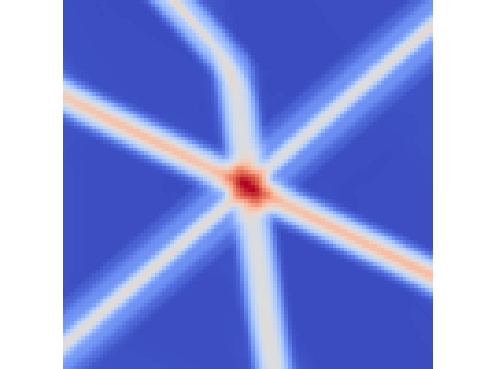}
        }
        \subfloat[Supp. green points.]
        {
            \includegraphics[width=0.23\linewidth]{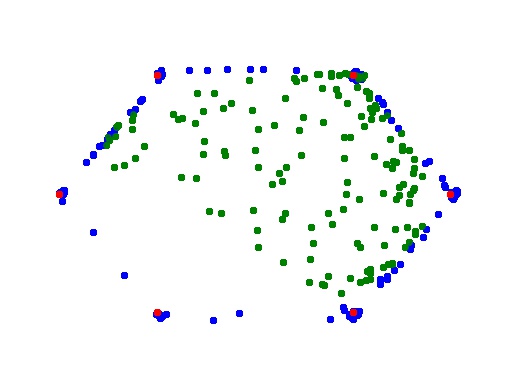}
        }\vfill
        \subfloat[$W_1({\widehat G}^{\star}_{K\sharp  U},\mu_n)=0.42$.]
        {
            \includegraphics[width=0.23\linewidth]{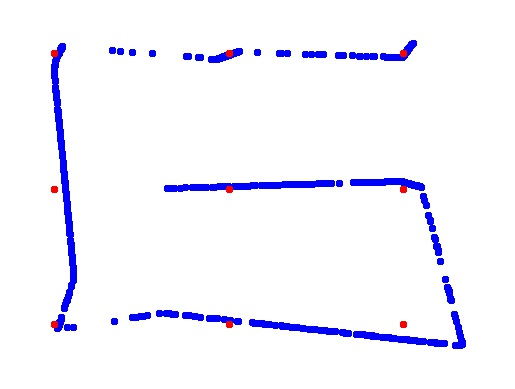}
        }\hfill
        \subfloat[$W_1({\bar G}^{\star}_{K\sharp  U},\mu_n)=0.13$.]
        {
            \includegraphics[width=0.23\linewidth]{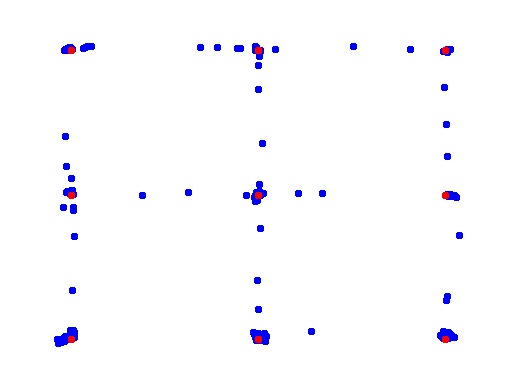}
        }
        \subfloat[Latent space heatmap.]
        {
            \includegraphics[width=0.23\linewidth]{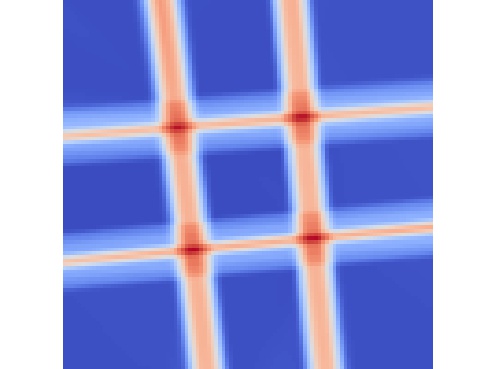}
        }
        \subfloat[Supp. green points.]
        {
            \includegraphics[width=0.23\linewidth]{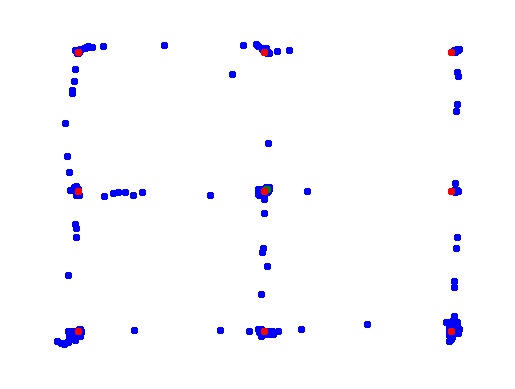}
        }
        
        \caption{Influence of the dimension $p$ of the latent space. In the left column, we use a uniform latent distribution on $[0,1]$ (target points in red, sampled points in blue). In the second column, we use a uniform latent distribution on $[0,1]^2$ (the optimal generator is denoted by ${\bar G}^{\star}_K$). The third column shows heatmaps of the gradients' norm of the optimal generator (the bluer the lower and the redder the higher). Finally, the last column shows supplementary points sampled close to $(\frac{1}{2},\frac{1}{2})$ (in the latent space). \label{fig:fig_conclusion}}
    \end{figure}
    
    Finally, an interesting research direction is to understand and analyze the mechanisms of WGANs when the dimension $p$ of the latent space is strictly larger than 1. In this context, the univariate shortest paths will be replaced by surfaces, and the interesting question will then be to understand the driving forces of WGANs when $p<d$ and $p=d$. As a teaser, we show in Figure~\ref{fig:fig_conclusion} the impact of increasing the dimension of the latent space from $p=1$ to $p=2$, in the case where data (in red) lie in dimension $d=2$. We note that when $p=1$, the WGAN is able to find the shortest paths for the squared Euclidean distance, as predicted by the theory. For $p=2$, the situation is quite intriguing since the $1$-Wasserstein distance between the empirical measure and the pushforward distribution of $U$ by the optimal function $G$ is decreasing. Besides, the generated distributions seem to be concentrated with positive mass on the data points and, with decreasing probabilities, on a path---theoretically undetermined---linking them. Note however that it seems also possible to generate samples anywhere in the convex hull of the data points. This is illustrated in the fourth column of the figure, where we voluntarily sample latent vectors close to the center $(\frac{1}{2},\frac{1}{2})$. We visualize on the heatmaps in the third column the appearance of areas with high gradients of the optimal generator, dividing the latent space. Analyzing the geometrical properties of these latent configurations is a very exciting avenue for future research.

%
%

\begin{acks}
The authors thank J.~Lambolley and M.~Pierre for stimulating and fruitful discussions.
They also thank the Editor-in-Chief and two anonymous referees for their careful reading of the paper and constructive comments, which led to a substantial improvement of the document.
\end{acks}
%

\bibliography{adversarial_training}

\begin{thebibliography}{41}
\providecommand{\natexlab}[1]{#1}
\providecommand{\url}[1]{\texttt{#1}}
\expandafter\ifx\csname urlstyle\endcsname\relax
  \providecommand{\doi}[1]{doi: #1}\else
  \providecommand{\doi}{doi: \begingroup \urlstyle{rm}\Url}\fi

\bibitem[Ambrosio and Gigli(2013)]{ambrosio2013user}
L.~Ambrosio and N.~Gigli.
\newblock A user's guide to optimal transport.
\newblock In B.~Piccoli and M.~Rascle, editors, \emph{Modelling and
  Optimisation of Flows on Networks: Cetraro, Italy 2009}, Berlin, 2013.
  Springer.

\bibitem[Arjovsky et~al.(2017)Arjovsky, Chintala, and
  Bottou]{arjovsky2017wasserstein}
M.~Arjovsky, S.~Chintala, and L.~Bottou.
\newblock Wasserstein generative adversarial networks.
\newblock In D.~Precup and Y.W. Teh, editors, \emph{Proceedings of the 34th
  International Conference on Machine Learning}, volume~70, pages 214--223.
  PMLR, 2017.

\bibitem[Aurenhammer et~al.(1998)Aurenhammer, Hoffmann, and
  Aronov]{aurenhammer1998minkowski}
F.~Aurenhammer, F.~Hoffmann, and B.~Aronov.
\newblock Minkowski-type theorems and least-squares clustering.
\newblock \emph{Algorithmica}, 20:\penalty0 61--76, 1998.

\bibitem[Biau et~al.(2021)Biau, Sangnier, and Tanielian]{BST}
G.~Biau, M.~Sangnier, and U.~Tanielian.
\newblock Some theoretical insights into {W}asserstein {GAN}s.
\newblock \emph{J. Mach. Learn. Res.}, 22\penalty0 (119):\penalty0 1--45, 2021.

\bibitem[Borji(2019)]{borji2019pros}
A.~Borji.
\newblock Pros and cons of {GAN} evaluation measures.
\newblock \emph{Comput. Vis. Image Underst.}, 179:\penalty0 41--65, 2019.

\bibitem[Deheuvels(1984)]{Deheuvels84}
P.~Deheuvels.
\newblock Strong limit theorems for maximal spacings from a general univariate
  distribution.
\newblock \emph{Ann. Probab.}, 12:\penalty0 1181--1193, 1984.

\bibitem[Deheuvels(1986)]{Deheuvels86}
P.~Deheuvels.
\newblock On the influence of the extremes of an {i.i.d.} sequence on the
  maximal spacings.
\newblock \emph{Ann. Probab.}, 14:\penalty0 194--208, 1986.

\bibitem[Evans and Gariepy(2015)]{EvansGariepy2015}
L.C. Evans and R.F. Gariepy.
\newblock \emph{Measure Theory and Fine Properties of Functions}.
\newblock CRC Press, Boca Raton, 2015.

\bibitem[Facco et~al.(2017)Facco, {d’}Errico, Rodriguez, and Laio]{facco2017}
E.~Facco, M.~{d’}Errico, A.~Rodriguez, and A.~Laio.
\newblock Estimating the intrinsic dimension of datasets by a minimal
  neighborhood information.
\newblock \emph{Scientific Reports}, 7:\penalty0 12140, 2017.

\bibitem[Fefferman et~al.(2016)Fefferman, Mitter, and
  Narayanan]{fefferman2016testing}
C.~Fefferman, S.~Mitter, and H.~Narayanan.
\newblock Testing the manifold hypothesis.
\newblock \emph{Journal of the American Mathematical Society}, 29:\penalty0
  983--1049, 2016.

\bibitem[{Fournier} and {Guillin}(2015)]{FournierGuillin2015}
N.~{Fournier} and A.~{Guillin}.
\newblock {On the rate of convergence in Wasserstein distance of the empirical
  measure}.
\newblock \emph{Probab. Theory Related Fields}, 162:\penalty0 707--738, 2015.

\bibitem[Gei{\ss} et~al.(2013)Gei{\ss}, Klein, Penninger, and
  Rote]{geiss2013optimally}
D.~Gei{\ss}, R.~Klein, R.~Penninger, and G.~Rote.
\newblock Optimally solving a transportation problem using {V}oronoi diagrams.
\newblock \emph{Comput. Geom.}, 46:\penalty0 1009--1016, 2013.

\bibitem[Goodfellow et~al.(2014)Goodfellow, Pouget-Abadie, Mirza, Xu,
  Warde-Farley, Ozair, Courville, and Bengio]{GANs}
I.J. Goodfellow, J.~Pouget-Abadie, M.~Mirza, B.~Xu, D.~Warde-Farley, S.~Ozair,
  A.~Courville, and Y.~Bengio.
\newblock Generative adversarial nets.
\newblock In Z.~Ghahramani, M.~Welling, C.~Cortes, N.D. Lawrence, and K.Q.
  Weinberger, editors, \emph{Advances in Neural Information Processing
  Systems}, volume~27, pages 2672--2680. Curran Associates, Inc., 2014.

\bibitem[Gulrajani et~al.(2017)Gulrajani, Ahmed, Arjovsky, Dumoulin, and
  Courville]{gulrajani2017improved}
I.~Gulrajani, F.~Ahmed, M.~Arjovsky, V.~Dumoulin, and A.C. Courville.
\newblock Improved training of {W}asserstein {GAN}s.
\newblock In I.~Guyon, U.~von Luxburg, S.~Bengio, H.~Wallach, R.~Fergus,
  S.~Vishwanathan, and R.~Garnett, editors, \emph{Advances in Neural
  Information Processing Systems}, volume~30, pages 5767--5777. Curran
  Associates, Inc., 2017.

\bibitem[Gulrajani et~al.(2019)Gulrajani, Raffel, and
  Metz]{gulrajani2018towards}
I.~Gulrajani, C.~Raffel, and L.~Metz.
\newblock Towards {GAN} benchmarks which require generalization.
\newblock In \emph{International Conference on Learning Representations}, 2019.

\bibitem[Hartmann and Schuhmacher(2020)]{Hartmann2020}
V.~Hartmann and D.~Schuhmacher.
\newblock Semi-discrete optimal transport: {A} solution procedure for the
  unsquared {E}uclidean distance case.
\newblock \emph{Math. Methods Oper. Res.}, 92:\penalty0 133--163, 2020.

\bibitem[Kantorovich and Rubinstein(1958)]{kantorovich1958space}
L.V. Kantorovich and G.S. Rubinstein.
\newblock On a space of completely additive functions.
\newblock \emph{Vestnik Leningrad Univ. Math.}, 13:\penalty0 52--59, 1958.

\bibitem[Karras et~al.(2018)Karras, Aila, Laine, and
  Lehtinen]{karras2017progressive}
T.~Karras, T.~Aila, S.~Laine, and J.~Lehtinen.
\newblock Progressive growing of {GAN}s for improved quality, stability, and
  variation.
\newblock In \emph{International Conference on Learning Representations}, 2018.

\bibitem[Karras et~al.(2019)Karras, Laine, and Aila]{karras2018style}
T.~Karras, S.~Laine, and T.~Aila.
\newblock A style-based generator architecture for generative adversarial
  networks.
\newblock In \emph{2019 IEEE/CVF Conference on Computer Vision and Pattern
  Recognition}, pages 4396--4405, 2019.

\bibitem[Kodali et~al.(2017)Kodali, Abernethy, Hays, and
  Kira]{kodali2017convergence}
N.~Kodali, J.~Abernethy, J.~Hays, and Z.~Kira.
\newblock On convergence and stability of {GAN}s.
\newblock \emph{arXiv.1705.07215}, 2017.

\bibitem[Liang(2021)]{liang2018well}
T.~Liang.
\newblock How well generative adversarial networks learn distributions.
\newblock \emph{J. Mach. Learn. Res.}, 22\penalty0 (228):\penalty0 1--41, 2021.

\bibitem[Lucic et~al.(2018)Lucic, Kurach, Michalski, Gelly, and
  Bousquet]{lucic2018GANs}
M.~Lucic, K.~Kurach, M.~Michalski, S.~Gelly, and O.~Bousquet.
\newblock Are {GAN}s created equal? {A} large-scale study.
\newblock In S.~Bengio, H.~Wallach, H.~Larochelle, K.~Grauman, N.~Cesa-Bianchi,
  and R.~Garnett, editors, \emph{Advances in Neural Information Processing
  Systems}, volume~31, pages 697--706. Curran Associates, Inc., 2018.

\bibitem[Luise et~al.(2020)Luise, Pontil, and
  Ciliberto]{luise2020generalization}
G.~Luise, M.~Pontil, and C.~Ciliberto.
\newblock Generalization properties of optimal transport {GAN}s with latent
  distribution learning.
\newblock \emph{arXiv:2007.14641}, 2020.

\bibitem[Mescheder et~al.(2018)Mescheder, Geiger, and
  Nowozin]{mescheder2018training}
L.~Mescheder, A.~Geiger, and S.~Nowozin.
\newblock Which training methods for {GAN}s do actually converge?
\newblock In J.~Dy and A.~Krause, editors, \emph{Proceedings of the 35th
  International Conference on Machine Learning}, volume~80, pages 3481--3490.
  PMLR, 2018.

\bibitem[M{\"u}ller(1997)]{IPMsMuller}
A.~M{\"u}ller.
\newblock Integral probability metrics and their generating classes of
  functions.
\newblock \emph{Adv. in Appl. Probab.}, 29:\penalty0 429--443, 1997.

\bibitem[Pratelli(2007)]{pratelli2020existence}
A.~Pratelli.
\newblock On the equality between {M}onge{'}s infimum and {K}antorovich{'}s
  minimum in optimal mass transportation.
\newblock \emph{Ann. Inst. Henri Poincar{\'e} Probab. Stat.}, 43:\penalty0
  1--13, 2007.

\bibitem[Radford et~al.(2016)Radford, Metz, and
  Chintala]{radford2015unsupervised}
A.~Radford, L.~Metz, and S.~Chintala.
\newblock Unsupervised representation learning with deep convolutional
  generative adversarial networks.
\newblock In Y.~Bengio and Y.~Le{C}un, editors, \emph{4th International
  Conference on Learning Representations}, 2016.

\bibitem[Santambrogio(2015)]{santambrogio2015optimal}
F.~Santambrogio.
\newblock \emph{Optimal Transport for Applied Mathematicians}.
\newblock Birk{\"a}user, Cham, 2015.

\bibitem[Schreuder et~al.(2021)Schreuder, Brunel, and
  Dalalyan]{schreuder2021statistical}
N.~Schreuder, V.-E. Brunel, and A.~Dalalyan.
\newblock Statistical guarantees for generative models without domination.
\newblock In V.~Feldman, K.~Ligett, and S.~Sabato, editors, \emph{Proceedings
  of the 32nd International Conference on Algorithmic Learning Theory}, volume
  132, pages 1051--1071. PMLR, 2021.

\bibitem[Singh et~al.(2018)Singh, Uppal, Li, Li, Zaheer, and
  Poczos]{nonparametric2018singh}
S.~Singh, A.~Uppal, B.~Li, C.-L. Li, M.~Zaheer, and B.~Poczos.
\newblock Nonparametric density estimation under adversarial losses.
\newblock In S.~Bengio, H.~Wallach, H.~Larochelle, K.~Grauman, N.~Cesa-Bianchi,
  and R.~Garnett, editors, \emph{Advances in Neural Information Processing
  Systems}, volume~31, pages 10225--10236. Curran Associates, Inc., 2018.

\bibitem[Steele(1988)]{Steele88}
J.M. Steele.
\newblock Growth rates of {E}uclidean minimal spanning trees with power
  weighted edges.
\newblock \emph{Ann. Probab.}, 16:\penalty0 1767--1787, 1988.

\bibitem[St{\'e}phanovitch et~al.(2023)St{\'e}phanovitch, Tanielian, Cadre,
  Klutchnikoff, and Biau]{supplementarymaterial}
A.~St{\'e}phanovitch, U.~Tanielian, B.~Cadre, N.~Klutchnikoff, and G.~Biau.
\newblock Supplement to ``{O}ptimal $1$-{W}asserstein distance for {WGAN}s''.
\newblock 2023.

\bibitem[Tanielian et~al.(2020)Tanielian, Issenhuth, Dohmatob, and
  Mary]{tanielian2020learning}
U.~Tanielian, T.~Issenhuth, E.~Dohmatob, and J.~Mary.
\newblock Learning disconnected manifolds: {A} no {GAN}{'}s land.
\newblock In H.~Daum{\'e} III and A.~Singh, editors, \emph{Proceedings of the
  37th International Conference on Machine Learning}, volume 119, pages
  9418--9427. PMLR, 2020.

\bibitem[Uppal et~al.(2019)Uppal, Singh, and Poczos]{nonparametric2019wallach}
A.~Uppal, S.~Singh, and B.~Poczos.
\newblock Nonparametric density estimation and convergence rates for {GAN}s
  under {B}esov {IPM} losses.
\newblock In H.~Wallach, H.~Larochelle, A.~Beygelzimer, F.~d'Alch\'{e}{-}Buc,
  E.~Fox, and R.~Garnett, editors, \emph{Advances in Neural Information
  Processing Systems}, volume~32, pages 9089--9100. Curran Associates, Inc.,
  2019.

\bibitem[Vaishnavh et~al.(2018)Vaishnavh, Raffel, and
  Goodfellow]{nagarajan2018theoretical}
N.~Vaishnavh, C.~Raffel, and I.J. Goodfellow.
\newblock Theoretical insights into memorization in {GAN}s.
\newblock In \emph{Neural Information Processing Systems 2018 - Integration of
  Deep Learning Theories Workshop}, 2018.

\bibitem[Villani(2008)]{villani2008optimal}
C.~Villani.
\newblock \emph{Optimal Transport: Old and New}.
\newblock Springer, Berlin, 2008.

\bibitem[Vondrick et~al.(2016)Vondrick, Pirsiavash, and
  Torralba]{vondrick2016generating}
C.~Vondrick, H.~Pirsiavash, and A.~Torralba.
\newblock Generating videos with scene dynamics.
\newblock In D.~Lee, M.~Sugiyama, U.~von Luxburg, I.~Guyon, and R.~Garnett,
  editors, \emph{Advances in Neural Information Processing Systems}, volume~29,
  pages 613--621. Curran Associates, Inc., 2016.

\bibitem[YoonHaeng et~al.(2021)YoonHaeng, Guo, and Liang]{YoonHaeng2021}
H.~YoonHaeng, W.~Guo, and T.~Liang.
\newblock Reversible {G}romov-{M}onge sampler for simulation-based inference.
\newblock \emph{arXiv:2109.14090}, 2021.

\bibitem[Yu et~al.(2017)Yu, Zhang, Wang, and Yu]{SeqGANs}
L.~Yu, W.~Zhang, J.~Wang, and Y.~Yu.
\newblock Seq{GAN}: {S}equence generative adversarial nets with policy
  gradient.
\newblock In \emph{Proceedings of the Thirty-First AAAI Conference on
  Artificial Intelligence}, pages 2852--2858. AAAI Press, 2017.

\bibitem[Yukich(2000)]{Yukich2000}
J.E. Yukich.
\newblock Asymptotics for weighted minimal spanning trees on random points.
\newblock \emph{Stochastic Process. Appl.}, 85:\penalty0 123--138, 2000.

\bibitem[Zhou et~al.(2019)Zhou, Liang, Song, Yu, Wang, Zhang, Yu, and
  Zhang]{zhou2019lipschitz}
Z.~Zhou, J.~Liang, Y.~Song, L.~Yu, H.~Wang, W.~Zhang, Y.~Yu, and Z.~Zhang.
\newblock Lipschitz generative adversarial nets.
\newblock In K.~Chaudhuri and R.~Salakhutdinov, editors, \emph{Proceedings of
  the 36th International Conference on Machine Learning}, volume~97, pages
  7584--7593. PMLR, 2019.

\end{thebibliography}



\appendix
\section{Proof of Lemma ~\ref{lemma1}}
    We only focus on the first statement since the proof of the second one is similar. Let $G, G' \in \text{Lip}_K ([0,1],\mathbb R^d)$. Observe that by the triangle inequality and the primal definition of the $1$-Wasserstein distance, we have 
    \begin{align*}
        |W_1(G_{\sharp U}, \mu_n) - W_1(G'_{\sharp U}, \mu_n)|
        & \leqslant W_1(G_{\sharp U}, G'_{\sharp U}) \\
        & \leqslant  \int_{\mathbb{R}^d\times \mathbb R^d} \|x-y\| {\rm d}\gamma(x,y),
    \end{align*}
        where $\gamma$ is the pushforward distribution of $U$ by the pair $(G,G')$, with marginals $G_{\sharp U}$ and $G'_{\sharp U}$.  Thus,
        \begin{align*}
            |W_1(G_{\sharp U}, \mu_n) - W_1(G'_{\sharp U}, \mu_n)| & 
            \leqslant \int_{[0,1]} \|G(u)- G'(u)\| {\rm d} u \\
            &\leqslant \|G-G'\|_\infty,
        \end{align*}
    where $\|\cdot\|_{\infty}$ denotes the supremum norm of functions, i.e., for $f:[0,1]\to\mathbb{R}^d$, $\|f\|_\infty = \sup\{\|f(x)\| : x\in[0,1] \}$. Hence the map $\text{Lip}_K ([0,1],\mathbb R^d)\ni G \mapsto W_1(G_{\sharp U}, \mu_n)$ is continuous with respect to the uniform norm.
    
    Now let $G^0\equiv X_1$ be a constant function on $[0,1]$. Then, clearly,
    $W_1(G^0_{\sharp U}, \mu_n)< \infty.$ Next, let $G$ be any function in $\text{Lip}_K ([0,1],\mathbb R^d)$ such that \[\|G\|_{\infty}\geqslant W_1(G^0_{\sharp U}, \mu_n) + K+ \underset{ i=1,\hdots,n}{\max} \|X_i\|.\] Then, upon observing that there exists $u_0 \in [0,1]$ such that $\|G(u_0)\|=\|G\|_{\infty}$ and using the fact that $G$ is $K$-Lipschitz continuous on $[0,1]$, we deduce that for all $u \in [0,1]$ and any $i \in \{1,\hdots, n\}$, one has 
    \[\|G(u)-X_i\|\geqslant \|G\|_\infty -K- \|X_i\| \geqslant\|G\|_\infty  -K-\underset{ i=1,\hdots,n}{\max} \|X_i\|.\]
    Hence, $\|G(u)-X_i\| \geqslant W_1(G^0_{\sharp U}, \mu_n)$, which implies that $W_1(G_{\sharp U}, \mu_n) \geqslant W_1(G^0_{\sharp U}, \mu_n)$. Therefore, letting 
    \[
    \mathscr H_K=
        \{ G \in \text{Lip}_K ([0,1],\mathbb R^d)\ :\ \|G\|_\infty \leqslant W_1(G^0_{\sharp U}, \mu_n) + K+\underset{i=1,\hdots, n}{\max} \|X_i\|\}, 
    \]
    we see that
    \[
    \inf_{G\in \textrm{Lip}_K([0,1], \mathbb R^d)} W_1(G_{\sharp  U},\mu_n)
    =
    \inf_{G\in \mathscr H_K} W_1(G_{\sharp  U},\mu_n).
    \]
    
    Endowed with the uniform norm, $\mathscr H_K$ is closed and relatively compact by the Arzel{\`a}-Ascoli theorem. It is thus a compact subset of $\text{Lip}_K ([0,1],\mathbb R^d)$. Consequently, by continuity and the above equality, $\text{Lip}_K ([0,1],\mathbb R^d)\ni G \mapsto W_1(G_{\sharp U}, \mu_n)$ attains its minimum on $\mathscr H_K$. Therefore,  ${\widehat{\mathscr{G}}}_K$ is not empty. 
    \section{Proof of Theorem~\ref{theorem1asymptotique1}}
    \subsection*{Proof of $1(i)$} Since $\mu$ is of order $1$, one has $\lim_{n \to \infty}W_1(\mu,\mu_n)=0$ a.s.~according to \citet[][Theorem 6.8]{villani2008optimal}. Hence, by the triangle inequality and because ${\widehat G}_{K} \in \mathscr {\widehat G}_{K}$, we only need to prove that 
    \[\lim_{n\to \infty} \inf_{G\in \textrm{Lip}_K([0,1],\mathbb{R})} W_1(G_{\sharp  U},\mu_n)=0 \text{ a.s.}\]
    If $K\geqslant K_0$, then $\text{Lip}_{K_0}([0,1],\mathbb{R})\subseteq\text{Lip}_K([0,1],\mathbb{R})$. Therefore,
    \[0\leqslant \inf_{G\in \text{Lip}_K([0,1],\mathbb{R})} W_1(G_{\sharp  U},\mu_n)\leqslant \inf_{G\in \text{Lip}_{K_0}([0,1],\mathbb{R})} W_1(G_{\sharp  U},\mu_n)\leqslant W_1(F^{-1}_{\sharp  U},\mu_n),\]
    since, by assumption, $F^{-1}\in \text{Lip}_{K_0}([0,1],\mathbb{R})$. 
    But $F^{-1}(U)$ has distribution $\mu$, and thus one has $\lim_{n\to \infty} W_1(F_{\sharp  U}^{-1},\mu_n)=0$. This proves the result.
    \subsection*{Proof of $(2)$} The result is proved by contradiction. Fix $K>0$ and assume that on an event of strictly positive probability 
    \[\liminf_{n\to \infty} \ W_1({\widehat G}_{K\sharp  U},\mu)=0.\]
    Since $\lim_{n \to \infty}W_1(\mu,\mu_n)=0$ a.s.~and ${\widehat G}_{K} \in \mathscr {\widehat G}_{K}$, we see that
    \[\inf_{G\in \text{Lip}_K([0,1],\mathbb{R})} W_1(G_{\sharp  U},\mu)=0.\]
    Now, by Lemma~\ref{lemma1}, there exists $G_K \in \text{ Lip}_K([0,1], \mathbb{R})$ such that
    \[W_1(G_{K\sharp  U},\mu)=\inf_{G\in \text{Lip}_K([0,1],\mathbb{R})} W_1(G_{\sharp  U},\mu).\]
    So, $W_1(G_{K\sharp  U},\mu)=0$ and therefore, since $F^{-1}(U)$ has distribution $\mu$, we have 
    \begin{equation}
    \label{refloi} 
    G_K(U)\stackrel{\mathscr L}{\sim} F^{-1}(U). 
    \end{equation} 
    Next, by continuity of $G_K$, there exists a compact set $C\subseteq \mathbb{R}$ such that $\mathbb P(G_K (U)\in C)=1$. But, since $S(\mu)$ is unbounded, $\mathbb P(F^{-1}(U)\in C)=\mu(C)<1$, which contradicts~\eqref{refloi}. 
    \subsection*{Proof of $1(ii)$} We show the result by contradiction, assuming as in the proof of statement $(2)$ that for $K<1/K_1$, on an event of strictly positive probability, 
    \[\liminf_{n\to \infty} \ W_1({\widehat G}_{K\sharp  U},\mu)=0.\]
    Arguing as in the previous proof, we have that $G_K(U)\stackrel{\mathscr L}{\sim} F^{-1}(U)$. Then, it is a classical exercise to deduce from~\eqref{refloi}, since $F^{-1}(u)>-\infty$ for all $u\in (0,1)$ and $F$ is continuous, that $F\circ G_K (U)\stackrel{\mathscr L}{\sim} U$. Iterating this relation leads to
    \begin{equation}
    \label{iter}
    (F\circ G_K)^\ell(U)\stackrel{\mathscr L}{\sim} U, \quad \forall \ell\geqslant 0.
    \end{equation}
    Moreover, both assumptions $F\in \text{Lip}_{K_1}(\mathbb{R},[0,1])$ and $G_K\in \text{Lip}_K([0,1],\mathbb{R})$ imply 
    \[|F\circ G_K(u)-F\circ G_K (v)|\leqslant KK_1|u-v|\leqslant KK_1, \quad \forall (u,v)\in [0,1]^2.\]
    Repeating this inequality entails, for all $\ell\geqslant 0$,
    \[|(F\circ G_K)^\ell(u)-(F\circ G_K)^\ell (v)|\leqslant (KK_1)^\ell, \quad \forall (u,v)\in [0,1]^2.\]
    But, for all $u\in [0,1]$, the sequence $((F\circ G_K)^\ell(u))_{\ell\geqslant 1}$ is bounded by 1. In addition, $KK_1<1$ by assumption. Thus, there exist $a\in [0,1]$ and a subsequence $(\ell_q)_{q\geqslant 1}$ such that, for all $u\in [0,1]$, 
    \[\lim_{q\to \infty} (F\circ G_K)^{\ell_q}(u)=a.\]
    Hence, as $q\to\infty$, $(F\circ G_K)^{\ell_q}(U)$ almost surely converges to $a$, which contradicts~\eqref{iter}. 
    \section{Proof of Theorem~\ref{theorem1asymptotique2}}
    Looking for a contradiction, we start as in the proof of Theorem~\ref{theorem1asymptotique1}, cases $(1ii)$ and $(2)$, by assuming that on an event of strictly positive probability, 
    {\[\liminf_{n\to \infty} \ W_1({\widehat G}_{K\sharp U},\mu)=0.\]}
    As we have seen, this implies $W_1(G_{K\sharp  U},\mu)=0$ and, in turn, 
    since the support of ${G}_{K\sharp U}$ is included in $G_K([0,1])$,  $S(\mu)\subseteq G_K([0,1])$. By our assumption on $S(\mu)$, we therefore conclude that $\lambda_d(G_K ([0,1]))>0$. Moreover, since $G_K \in \text{Lip}_K([0,1],\mathbb{R}^d)$, we have that
    $0<\lambda_d(G_K([0,1])) = \mathcal H_d(G_K([0,1]))\leqslant K^d \mathcal H_d([0,1])$, where $\mathcal H_d$ is the $d$-dimensional Hausdorff measure \citep[see, e.g.,][Theorem 2.8]{EvansGariepy2015}. But this is impossible since $\mathcal H_d([0,1])=0$ as soon as $d>1$. 
    \section{Proof of Proposition~\ref{proposition1}}
    To lighten the notation, it is assumed throughout the proof that the $X_i$'s are ordered by increasing values, i.e., $X_1 \leqslant X_2 \leqslant \cdots \leqslant X_n$. According to \citet[][Proposition 2.17]{santambrogio2015optimal}, the $1$-Wasserstein distance between two probability measures $\pi_1$ and $\pi_2$ on the real line, with respective generalized inverses $F_1^{-1}$ and $F_2^{-1}$, is such that 
\[W_1(\pi_1,\pi_2)=\int_0^1 |F_1^{-1}(u)-F_2^{-1}(u)|{\rm d} u.\]
Since ${\widehat G}_{K}^\star$ is monotone and continuous, the generalized inverse of ${\widehat G}_{K\sharp  U}^\star$ is ${\widehat G}_{K}^\star$. Also, denoting by ${F}^{-1}_{\mu_n}$ the generalized inverse of $\mu_n$, we have ${F}^{-1}_{\mu_n}(u)=\sum_{i=1}^nX_i \mathds{1}{\{u\in ((i-1)/n,i/n]\}}$. Therefore,
\begin{align*}
    W_1({\widehat G}^{\star}_{K\sharp  U},\mu_n)& = \int_0^1 |{\widehat G}_{K}^\star(u)-{F}^{-1}_{\mu_n}(u)|{\rm d} u\\
    &  =  \sum \limits_{i=1}^{n-1} \int_{i/n-\frac{X_{i+1}-X_{i}}{2K}}^{i/n}  \Big|X_i+K\big(u-(\frac{i}{n}-\frac{X_{i+1}-X_{i}}{2K})\big)-X_i\Big|{\rm d}u\\
    & \quad + \sum \limits_{i=1}^{n-1} \int_{i/n}^{i/n+\frac{X_{i+1}-X_{i}}{2K}}  \Big|\frac{X_{i+1}-X_{i}}{2K}+K(u-\frac{i}{n})-X_{i+1}\Big|{\rm d}u\\
    & = \sum \limits_{i=1}^{n-1} \frac{1}{2}K\big(\frac{(X_{i+1}-X_i)^2}{4K^2} +\frac{(X_{i+1}-X_i)^2}{4K^2}\big) \\
    & = \frac{1}{4K} \sum \limits_{i=1}^{n-1}(X_{i+1}-X_i)^2,
\end{align*} 
as desired. 
    \section{Proof of Theorem~\ref{theorem1}}
    As in the proof of Proposition~\ref{proposition1}, it is assumed without loss of generality that the $X_i$'s are ordered by increasing values, i.e., $X_1 \leqslant X_2 \leqslant \cdots \leqslant X_n$. Let $G:[0,1] \to \mathbb{R}$ be an arbitrary $K$-Lipschitz continuous function in $\widehat {\mathscr G}_K$, with $K\geqslant n \max \limits_{i=1,\hdots,n-1} (X_{i+1}-X_i)$. According to Proposition~\ref{proposition1}, the first statement will be proven if we show that for such a function $G$, 
    \[W_1(G_{\sharp  U},\mu_n) \geqslant \sum \limits_{i=1}^{n-1} \frac{(X_{i+1}-X_i)^2}{4K}.\]
    
Let $\Pi(\pi_1,\pi_2)$ be the set of couplings between two probability measures $\pi_1$ and $\pi_2$. According to \citet[][Lemma 2.12]{ambrosio2013user}, for any $\pi \in \Pi(G_{\sharp U},\mu_n)$, there exists a coupling $\gamma\in \Pi(\lambda_1,\mu_n)$ such that $\pi=(G,\textrm{Id})_{\# \gamma}$, where $\lambda_1$ stands for the Lebesgue measure on the interval $[0,1]$ and $\textrm{Id}$ is the identity function. Therefore,
\begin{align*}
    W_1(G_{\sharp  U},\mu_n)& = \inf_{\pi \in \Pi(G_{\sharp U},\mu_n)}\int_{\mathbb{R}\times \mathbb{R}}|x-y|{\rm d}\pi(x,y)\\
    & \geq \inf_{\gamma \in \Pi(\lambda_1,\mu_n)}\int_{[0,1]\times \mathbb{R}}|G(u)-y|{\rm d}\gamma(u,y).
\end{align*}

Since the function $(u,y)\mapsto |G(u)-y|$ is continuous, then, according to \citet[][Theorem B]{pratelli2020existence}, we have 
    \[
    \inf \limits_{\gamma \in \Pi(\lambda_1,\mu_n)}\int_{[0,1]\times \mathbb{R}} |G(u)-y|{\rm d} \gamma(u,y) = \inf \limits_T\int_0^1 |G(u)-T(u)|{\rm d} u,
    \]
    where the infimum is taken over all measurable functions $T:[0,1] \to \{X_1, \hdots, X_n\}$ such that $T_{\sharp  U} = \mu_n$. 
    Any such transport map $T$ takes the form $T(u) = \sum_{i=1}^n  X_i\mathds{1}\{u \in C_i\}$, where $C_1,\hdots,C_n$ are Borel subsets of $[0,1]$ such that 
    $\lambda_1(C_i)= \frac{1}{n}$. We conclude that
    \begin{equation}\label{Monge}  W_1(G_{\sharp  U},\mu_n)  \geqslant \inf \limits_{C_1,\hdots,C_n}  \sum \limits_{i=1}^n\int_{C_i}  |G(u)-X_i|{\rm d}u,
    \end{equation} 
    where the infimum is taken over all disjoint Borel sets $C_1,\hdots,C_n \subseteq [0,1]$ such that $\lambda_1(C_i)= \frac{1}{n}$. To prove the first statement of the theorem, it is therefore sufficient to lower bound the infimum above. 
    
    The case $n=1$ is clear since the function $G(u) \equiv X_1$ satisfies $W_1(G_{\sharp  U},\mu_1)=0$. Thus, in the sequel, it is assumed that $n\geqslant 2$. We let $a = \inf \limits_{[0,1]} G$, $b= \sup \limits_{ [0,1]} G$, and $\ell_1\leqslant \ell_2$ so that $X_{{\ell}_1} = \min \limits_{X_i \geqslant a} X_i$ and $X_{{\ell}_2} = \max \limits_{X_i \leqslant b} X_i$. Note that we can safely assume that $\ell_1$ and $\ell_2$ are well-defined, since for $\hat{G}(u) := G(u)\mathds{1}\{G(u) \in [X_1,X_n]\} +X_1 \mathds{1}\{G(u) < X_1\} +X_n \mathds{1}\{G(u) > X_n\}$, we have 
    \[\inf \limits_{C_1,\hdots,C_n}  \sum \limits_{i=1}^n\int_{C_i}  |G(u)-X_i|{\rm d}u \geqslant \inf \limits_{C_1,\hdots,C_n}  \sum \limits_{i=1}^n\int_{C_i}  |\hat{G}(u)-X_i|{\rm d}u.\]
    We also suppose that $n>\ell_2 \geqslant \ell_1 + 1>1$ and leave the other cases as straightforward adaptations. Since $G$ is continuous, for each $i \in \{{\ell}_1,\hdots,{\ell}_2-1\}$, there exists $u_i \in [0,1]$ such that $G(u_i) = \frac{X_i+X_{i+1}}{2}$. We let $A_i^- = [u_i-\frac{X_{i+1}-X_{i}}{2K},u_i]$, $A_i^+ = [u_i,u_i +  \frac{X_{i+1}-X_{i}}{2K}]$, and write $T(u) = \sum \limits_{j=1}^n X_j \mathds{1}\{u\in C_j\}$. With this notation, 
    \begin{align} 
    \int_{A_i^-} |G(u)-T(u)|{\rm d}u & = \sum \limits_{j=1}^i \int_{A_i^-} (G(u)-X_i + X_i - X_j) \mathds{1}\{u\in C_j\} {\rm d}u \nonumber\\
    & +  \sum \limits_{j=i+1}^n \int_{A_i^-} (X_{i+1}-G(u) + X_j - X_{i+1}) \mathds{1}\{u\in C_j\} {\rm d}u\nonumber\\
    & = \sum \limits_{j=1}^i \Big[\int_{A_i^-} (G(u)-X_i ) \mathds{1}\{u\in C_j\} {\rm d}u +\lambda_1(C_j\cap A_i^-)(X_i-X_j)\Big] \nonumber\\
    & \hspace{-1cm} + \sum \limits_{j=i+1}^n \Big[\int_{A_i^-} (X_{i+1}-G(u)) \mathds{1}\{u\in C_j\} {\rm d}u + \lambda_1(C_j\cap A_i^-)(X_j - X_{i+1})\Big].\label{Ai-}
    \end{align}
    Exploiting the fact that the function $G$ is $K$-Lipschitz continuous and $G(u_i)=\frac{X_{i}+X_{i+1}}{2}$, we have that for $u\in A_i^- \cup A_i^+$, $\frac{X_{i}+X_{i+1}}{2}-K|u_i-u|\leqslant G(u) \leqslant \frac{X_{i}+X_{i+1}}{2}+K|u_i-u|$. Thus, 
    \begin{align}
     &\sum_{j=1}^i \int_{A_i^-} (G(u)-X_i ) \mathds{1}\{u\in C_j\} {\rm d}u + \sum_{j=i+1}^n \int_{A_i^-} (X_{i+1}-G(u)) \mathds{1}\{u\in C_j\} {\rm d}u\nonumber\\ 
    & \quad \geqslant  \sum_{j=1}^i \int_{A_i^-} \Big(\frac{X_{i}+X_{i+1}}{2}-K(u_i-u)-X_i \Big) \mathds{1}\{u\in C_j\} {\rm d}u\nonumber \\ 
    & \qquad + \sum_{j=i+1}^n \int_{A_i^-} \Big(X_{i+1}-\Big(\frac{X_{i}+X_{i+1}}{2}+K(u_i-u)\Big)\Big) \mathds{1}\{u\in C_j\} {\rm d}u \nonumber \\
    & \quad = \sum_{j=1}^n \int_{A_i^-} \Big(\frac{X_{i+1}-X_i}{2}-K(u_i-u)\Big) \mathds{1}\{u\in C_j\} {\rm d}u\nonumber\\
    & \quad = \int_{A_i^-} \Big(\frac{X_{i+1}-X_i}{2}-K(u_i-u)\Big){\rm d}u\nonumber \\
    & \quad = \frac{(X_{i+1}-X_i)^2}{4K}-\frac{1}{2}\frac{(X_{i+1}-X_i)^2}{4K} \nonumber\\
    & \quad = \frac{(X_{i+1}-X_i)^2}{8K}. \label{Kspeed} 
    \end{align}  
    Combining this inequality with~\eqref{Ai-} yields
    \begin{align*}
    \int_{A_i^-} |G(u)-T(u)|{\rm d}u &\geqslant \frac{(X_{i+1}-X_i)^2}{8K}\\
    & \quad + \sum \limits_{j=1}^{i-1}  \lambda_1(C_j\cap A_i^-) (X_i-X_j) +  \sum \limits_{j=i+1}^{n}  \lambda_1(C_j\cap A_i^-) (X_j-X_{i+1}). 
    \end{align*}
    
    Employing the same technique for $A_i^+$, we obtain 
    \begin{align*}
    \int_{A_i^+} |G(u)-T(u)|{\rm d}u & \geqslant \frac{(X_{i+1}-X_i)^2}{8K}\\
    & \quad + \sum \limits_{j=1}^{i-1} \lambda_1(C_j\cap A_i^+) (X_i-X_j) + \sum \limits_{j=i+1}^{n} \lambda_1(C_j\cap A_i^+) (X_j-X_{i+1}).
    \end{align*}
    So, letting $A_i =A_i^- \cup A_i^+$ and using the fact that $X_{\ell+1}\geqslant X_\ell$ for all $\ell\leqslant n-1$, we are led to
    \begin{align}
    \int_{A_i} |G(u)-T(u)|{\rm d}u &  \geqslant \frac{(X_{i+1}-X_i)^2}{4K} \nonumber\\
    & \quad + \sum \limits_{j=1}^{i-1} \lambda_1(C_j\cap A_i) (X_{j+1}-X_j) + \sum \limits_{j=i+2}^{n} \lambda_1(C_j\cap A_i) (X_j-X_{j-1}). \label{one}
    \end{align}
    
    Now, let $u_{{\ell}_1-1} \in [0,1]$ be such that $G(u_{{\ell}_1-1})= \frac{a+X_{{\ell}_1}}{2}$. With a slight abuse of notation, define $A_{{\ell}_1-1}^-=[u_{{\ell}_1-1} -\frac{X_{{\ell}_1}-a}{2K},u_{{\ell}_1-1}]$ and $A_{{\ell}_1-1}^+=[u_{{\ell}_1-1}, u_{{\ell}_1-1} +\frac{X_{{\ell}_1}-a}{2K}]$. Then, using the same method as above, one easily shows that, for $A_{{\ell}_1-1} = A_{{\ell}_1-1}^-\cup A_{{\ell}_1-1}^+$, 
    \begin{align*}
    \int_{A_{{\ell}_1-1}} |G(u)-T(u)|{\rm d}u & \geqslant \frac{(X_{{\ell}_1}-a)^2}{4K} \\
    & \hspace{-0.5cm} + \sum \limits_{j=1}^{{\ell}_1-1} \lambda_1(C_j\cap A_{{\ell}_1-1}) (a-X_j) + \sum \limits_{j={\ell}_1+1}^{n} \lambda_1(C_j\cap A_{{\ell}_1-1}) (X_j-X_{{\ell}_1}).
    \end{align*}
    In a similar fashion, for $u_{{\ell}_2}\in [0,1]$ such that $G(u_{{\ell}_2}) = \frac{X_{{\ell}_2}+b}{2}$ and, with a slight abuse of notation, letting $A_{\ell_2}=[u_{\ell_2}-\frac{b-X_{\ell_2+1}}{2K},u_{\ell_2}+\frac{b-X_{\ell_2+1}}{2K}]$, we obtain
    \begin{align*}
    \int_{A_{{\ell}_2}} |G(u)-T(u)|{\rm d}u & \geqslant \frac{(b-X_{{\ell}_2})^2}{4K} \\
    & \quad + \sum \limits_{j=1}^{{\ell}_2-1} \lambda_1(C_j\cap A_{{\ell}_2}) (X_{{\ell}_2}-X_j) + \sum \limits_{j={\ell}_2+1}^{n} \lambda_1(C_j\cap A_{{\ell}_2}) (X_j-b).
    \end{align*}
    Accordingly,
    \begin{align}
    \int_{A_{{\ell}_1-1}\cup A_{{\ell}_2}} |G(u)-T(u)|{\rm d}u & \geqslant \frac{(X_{{\ell}_1}-a)^2}{4K} + \frac{(b-X_{{\ell}_2})^2}{4K} \nonumber\\
    & \quad + \sum \limits_{j=1}^{{\ell}_1-2} \lambda_1(C_j\cap A_{{\ell}_1-1}) (X_{j+1}-X_j)\nonumber \\
    & \quad +\lambda_1(C_{{\ell}_1-1} \cap A_{{\ell}_1-1})(a-X_{{\ell}_1-1})\nonumber\\
    & \quad + \sum \limits_{j={\ell}_1+1}^{n} \lambda_1(C_j\cap A_{{\ell}_1-1}) (X_{j}-X_{j-1})\nonumber\\
    & \quad + \sum \limits_{j=1}^{{\ell}_2-1} \lambda_1(C_j\cap A_{{\ell}_2}) (X_{j+1}-X_j)\nonumber \\
    & \quad +\lambda_1(C_{{\ell}_2+1} \cap A_{{\ell}_2})(X_{{\ell}_2+1}-b)\nonumber\\
    & \quad + \sum \limits_{j={\ell}_2+2}^{n} \lambda_1(C_j\cap A_{{\ell}_2}) (X_{j}-X_{j-1}).\label{two}
    \end{align}
    
    Let $B=\bigcup_{i={\ell}_1-1}^{{\ell}_2}A_i$, and observe that the target integral can be decomposed in the following way:
    \begin{align}
    \int_0^1 |G(u)-T(u)|{\rm d}u & =\int_{B}|G(u)-T(u)|{\rm d}u  +\int_{B^c}|G(u)-T(u)|{\rm d}u.\label{three}
    \end{align} 
    Inequalities~\eqref{one} and~\eqref{two} provide a lower bound on the first term on the right-hand side of~\eqref{three}. Let us now work out the second term. To this aim, observe that 
    \begin{align*}
    \int_{B^c} |G(u)-T(u)|{\rm d}u & \geqslant \sum \limits_{j=1}^{{\ell}_1-1} \int_{B^c} |G(u)-X_j|\mathds{1}\{u\in  C_j\}{\rm d}u \\
    & \quad + \sum \limits_{j={\ell}_2+1}^{n} \int_{B^c} |G(u)-X_j|\mathds{1}\{u\in  C_j\} {\rm d}u\\ 
    & \geqslant \sum \limits_{j=1}^{{\ell}_1-2} \int_{B^c}  (X_{{\ell}_1-1}-X_j)\mathds{1}\{u\in  C_j\}{\rm d}u \\
    & \quad + \int_{B^c} (a-X_{{\ell}_1-1})\mathds{1}\{u\in  C_{{\ell}_1-1}\} {\rm d}u\\
    & \quad + \int_{B^c} (X_{{\ell}_2+1}-b)\mathds{1}\{u\in  C_{{\ell}_2+1}\} {\rm d}u \\
    & \quad + \sum \limits_{j={\ell}_2+2}^{n} \int_{B^c} (X_j-X_{{\ell}_2+1})\mathds{1}\{u\in  C_j\} {\rm d}u.
    \end{align*}
    Exploiting $\lambda_1(C_j)=\frac{1}{n}$ for $j \in \{1, \hdots, n\}$, we see that
    \begin{align}
    \int_{B^c} |G(u)-T(u)|{\rm d}u 
    &  \geqslant \sum \limits_{j=1}^{{\ell}_1-2} \Big(\frac{1}{n}-\sum \limits_{i={\ell}_1-1}^{{\ell}_2} \lambda_1(C_j\cap A_i)\Big)(X_{j+1}-X_j)\nonumber\\
    & \quad +\Big(\frac{1}{n}-\sum \limits_{i={\ell}_1-1}^{{\ell}_2} \lambda_1(C_{{\ell}_1-1}\cap A_i)  \Big) (a-X_{{\ell}_1-1})\nonumber\\
    & \quad + \Big(\frac{1}{n}-\sum \limits_{i={\ell}_1-1}^{{\ell}_2}\lambda_1(C_{{\ell}_2+1}\cap A_i)\Big) (X_{{\ell}_2+1}-b)\nonumber\\
    & \quad+ \sum \limits_{j={\ell}_2+2}^{n} \Big(\frac{1}{n}-\sum \limits_{i={\ell}_1-1}^{{\ell}_2} \lambda_1(C_j\cap A_i) \Big) (X_{j}-X_{j-1}).\label{four}
    \end{align}
    Thus, using identity~\eqref{three} together with inequalities~\eqref{one},~\eqref{two}, and~\eqref{four}, we are led to 
    \begin{align*}
    \int_0^1 |G(u)-T(u)|{\rm d}u & \geqslant \frac{(X_{{\ell}_1}-a)^2}{4K} + \frac{(b-X_{{\ell}_2})^2}{4K} \\
    & \hspace{-0.35cm} + \sum \limits_{j=1}^{{\ell}_1-2} \Big(\frac{1}{n}-\sum \limits_{i={\ell}_1-1}^{{\ell}_2} \lambda_1(C_j\cap A_i) + \sum \limits_{i={\ell}_1-1}^{{\ell}_2} \lambda_1(C_j\cap A_i)\Big)(X_{j+1}-X_j)\\
    & \hspace{-0.35cm} +\Big(\frac{1}{n}-\sum \limits_{i={\ell}_1-1}^{{\ell}_2} \lambda_1(C_{{\ell}_1-1}\cap A_i)  +\sum \limits_{i={\ell}_1-1}^{{\ell}_2} \lambda_1(C_{{\ell}_1-1}\cap A_i) \Big) (a-X_{{\ell}_1-1})\\
    & \hspace{-0.35cm} + \sum \limits_{i={\ell}_1}^{{\ell}_2-1} \frac{(X_{i+1}-X_i)^2}{4K}\\
    & \hspace{-0.35cm} + \Big(\frac{1}{n}-\sum \limits_{i={\ell}_1-1}^{{\ell}_2}\lambda_1(C_{{\ell}_2+1}\cap A_i) +\sum \limits_{i={\ell}_1-1}^{{\ell}_2}\lambda_1(C_{{\ell}_2+1}\cap A_i) \Big)(X_{{\ell}_2+1}-b)\\
    & \hspace{-0.35cm} + \sum \limits_{j={\ell}_2+2}^{n} \Big(\frac{1}{n}-\sum \limits_{i={\ell}_1-1}^{{\ell}_2} \lambda_1(C_j\cap A_i)+\sum \limits_{i={\ell}_1-1}^{{\ell}_2} \lambda_1(C_j\cap A_i) \Big) (X_{j}-X_{j-1}).
    \end{align*}
    So,
    \begin{align*}
    \int_0^1 |G(u)-T(u)|{\rm d}u & \geqslant \frac{(X_{{\ell}_1}-a)^2}{4K} + \sum \limits_{i={\ell}_1}^{{\ell}_2-1} \frac{(X_{i+1}-X_i)^2}{4K} +\frac{(b-X_{{\ell}_2})^2}{4K}\\
    & \quad +  \sum \limits_{j\in \{1,\hdots,{\ell}_1-2\}\cup \{{\ell}_2+1,\hdots,n-1\}} \frac{X_{j+1}-X_j}{n} + \frac{1}{n} (a-X_{{\ell}_1-1}) \\
    & \quad + \frac{1}{n} (X_{{\ell}_2+1}-b).
    \end{align*} 
    Since $K\geqslant n \max \limits_{i=1,\hdots,n-1} (X_{i+1}-X_i)$, we have $\frac{X_{j+1}-X_j}{n} \geqslant \frac{(X_{j+1}-X_j)^2}{K}$, and thus
    \begin{align}\label{Uni}  
    \frac{(X_{{\ell}_1}-a)^2}{4K} + \frac{1}{n} (a-X_{{\ell}_1-1}) &\geqslant \frac{1}{4K} \big((X_{{\ell}_1}-a)^2 + 4(a-X_{{\ell}_1-1})(X_{{\ell}_1}-X_{{\ell}_1-1})\big)\nonumber \\
    & = \frac{1}{4K} \big((X_{{\ell}_1}-a)^2 + 4(a-X_{{\ell}_1-1})(X_{{\ell}_1}-a\big)\nonumber \\
    & \quad + 4 (a-X_{{\ell}_1-1})^2\big)\nonumber \\
    & \geqslant \frac{(X_{{\ell}_1}-X_{{\ell}_1-1})^2}{4K}.
    \end{align}
    Similarly, 
    \[\frac{(X_{{\ell}_2}-b)^2}{4K} + \frac{1}{n} (X_{{\ell}_2+1}-b) \geqslant \frac{(X_{{\ell}_2+1}-X_{{\ell}_2})^2}{4K}.\] 
    Using once again the assumption on $K$, we conclude that 
    \[\int_0^1 |G(u)-T(u)|{\rm d}u \geqslant \sum \limits_{i=1}^{n-1} \frac{(X_{i+1}-X_i)^2}{4K}.\]
    
    To complete the proof, it remains to show that ${\widehat G}_{K}^{\star}$ and ${\widehat G}_{K}^{\star}\circ S$ are the only minimizers of~\eqref{eq:studyGANs} (Main Document).
    Returning to inequality~\eqref{Uni}, we see that if the function $G$ does not visit each data points, then 
    \[
    \int_0^1 |G(u)-T(u)|{\rm d}u > \sum \limits_{i=1}^{n-1} \frac{(X_{i+1}-X_i)^2}{4K}.
    \]
    Also, according to~\eqref{Kspeed}, for the function $G$ to be optimal it needs to go at speed $K$ between each observation. Finally, with equation~\eqref{Monge}, we have that an optimal $G$ must be such that
    \[
    \lambda_1 \big(\{u \in [0,1] :  |G(u) -X_i |\leqslant  |G(u) -X_j |, \ j=1, \hdots, n \}\big) = \frac{1}{n},
    \]
    a property satisfied by ${\widehat G}_{K}^{\star}$ and ${\widehat G}_{K}^{\star}\circ S$ according to~\eqref{NFR} (Main Document). We conclude that ${\widehat G}^{\star}_K$ and $ {\widehat G}^{\star}_K\circ S$ are the unique minimizers of Problem~\eqref{eq:studyGANs} (Main Document) as they are the only functions satisfying these three conditions. 
    \section{Proof of Proposition~\ref{K1d1}}
    The first statement is a straightforward consequence of \citet[][Theorem 2]{Deheuvels84}. Regarding the second statement, we know from Theorem~\ref{theorem1} that, for all $K\geqslant \underline {K}_1$, 
    \[W_1({\widehat G}^{\star}_{K\sharp  U},\mu_n)=\inf_{G\in \text{Lip}_K([0,1],\mathbb{R})} W_1(G_{\sharp  U},\mu_n) =  \frac{1}{4K} \sum_{i=1}^{n-1} (X_{(i+1)}-X_{(i)})^2.\]
    Therefore,
    \begin{align*}
    W_1({\widehat G}^{\star}_{K\sharp  U},\mu_n)
    & \leqslant  \frac{\sum_{i=1}^{n-1} (X_{(i+1)}-X_{(i)})^2}{n\max_{i=1,\hdots,n-1} (X_{(i+1)}-X_{(i)})}\\
    & \leqslant  \frac{1}{n} \sum_{i=1}^{n-1} (X_{(i+1)}-X_{(i)})\\
    & =  \frac{1}{n} (X_{(n)}-X_{(1)})\\
    & \leqslant \frac{B-A}{n}.
    \end{align*}
    Recalling that $W_1(\mu,\mu_n)=\mathscr O(n^{-1/2})$ in probability \cite[][Theorem 1]{FournierGuillin2015}, the conclusion follows from the triangle inequality.
    
    \section{Proof of Proposition~\ref{proposition2}}
    The result is a consequence of the following lemma:
    \begin{lemma}\label{lemma2}
    For each $G \in {\rm Lip}_K([0,1],\mathbb R^d)$, there exists a sequence of functions $(G_m)_{m\in \mathbb{N}}$ in ${\rm Lip}_K([0,1],\mathbb R^d)$ such that each $G_{m\sharp  U}$ is nonatomic and $W_1(G_{m\sharp  U},\mu_n) \to W_1(G_{\sharp  U},\mu_n)$ as $m\to\infty$. 
    \end{lemma}
    \begin{proof}
    Let $G \in \text{Lip}_K([0,1],\mathbb R^d)$ and $m\in\mathbb N$.
    We define $G_m$ by slightly modifying $G$ on each  interval where it is constant. More precisely, let $\mathcal{I}$ be the set of all non degenerated connected components of  $G^{-1}(\{y\in\mathbb R^d : \lambda_1(G^{-1}(y))>0\})$. This set is at most countable and, since $G$ is continuous, it contains only disjoint closed intervals, i.e.,    
    \[
    \mathcal{I} = \{[a_\ell, b_\ell] : \ell\in \mathcal L\},
    \]
    where $\mathcal L\subset\mathbb N $ and $0\leqslant a_\ell < b_\ell \leqslant 1$.
    Let $K_m = \min(K,1/m)$, $e_1=(1,0,\hdots,0)\in\mathbb R^d$, and 
    \[
    G_m(u) = \begin{cases}
    G(a_\ell) + K_m\big(\frac{b_\ell-a_\ell}{2} - 
    \big\lvert\frac{a_\ell+b_\ell}{2} - u \big\rvert\big) e_1 &\text{if $u\in [a_\ell, b_\ell]$ for some $\ell\in\mathcal L$} \\
    G(u) &\text{otherwise.}
    \end{cases}
    \]
    It is easy to see that $G_m\in{\rm Lip}_K([0,1],\mathbb R^d)$. Moreover, $G_m$ is not constant over any non degenerated interval. Thus, the distribution $\ G_{m\sharp U}$ is nonatomic. In addition, $\|G_m-G\|_\infty \to 0$ as $m\to\infty$. In particular, for any continuous bounded function $f:\mathbb R^d\to \mathbb R$, $\|f(G_m)-f(G)\|_\infty \to 0$, so that $G_{m\sharp  U} \to G_{\sharp  U}$ weakly, as $m$ tends to infinity. As the $G_{m\sharp  U}$'s have supports included in the same compact set, we conclude by \citet[][Theorem 6.9]{villani2008optimal}
     that $\lim_{m\to\infty} W_1(G_{m\sharp  U},G_{\sharp  U})= 0$. But, by the triangle inequality,
     \[
    \big| W_1(G_{m\sharp  U},\mu_n)-W_1(G_{\sharp  U},\mu_n)\big|\leqslant  W_1(G_{m\sharp  U},G_{\sharp  U}),
     \]
    from which $\lim_{m\to \infty} W_1(G_{m\sharp  U},\mu_n)= W_1(G_{\sharp  U},\mu_n)$ follows, as desired.
    \end{proof}
    \section{Proof of Proposition~\ref{proposition3}}
 Assuming that such a transport map $T^\star \in \mathscr{H}^{w^\star}$ exists, we write $w_{T^\star(x)}^\star$ instead of $w_i^\star$ whenever $T^\star(x)=X_i$, $i\in \{1,\hdots,n\}$.
Let $\varphi:\mathbb{R}^d\rightarrow \mathbb{R}$ be the $1$-Lipschitz map defined by
\[
\varphi(x)=\|x-T^\star(x)\|-w_{T^\star(x)}^\star.
\]
Since $T^\star(X_i)=X_i$ for all $i\in \{1,\hdots,n\}$, we have in particular that $\varphi(x)-\varphi(T^\star(x))=\|x-T^\star(x)\|$. Then, denoting by
\[
\partial \varphi:= \{(x,y)\in \mathbb{R}^d\times \mathbb{R}^d : \varphi(x)-\varphi(y)=\|x-y\|\}
\]
the superdifferential of $\varphi$ \citep[][Definition 5.7]{villani2008optimal}, the graph of $T^\star$ is included in $\partial \varphi$. Therefore,
\begin{align*}
    \int_{\mathbb{R}^d\times \mathbb{R}^d}\|x-T^\star(x)\|{\rm d}\nu(x) & =\int_{\mathbb{R}^d\times \mathbb{R}^d}(\varphi(x)-\varphi(T^\star(x))){\rm d}\nu(x)\\
    & = \int_{\mathbb{R}^d}\varphi(x){\rm d}\nu(x)-\int_{\mathbb{R}^d}\varphi(y){\rm d}\mu_n(y)\\
    & \leq W_1(\nu,\mu_n).
\end{align*}
We conclude that $T^\star$ is an optimal transport map.
    \section{Proof of Proposition~\ref{proposition4}}
    Let us first show that, for all $i\in \{1,\hdots,n+k -1\}$ and $j\notin \{\sigma(i),\sigma(i+1)\}$, 
    \[[V_i+ \varphi(\sigma(i)),V_{i+1}] \cap {\widehat G}_K^{\star{-1}} (\text{Vor}(j)^\circ) = \emptyset.\] 
    Suppose on the contrary that there exists $t\in (0,1)$ such that $Y_i := X_{\sigma(i)}+t(X_{\sigma(i+1)}-X_{\sigma(i)}) \in \text{Vor}(j)^\circ$. Then 
    \[X_j \in B^{\circ}(Y_i,\|X_{\sigma(i)}-Y_i\|) \cap B^{\circ}(Y_i,\|X_{\sigma(i+1)}-Y_i\|),\]
    where $B^{\circ}(x,\varepsilon)$ stands for the open ball centered at $x$ of radius $\varepsilon$. Observe that for $t\leqslant 1/2$,  \[B^{\circ}(Y_i,\|X_{\sigma(i)}-Y_i\|) \subseteq B^{\circ}\Big(\frac{X_{\sigma(i)}+X_{\sigma(i+1)}}{2},\frac{\|X_{\sigma(i+1)}-X_{\sigma(i)}\|}{2}\Big),\]
    whereas for $t\geqslant 1/2$,
    \[B^{\circ}(Y_i,\|X_{\sigma(i+1)}-Y_i\|) \subseteq B^{\circ}\Big(\frac{X_{\sigma(i)}+X_{\sigma(i+1)}}{2},\frac{\|X_{\sigma(i+1)}-X_{\sigma(i)}\|}{2}\Big).
    \]
    Consequently, 
    \[X_j \in B^{\circ}\Big(\frac{X_{\sigma(i)}+X_{\sigma(i+1)}}{2},\frac{\|X_{\sigma(i+1)}-X_{\sigma(i)}\|}{2}\Big).
    \]
    We deduce that $\langle X_{\sigma(i)}-X_j, X_{\sigma(i+1)}-X_j\rangle < 0$ (notation $\langle \cdot,\cdot \rangle$ means the scalar product), and so 
    \[\|X_{\sigma(i+1)}-X_{\sigma(i)}\|^2> \|X_{\sigma(i+1)}-X_j\|^2+\|X_{\sigma(i)}-X_j\|^2.
    \]
    However, such an inequality is impossible by definition of $\sigma$. We conclude that, for all $t \in [0,1/2]$,
    \[ X_{\sigma(i)} + t(X_{\sigma(i+1)}-X_{\sigma(i)}) \in \text{Vor}(\sigma(i))
    \]
    and, for all $t \in [1/2,1]$, 
    \[X_{\sigma(i)} + t(X_{\sigma(i+1)}-X_{\sigma(i)}) \in \text{Vor}(\sigma(i+1)).
    \]
    
    Let us now turn to the computation of $W_1({\widehat G}^{\star}_{K\sharp  U},\mu_n)$. First, by definition of $\varphi(i)$, for $i \in \{1,\hdots,n\}$, we have 
    \begin{align*}
    &\sum \limits_{j\in \sigma^{-1}(i)}  \lambda_1\Big(\big[V_j,V_j+ \varphi(i)+ \frac{\|X_{\sigma(j+1)}-X_i\|}{2K}\big]\Big) \\
    & \ + \lambda_1\Big(\big[V_{j-1} + \varphi(\sigma(j-1))+ \frac{\|X_{\sigma(j-1)}-X_i\|}{2K},V_{j-1} + \varphi(\sigma(j-1))+ \|X_{\sigma(j-1)}-X_i\|\big]\Big)\\ 
    & \ =\sum \limits_{j\in \sigma^{-1}(i)} \Big(\varphi(i)+ \frac{\|X_{\sigma(j+1)}-X_i\|}{2K} +\frac{\|X_{\sigma(j-1)}-X_i\|}{2K} \Big)\\ 
    & \ =\frac{1}{n}.
    \end{align*}
    This shows that $\lambda_1({\widehat G}_K^{\star-1}(\text{Vor}(i)))=\frac{1}{n}$, $i\in \{1,\hdots,n\}$---or, said differently, that the function 
    ${\widehat G}_{K}^{\star}$ spends a total time $1/n$ in each Voronoi cell. Now, introduce $T^{\star} : \mathbb{R}^d \to \{X_1,\hdots,X_n\}$ defined ${\widehat G}_{K\sharp U}$-almost everywhere by $T^{\star}(x)=X_i $  if $x \in {\rm Vor}(i)$. Then, clearly, $ T^{\star} \in {\mathscr H}^0$, 
    where we recall that
    \begin{align*}
    {\mathscr H}^0 &=\big\{T:\mathbb{R}^d \to \{X_1, \hdots, X_n\} \ : \ \forall x \in \text{Vor}(i), T(x)=X_i \\
    & \hspace{4.5cm}\text{and } \forall x \in \Gamma_{j_1\hdots j_p}^0, T(x) \in \{X_{j_1},\hdots  ,X_{j_p}\}\big\}.
    \end{align*}
    Arguing as in the proof of Lemma~\ref{lemma2}, one shows that there exists a sequence of functions $(G^\star_m)_{m \in \mathbb N}\subset\text{Lip}_K([0,1],\mathbb R^d)$ such that each $G^\star_{m\sharp  U}$ is nonatomic, $W_1(G^\star_{m\sharp  U},\mu_n)\to W_1({\widehat G}^{\star}_{K\sharp  U},\mu_n)$ as $m\to \infty$, and, for all $m$ large enough, $\lambda_1(G^{\star-1}_m(\text{Vor}(i)))=\frac{1}{n}$, $i\in \{1,\hdots,n\}$. According to Proposition~\ref{proposition3}, we have  
    \[W_1(G^\star_{m\sharp  U},\mu_n) = \int_0^1 \|G^\star_m(u)-T^{\star}(G^\star_m(u))\|{\rm d}u.\]
    By dominated convergence, we obtain $W_1({\widehat G}^{\star}_{K\sharp  U},\mu_n) = \int_0^1 \|{\widehat G}^{\star}_{K}(u)-T^{\star}({\widehat G}^{\star}_{K}(u))\|{\rm d}u$, so that $T^{\star}$ is an optimal transport map from ${\widehat G}_{K}^\star$ to $\mu_n$. Finally, 
    \begin{align*}
    W_1({\widehat G}^{\star}_{K\sharp  U},\mu_n) & = \int_0^1 \|{\widehat G}^{\star}_{K}(u)-T^{\star}({\widehat G}^{\star}_{K}(u))\|{\rm d}u \\
    & = \sum \limits_{j=1}^{n+k-1} \int_{V_j}^{V_j+ \varphi(\sigma(j)) + \frac{\|X_{\sigma(j+1)}-X_{\sigma(j)}\|}{2K}} \|X_{\sigma(j)}-{\widehat G}_{K}^\star(u)\|{\rm d}u \\
    & \quad + \int_{V_j+ \varphi(\sigma(j))+ \frac{\|X_{\sigma(j+1)}-X_{\sigma(j)}\|}{2K}}^{V_j+ \varphi(\sigma(j)) + \|X_{\sigma(j+1)}-X_{\sigma(j)}\|} \|X_{\sigma(j+1)}-{\widehat G}_{K}^\star(u)\|{\rm d}u\\
    & = \sum \limits_{j=1}^{n+k-1} \int_{V_j + \varphi(\sigma(j)) }^{V_j+ \varphi(\sigma(j)) + \frac{\|X_{\sigma(j+1)}-X_{\sigma(j)}\|}{2K}} K\big(u-(V_j + \varphi(\sigma(j)) )\big){\rm d}u\\
    &  \hspace{-0.3cm} + \int_{V_j+ \varphi(\sigma(j)) + \frac{\|X_{\sigma(j+1)}-X_{\sigma(j)}\|}{2K}}^{V_j+ \varphi(\sigma(j)) + \|X_{\sigma(j+1)}-X_{\sigma(j)}\|} K(V_j+ \varphi(\sigma(j)) + \|X_{\sigma(j+1)}-X_{\sigma(j)}\|-u){\rm d}u\\
    & = \sum \limits_{j=1}^{n+k-1} \frac{1}{8K}  \|X_{\sigma(j+1)}-X_{\sigma(j)}\|^2 + \frac{1}{8K}  \|X_{\sigma(j+1)}-X_{\sigma(j)}\|^2 \\
    & = \frac{1}{4K} \sum \limits_{j=1}^{n+k-1}  \|X_{\sigma(j+1)}-X_{\sigma(j)}\|^2.
    \end{align*}
    \section{Proof of Proposition~\ref{propositionK2}}
    First note, since $\sigma$ is a path with points that may be visited several times, that
    \begin{align} 
    \underline {K}_2 & \geqslant  \sum_{i=1}^n \sum_{j\in \sigma^{-1}(i)} \frac{1}{2} (\|X_{\sigma(j-1)}-X_i\|+\|X_{\sigma(j+1)}-X_i\|) \nonumber\\
    & \geqslant \inf_{\tau\in \mathscr P_n} \sum_{j=1}^{n-1} \|X_{\tau(j)}-X_{\tau(j+1)}\|, \label{minK}
    \end{align}
    where $\mathscr P_n$ stands for the set of permutations of $\{1,\hdots,n\}$. But, according to \cite{Steele88}, under the conditions of the theorem, there exists a constant $C>0$ satisfying 
    \[\lim_{n\to \infty} n^{-1+1/d} \inf_{\tau\in \mathscr P_n} \sum_{j=1}^{n-1} \|X_{\tau(j)}-X_{\tau(j+1)}\|=C \text{ a.s.}\]
    This shows the first statement of the proposition.
    
    We start the proof of the second statement by recalling that, according to \citet[][Theorem 1]{FournierGuillin2015}, one has, in probability,
    \[W_1(\mu,\mu_n)=\left\{
        \begin{array}{ll}
            \mathscr O(\frac{\log n}{\sqrt{n}}) & \mbox{for } d=2 \\
            \mathscr O (n^{-1/d}) & \mbox{for } d\geqslant 3.
        \end{array}
    \right.
    \]
    Therefore, by the triangle inequality, it is enough to show that, for $d \geqslant 2$, in probability, 
    \[W_1({\widehat G}^{\star}_{K\sharp  U},\mu_n)=\mathscr O(n^{-1/d}).\]
    According to Theorem~\ref{theorem3}, we only need to show that, in probability,
    \[\frac{1}{4K} \sum_{j=1}^{n+k-1} \|X_{\sigma(j+1)}-X_{\sigma (j)}\|^2=\mathscr O(n^{-1/d}),\]
    whenever $K\geqslant \underline {K}_2$. But, by the very definition~\eqref{shortestpath} (Main Document) of the pair $(k,\sigma)$, we have 
    \[\sum_{j=1}^{n+k-1} \|X_{\sigma(j+1)}-X_{\sigma (j)}\|^2\leqslant \sum_{j=1}^{n-1} \|X_{\tau(j+1)}-X_{\tau (j)}\|^2,\]
    where $\tau\in \mathscr P_n$ is a permutation that  minimizes the length among the whole set of paths that visit only once each data, i.e.,
    \[\sum_{j=1}^{n-1} \|X_{\tau(j+1)}-X_{\tau (j)}\|\leqslant \sum_{j=1}^{n-1} \|X_{\tau'(j+1)}-X_{\tau' (j)}\|, \text{ for all } \tau'\in\mathscr P_n.\]
    Therefore, since $K\geqslant \underline {K}_2$, we have by inequality~\eqref{minK} ,
    \[\frac{1}{K} \sum_{j=1}^{n+k-1} \|X_{\sigma(j+1)}-X_{\sigma (j)}\|^2\leqslant \frac{\sum_{j=1}^{n-1} \|X_{\tau(j+1)}-X_{\tau (j)}\|^2}{\sum_{j=1}^{n-1} \|X_{\tau(j+1)}-X_{\tau (j)}\|}.\]
    Now, under the additional condition on the density of $\mu$, we know by \citet[][Theorem 1.3]{Yukich2000} that, for each $0\leqslant \ell\leqslant d$, there exists $C(\ell)>0$ such that 
    \[\lim_{n\to \infty} n^{-1+\ell/d} \sum_{j=1}^{n-1} \|X_{\tau(j+1)}-X_{\tau (j)}\|^\ell=C(\ell) \text{ a.s.}\]
    By the above, we conclude that
    \[\frac{1}{4K} \sum_{j=1}^{n+k-1} \|X_{\sigma(j+1)}-X_{\sigma (j)}\|^2=\mathscr{O}(n^{-1/d}) \text{ a.s.}\]

\end{document}